\documentclass[letterpaper,10pt]{article}
\usepackage[margin=1in]{geometry}

\usepackage{graphicx}
\usepackage{multirow}
\usepackage{amsmath,amssymb,amsfonts}
\usepackage{amsthm}
\usepackage{mathrsfs}
\usepackage{xcolor}
\usepackage{textcomp}
\usepackage{booktabs}
\usepackage{float}
\usepackage{algorithm}
\usepackage{algpseudocode}
\usepackage{subcaption}
\usepackage{caption}
\usepackage{mathtools}
\usepackage{listings}

\usepackage[authoryear,round]{natbib}
\setcitestyle{authoryear,round,citesep={;},aysep={,},yysep={;}}

\usepackage{hyperref}
\usepackage{url}

\usepackage[T1]{fontenc}
\usepackage{bm}
\usepackage{enumitem}
\usepackage{tikz}
\usetikzlibrary{arrows.meta,positioning,calc}

\definecolor{supfill}{RGB}{233,244,234}
\definecolor{supline}{RGB}{63,142,88}
\definecolor{subfill}{RGB}{228,237,250}
\definecolor{subline}{RGB}{47,111,214}
\definecolor{neufill}{RGB}{246,246,244}
\definecolor{neuline}{RGB}{205,205,200}
\definecolor{inkgray}{RGB}{112,112,112}
\definecolor{cblue}{RGB}{47,111,214}
\definecolor{corange}{RGB}{232,99,42}

\newcommand{\subtxt}[1]{{\footnotesize\color{inkgray}#1}}
\newcommand{\R}{\mathbb{R}}
\newcommand{\eps}{\varepsilon}
\newcommand{\trans}{^{\top}}
\newcommand{\diag}{\operatorname{diag}}
\newcommand{\Span}{\operatorname{span}}
\newcommand{\Tr}{\operatorname{tr}}
\newcommand{\norm}[1]{\left\lVert #1\right\rVert}
\newcommand{\abs}[1]{\left|#1\right|}
\newcommand{\Amap}{\mathcal{A}}
\newcommand{\Emap}{\mathcal{E}}
\newcommand{\xp}{x^{\scriptscriptstyle \parallel}}
\newcommand{\xperp}{x^{\perp}}
\newcommand{\mpar}{m^{\scriptscriptstyle \parallel}}
\newcommand{\vpar}{v^{\scriptscriptstyle \parallel}}
\newcommand{\PS}{\mathsf S}


\hypersetup{colorlinks, hypertexnames=false, pageanchor=true,
            linkcolor=blue, citecolor={green!50!black}, urlcolor=cyan,
            pdftitle={Adam at the Edge of Stability: Adaptive Feedback, Provable Oscillation, and Gradient Reversal}
            }
\mathtoolsset{centercolon}

\newtheorem{theorem}{Theorem}[section]
\newtheorem{lemma}[theorem]{Lemma}
\newtheorem{corollary}[theorem]{Corollary}
\newtheorem{proposition}[theorem]{Proposition}
\theoremstyle{definition}
\newtheorem{assumption}[theorem]{Assumption}
\newtheorem{definition}[theorem]{Definition}
\theoremstyle{remark}
\newtheorem{remark}[theorem]{Remark}

\title{Adam at the Edge of Stability: Adaptive Feedback, Provable Oscillation, and Gradient Reversal}

\author{%
  Yiman Fong\thanks{School of Engineering and Applied Sciences, Harvard University.}\\
  \texttt{fangyimin05@gmail.com}
  \and
  Heng Yang\footnotemark[2]\\
  \texttt{hankyang@seas.harvard.edu}
}

\begin{document}

\maketitle

\begin{abstract}
The edge-of-stability (EoS) phenomenon of full-batch Adam has been widely observed, yet its underlying dynamical mechanism remains poorly understood. In this paper, we identify Adam's second-moment adaptation as a negative-feedback mechanism that drives the dynamics toward the stability boundary.
We characterize this mechanism through the \emph{active curvature}, namely, the preconditioned curvature along the preconditioned gradient direction, and establish rigorous characterizations in progressively richer settings: rank-one quadratics with momentum, diagonal quadratics, on which the active curvature separates from the sharpness, and general objectives. Importantly, the mechanism predicts \emph{gradient reversal} of full-batch Adam near
the edge: consecutive gradients repeatedly point in nearly opposite
directions, as we observe across fully connected networks, ResNets, ViTs, LSTMs, GPT-2 medium, and Adam-family optimizers. Consistent with this picture, averaging iterates suppresses these fast oscillations and produces
smoother and lower loss curves. Together, these results provide an important first step towards fully understanding the
dynamical behavior of Adam's EoS through active curvature and gradient
reversal.
\end{abstract}

\section{Introduction}
\label{sec:intro}

\begin{figure}[!b]
\centering
\includegraphics[width=\textwidth]{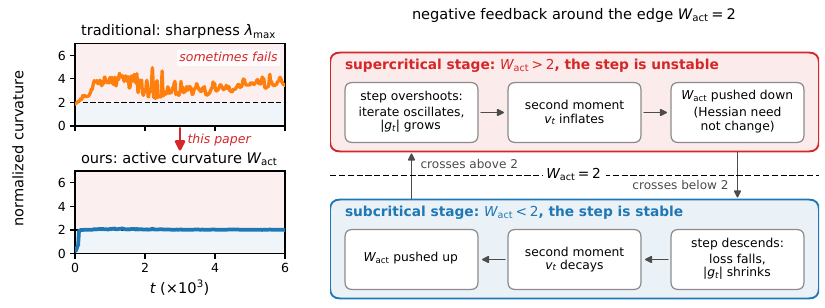}
\caption{Left: the normalized sharpness (Eq. \eqref{eq:P}, orange, top) and the active curvature
(Eq. \eqref{eq:Wact}, blue, bottom) along an Adam training trajectory; the
stability boundary of the frozen step is normalized to $2$ (dashed).  Right: the negative-feedback loop of the second moment that holds the active curvature around the threshold
(Section~\ref{sec:active}).  Experimental details are in
Appendix~\ref{app:nn-sharpness}.
}
\label{fig:mechanism-beta0}\label{fig:mechanism}
\end{figure}

Adam \citep{kingma2015adam} is among the most widely used optimizers in
modern deep learning. Despite its ubiquity, however, the behavior of Adam on neural networks and the mechanisms underlying that behavior remain poorly understood. One particularly striking phenomenon is the \emph{edge of stability} (EoS), whereby first-order optimization methods operate near the boundary of dynamical stability. For gradient descent (GD), \citet{cohen2021gradient} observed that, during neural-network training, the sharpness $\lambda_{\max}(H)$---the largest eigenvalue of the Hessian $H=\nabla^2 L(\theta)$---increases toward the classical stability threshold $2/\eta$, where $\eta$ is the learning rate.

For full-batch Adam, \citet{cohen2024adaptivegradientmethodsedge} observed a similar phenomenon: the \emph{preconditioned sharpness} $\PS_t$---the largest eigenvalue of the Hessian after Adam's coordinatewise rescaling---increases toward and subsequently remains near the stability threshold of the corresponding linearized dynamics. These observations suggest that EoS is a robust feature of neural-network optimization, yet its underlying mechanism remains largely unexplained. While the EoS phenomenon for gradient descent has received substantial theoretical attention, considerably less is understood for Adam. This leaves a basic question unresolved: \emph{why does  full-batch Adam operate near the edge of stability?}

A natural intuition is that Adam's second-moment normalization is a \emph{negative-feedback mechanism} that regulates the optimization dynamics around the stability boundary. In this paper, we formalize this intuition through a theoretical characterization of Adam's dynamics, as illustrated in Figure~\ref{fig:mechanism} (right). Specifically, our theory naturally leads to the notion of \emph{active curvature}, $W_{\mathrm{act},t}$, which is closely related to the preconditioned sharpness but measures curvature along the gradient direction.
The provable EoS mechanism can be sketched as follows. When $W_{\mathrm{act},t}$ exceeds the stability threshold, the dynamics become unstable: the iterates overshoot, the gradients grow, and Adam's second-moment estimate $v_t$ increases; the resulting stronger preconditioning pushes $W_{\mathrm{act},t}$ downward. When $W_{\mathrm{act},t}$ lies below the stability threshold, the dynamics are stable and contracting: the gradients shrink, $v_t$ decays, and $W_{\mathrm{act},t}$ is pushed upward. Thus, Adam's second-moment adaptation corrects deviations from the stability boundary in both directions, naturally driving the dynamics toward the edge.

Specifically, we establish this feedback mechanism rigorously in progressively richer settings: first on quadratic objectives, which provide the simplest setting in which Adam's momentum and nonlinear preconditioning already lead to highly nontrivial dynamical behaviors, and then, without momentum, on general objectives. Empirically, the mechanism extends further, in line with the broader perspective that quadratic models can capture important aspects of neural-network optimization \citep{meterez2026defense}. We test this prediction across a broad range of neural-network training settings, varying architecture, model size, learning rate, momentum, and optimizer variants. Across these settings, whenever Adam operates at the edge, $W_{\mathrm{act}}$ remains close to $2$, even when the preconditioned sharpness lies far above its nominal stability threshold (Section~\ref{sec:active}; Appendix~\ref{app:experiments}). These results corroborate our explanation that the quantity regulated by Adam is not necessarily the largest curvature of the preconditioned Hessian, but rather the curvature encountered along the active optimization direction.

A further analysis of the dynamics underlying $W_{\mathrm{act}}$ suggests another characteristic behavior: \emph{gradient reversal}. Near $W_{\mathrm{act}}=2$, the dynamics along the active direction have a multiplier close to $-1$, suggesting that consecutive gradients should repeatedly reverse direction rather than decay smoothly. Guided by these theoretical insights, we examine gradient dynamics in neural-network training and find the predicted reversal behavior across a broad range of neural-network architectures and variants of the Adam family (Section~\ref{sec:obs-flip}; Appendix~\ref{app:reversal}). Consistent with this oscillatory picture, we further find that averaging the iterates, either by exponential moving average or simple iterate averaging, substantially smooths the loss trajectory and suppresses the oscillations induced by gradient reversal.

In summary, our contributions are twofold.
\begin{itemize}[leftmargin=1.5em,itemsep=3pt,topsep=2pt]

\item \textbf{Theory: a dynamical analysis of Adam.}
We begin with quadratic objectives and characterize Adam's behavior in both
the subcritical ($W_{\mathrm{act},t}<2$) and supercritical
($W_{\mathrm{act},t}>2$) regimes. We overcome the challenges posed by
momentum and nonlinear preconditioning by viewing Adam as a time-varying
linear dynamical system and carefully tracking the evolution of its
state-dependent dynamics. This analysis yields a rigorous characterization
of the restoring behavior around the stability boundary. To the best of our
knowledge, no prior work provides such a characterization of Adam's EoS under
standard parameter choices, even on quadratic objectives.

\item \textbf{Empirics: active curvature and gradient reversal in deep learning.}
Across fully connected networks, ResNets, ViTs, LSTMs, GPT-2 medium, and
several Adam-family optimizers, we find that $W_{\mathrm{act}}$ stays near the
stability boundary whenever training operates at the edge, even when the
preconditioned sharpness is substantially larger. The theory further predicts
\emph{gradient reversal}, which we observe broadly across models and
optimizer variants. Consistent with this oscillatory dynamics, midpoint and
exponential moving averages substantially smooth the loss trajectory and
often achieve lower loss than the raw iterates.

\end{itemize}

\section{Setup and stability quantities}
\label{sec:setup-main}

\paragraph{Adam.}
Let $f:\R^d\to\R$ be twice continuously differentiable, and write
$g_t:=\nabla f(x_t)$ and $H_t:=\nabla^2f(x_t)$ along a trajectory.  We
consider the uncorrected  (full-batch) Adam recursion
\begin{equation}
\begin{aligned}
 m_{t+1}&=\beta_1m_t+(1-\beta_1)\,g_t,\\
 v_{t+1}&=\beta_2v_t+(1-\beta_2)\,g_t\odot g_t,\\
 x_{t+1}&=x_t-\eta\,D_{t}^{-1}m_{t+1},
 \qquad D_{t}:=\diag\bigl(\sqrt{v_{t+1}}+\eps\bigr),
\end{aligned}
\label{eq:adam}
\end{equation}
with $0\le\beta_1<1$, $0\le\beta_2<1$, $\eta>0$, $\eps>0$, and initial
state $(x_0,m_0,v_0)\in\R^d\times\R^d\times[0,\infty)^d$.  Here $\odot$ and
the square root act coordinatewise. The recursion
preserves $v_t\ge0$, so $D_t\succeq\eps I$ and every step is well defined.

\paragraph{Frozen stability and normalization.}
To assess local stability, freeze the second moment $D_t$. The resulting map
on $(x,m)$ is governed by the preconditioned Hessian $D_t^{-1}H_t$, and
\begin{equation}
    \mathsf S_t:=\lambda_{\max}(D_t^{-1}H_t)
    \label{eq:precond-Sharpness}
\end{equation}
is the \emph{preconditioned sharpness}. When $H_t\succ0$,
$D_t^{-1}H_t$ is similar to the symmetric matrix
$D_t^{-1/2}H_tD_t^{-1/2}$, and the linearized Adam map is Schur stable iff
\begin{equation}
\textstyle
\mathsf S_t < \mathsf S^\star
:=\frac{2(1+\beta_1)}{\eta(1-\beta_1)}
\label{eq:Sstar}
\end{equation}
(Appendix~\ref{app:frozen}; Lemma~\ref{lem:frozen-main}).
This is the adaptive stability threshold of
\citet{cohen2024adaptivegradientmethodsedge}.

For convenience, let $c:=(1-\beta_1)/(1+\beta_1)$ and normalize the
preconditioned Hessian as
\begin{equation}
P_t:=c\eta\,D_t^{-1/2}H_tD_t^{-1/2},
\qquad
\lambda_{\max}(P_t)=c\eta\,\mathsf S_t.
\label{eq:P}
\end{equation}
The frozen stability boundary is then simply $\lambda_{\max}(P_t)=2$.

\paragraph{Active curvature.}
Let $\hat z_t:=D_t^{-1/2}g_t/\|D_t^{-1/2}g_t\|$ be the
preconditioned gradient direction.  Then
\begin{equation}
\textstyle
 W_{\mathrm{act},t}:=\hat z_t\trans P_t\hat z_t
 =c\eta\,\frac{g_t\trans D_t^{-1}H_tD_t^{-1}g_t}{g_t\trans D_t^{-1}g_t}
 \label{eq:Wact}
\end{equation}
is the \emph{active curvature}: the curvature of $P_t$ along the
preconditioned gradient. By expanding the unit direction $\hat z_t$ in the
eigenbasis of $P_t$, we can regard $W_{\mathrm{act},t}$ as the mean of the
eigenvalues under the gradient's own spectral measure, and therefore
$\lambda_{\min}(P_t)\le W_{\mathrm{act},t}\le\lambda_{\max}(P_t)$, with
equality on the right exactly when the gradient lies in the top eigenspace.

\section{Adam operates near an active stability boundary}
\label{sec:active}\label{sec:diag}

\citet{cohen2024adaptivegradientmethodsedge} observed that along full-batch Adam trajectories the preconditioned sharpness $\mathsf S_t$ holds roughly at the frozen threshold $\mathsf S^\star$, that is, $\lambda_{\max}(P_t)\approx2$ for large $t$. While such behavior has long been equated to the EoS, we observe that the sharpness itself may be only \emph{part} of the picture, empirically or theoretically. To capture the full picture of Adam EoS, we provide an in-depth investigation of the preconditioned sharpness and the active curvature in this section.

\paragraph{Motivating experiments.}
As an illustration, we compare the behavior of preconditioned sharpness $\mathsf S_t$ (or equivalently, the normalized sharpness $\lambda_{\max}(P_t)=c\eta\mathsf S_t$) and the active curvature $W_{\mathrm{act},t}$ on various settings (Figure~\ref{fig:motivate}; implementation details in Appendix~\ref{app:zoo}, more experiments in Appendices~\ref{app:nn-sharpness}--\ref{app:gpt2}). The sharpness can stay well above the threshold predicted by the linearized stability analysis, particularly when $\beta_1=0$. In contrast, the active curvature $W_{\mathrm{act}}$ is much more stable and remains close to the threshold for most of training. A natural explanation is that preconditioned sharpness provides a worst-case spectral view and can therefore be dominated by high-curvature directions that carry little gradient mass. Active curvature instead measures the curvature along the direction actually explored by the optimizer, complementing this worst-case perspective (Section~\ref{sec:separation}). In Appendix~\ref{app:experiments}, we provide additional experiments showing that active curvature more closely tracks the stability of the Adam trajectory and study how these observations are affected by mini-batch noise.

\begin{figure}[t]
\centering
\includegraphics[width=\textwidth]{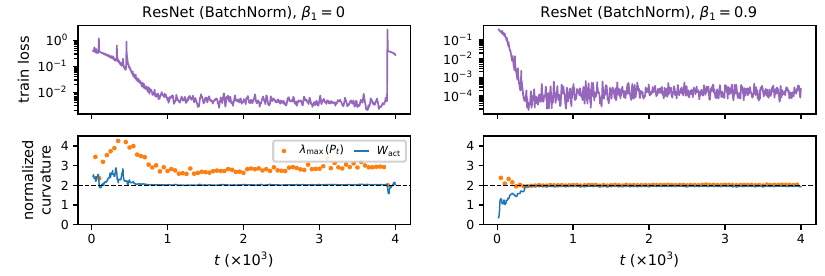}
\caption{Sharpness and active curvature along the same run.
BatchNorm ResNet, full-batch MSE on $1000$ CIFAR-10 images,
$\eta=3\cdot10^{-4}$, $\beta_2=0.999$, $4000$ steps, at $\beta_1=0$ (left)
and $\beta_1=0.9$ (right). Training details are in Appendix~\ref{app:zoo}.}
\label{fig:motivate}
\end{figure}

\paragraph{Our theory.} The rest of this section explains the mechanism behind Adam's EoS. The roadmap is as follows:
\begin{itemize}[leftmargin=1.6em,itemsep=1pt,topsep=2pt]
    \item In Section~\ref{sec:rank-1} and \ref{sec:separation}, we start with \emph{quadratic} objectives. We first provide rigorous characterization of Adam's dynamics on rank-one quadratics (Section~\ref{sec:rank-1}), thereby establishing the \emph{negative-feedback} loop between the iterate and the curvature. To the best of our knowledge, this is the first rigorous explanation of Adam's EoS with momentum ($\beta_1>0$)\footnote{Existing work~\citep{cohen2025understanding,bai2026adaptivepreconditionerstriggerloss} either analyzes Adam at degenerate hyperparameter $\beta_1=0$ or adopts heuristics (e.g., central flow).}.
    \item In Section~\ref{sec:separation}, we study 2-dimensional diagonal quadratics, where the normalized sharpness can hold well above the threshold while the active curvature provably cannot stay below an explicit cutoff $\overline W_{\rm diag}$ ($\ge1.9$ at Adam's default parameters), nor, at the default parameters, above the threshold unless the iterates converge.
    \item In Section~\ref{sec:general}, we generalize our analysis to general objectives with $\beta_1=0$ and argue that the active curvature is indeed a natural measure of stability.
    \item Finally, in Section~\ref{sec:flip}, we relate the active curvature to the phenomenon of \emph{gradient reversal}, identifying it as a common signature of Adam's EoS.
\end{itemize}

\subsection{Rank one: the feedback loop that pins the active curvature}
\label{sec:rank-1}\label{sec:theory-pinned}

We begin with the rank-one quadratic, the simplest objective on which the
direction that carries the gradient need not be a coordinate axis.  On this
objective, the
edge-of-stability behavior of Adam admits a precise characterization.
Gradient descent, by contrast, shows no such behavior, because the curvature of a quadratic is fixed.

\paragraph{Reduction to scalar dynamics.}
Let $f(x)=\tfrac12\lambda(u\trans x)^2$ with $\lambda>0$ and $\norm{u}=1$,
and write $\xp_t:=u\trans x_t$ and $\mpar_t:=u\trans m_t$.  Since
$g_t=\lambda\xp_tu$, the standard initialization $m_0=v_0=0$ keeps the
momentum in $\Span\{u\}$ and the second moment proportional to $u\odot u$,
$v_t=\lambda^2\vpar_t\,u\odot u$ for a scalar $\vpar_t\ge0$
(Appendix~\ref{app:frozen}).  Adam therefore reduces to the scalar system
\begin{equation}
\begin{aligned}
    \mpar_{t+1}
    &=\beta_1\mpar_t+(1-\beta_1)\lambda\xp_t,\\
    \vpar_{t+1}
    &=\beta_2\vpar_t+(1-\beta_2)(\xp_t)^2,\\
    \xp_{t+1}
    &=\xp_t-\eta\bigl(u\trans D_t^{-1}u\bigr)\mpar_{t+1},
\end{aligned}
\qquad
\mpar_0=\vpar_0=0,
\label{eq:reduced-dynamics}
\end{equation}
while the component of $x_t$ orthogonal to $u$ does not feed back.  On
this objective the normalized sharpness and the active curvature coincide
at every step (Appendix~\ref{app:rank-one-curvature}):
\begin{equation}
    W_{\mathrm{act},t}
    =\lambda_{\max}(P_t)
    =c\eta\lambda\,u\trans D_t^{-1}u
    ={w_{t+1}:=c\eta\lambda
      \sum_{i=1}^d
      \frac{u_i^2}{\sqrt{v_{i,t+1}}+\eps}}.
    \label{eq:w-rank-one}
\end{equation}
In particular $0<w_t\le w_{\max}:=c\eta\lambda/\eps$, the value at
${\vpar_t}=0$. Therefore, to understand the EoS behavior of Adam, it is sufficient to study the dynamics of $(\xp_t, \mpar_t, w_t)$. As our starting point, we first point out that the threshold $w=2$ is not merely a level the trajectory crosses: it is realized exactly, by a $2$-cycle of this dynamics on which the active curvature equals $2$ at every step.

\begin{proposition}[2-cycle at the edge]
\label{prop:main-cycle}
Assume $w_{\max}>2$, and let $x_\star>0$ be the unique value for which the second-moment level $\vpar_t=x_\star^2$ gives $w_t=2$.  Then
\begin{equation*}
 {(x_\star,\,-c\lambda x_\star,\,x_\star^2)\ \longleftrightarrow\ (-x_\star,\,c\lambda x_\star,\,x_\star^2)}
\end{equation*}
is a $2$-cycle of the reduced Adam dynamics of $(\xp_t,\mpar_t,\vpar_t)$, and along it $W_{\mathrm{act},t}={w_{t+1}}\equiv2$ and the gradient reverses at every step.
\end{proposition}

\paragraph{Illustrative example: RMSProp ($\beta_1=0$).}
As an illustrative example, we start with the case $\beta_1=0$, in which
Adam reduces to RMSProp \citep{tieleman2012rmsprop}.  This is the simpler
case because $c=1$ and $m_{t+1}=g_t$, so the recursion
reduces to
\begin{equation}
 \xp_{t+1}=(1-w_{t+1})\,\xp_t.
 \label{eq:beta0-main}
\end{equation}
In this setting, we can describe the negative-feedback loop of EoS as follows:
\begin{itemize}[leftmargin=1.6em,itemsep=1pt,topsep=2pt]
    \item \textbf{Subcritical} regime: For time steps $t$ such that $w_{t+1}\in[\delta,2-\delta]$, we have $|\xp_{t+1}|\leq (1-\delta)|\xp_t|$, and hence the sequence $(\xp_t)$ contracts at such steps. Then the second moments $(v_t)$ also shrink, which pushes $(w_t)$ upward.
    \item \textbf{Supercritical} regime: For time steps $t$ such that $w_{t+1}>2+\delta$, we have $|\xp_{t+1}|\geq (1+\delta)|\xp_t|$, and hence the sequence $(\xp_t)$ blows up exponentially at such steps. Then the second moments $(v_t)$ also blow up, which pushes $(w_t)$ downward.
\end{itemize}
In other words, for every $\delta>0$, a trajectory
cannot remain indefinitely below $2-\delta$ or above $2+\delta$, and the
number of steps it can spend on either side is bounded explicitly.  Since
$\delta$ may be chosen arbitrarily close to $0$, the result formalizes
the restoring behavior toward the edge. The following theorem makes this quantitative, with explicit bounds on the number of steps spent on either side of the threshold; the proof is in Appendix~\ref{app:beta0}.\footnote{In particular, this recovers the asymptotic result of \citet{bai2026adaptivepreconditionerstriggerloss}
on
\emph{one-dimensional} quadratics.}
\begin{theorem}[RMSProp in rank one]
\label{thm:RMSProp-rank1}
Consider the rank-one model with $w_{\max}>2$, $\beta_1=0$, and $\xp_t\ne0$ for all $t$. If $w_T<\widehat w<2$ then $w_{T+n}\ge\widehat w$ for some
$1\le n\le N_{\rm sub}$, and if $w_T\ge2+\delta$ with $\delta>0$ then
$w_{T+n}<2+\delta$ for some $1\le n\le N_{\rm sup}$.
\end{theorem}

Consequently, for RMSProp, we always have
$\liminf_tw_t\le2\le\limsup_tw_t$.
Our analysis is \emph{non-asymptotic}: $N_{\rm sub}$ and $N_{\rm sup}$ are defined in \eqref{eq:Nsub-imp} and \eqref{eq:explicit-super}, and explicit closed-form upper bounds on them are given in
\eqref{eq:sub-explicit} and \eqref{eq:explicit-super}.

\paragraph{General case.}
With $\beta_1>0$ the momentum no longer drops out, and we have to consider
the two-dimensional dynamics
\begin{equation}
 \xp_{t+1}=\xp_t-(c\lambda)^{-1}w_{t+1}\,\mpar_{t+1},\qquad
 \mpar_{t+1}=\beta_1\mpar_t+(1-\beta_1)\lambda\xp_t .
 \label{eq:active-main}
\end{equation}
The frozen threshold is still $w=2$, but the presence of momentum introduces complications. In our analysis,  in the subcritical regime we need to control the joint evolution of
position and momentum; in the supercritical regime the key distinction is
whether the two are aligned or misaligned.  To this end, we introduce the
normalized momentum coordinate $h_t:=(c\lambda)^{-1}\mpar_t+\xp_t$.  We
generalize Theorem~\ref{thm:RMSProp-rank1} to this setting as follows.

\begin{theorem}[Adam in rank one; Proposition~\ref{prop:lyapunov},
Corollaries~\ref{cor:positive-sub-passage} and~\ref{cor:aligned-exit},
Lemma~\ref{lem:sign-geometry}]
\label{thm:Adam-rankone}
Consider the rank-one model with $w_{\max}>2$, $0<\beta_1<1$ and
$\beta_2>\beta_1^2$.
\begin{itemize}[leftmargin=1.6em,itemsep=1pt,topsep=2pt]
\item \textbf{Subcritical regime.}  Write
\begin{equation}
 \overline W:=\frac{2(\sqrt{\beta_2}-\beta_1)}
 {\sqrt{\beta_2}(1-\beta_1)-2\beta_1(1-\sqrt{\beta_2})/w_{\max}}.
 \label{eq:Wbar-restated}
\end{equation}
Then $0<\overline W<2$, and for any $b<\overline W$, if $w_T\le b$ then
$w_{T+n}>b$ for some finite $n$.
\item \textbf{Supercritical regime, aligned.}  Suppose that $\xp_Th_T>0$
and $w_T\ge2+\delta$ for some $\delta>0$.  Then $w_{T+n}<2+\delta$ for
some $1\le n\le{\overline N_{\rm sup}}$, with $\overline N_{\rm sup}$ explicit in
\eqref{eq:aligned-explicit}.
\item \textbf{Supercritical regime, misaligned.}  Suppose that $w_t>2$ and
$\xp_th_t<0$ for all $t\ge T$.  Then $\abs{\xp_{t+1}}<\abs{\xp_t}/(w_{t+1}-1)$
for $t\ge T$, so $\abs{\xp_t}$ decreases at every step although the frozen dynamics
remains unstable.  In Proposition~\ref{prop:persistent}, we show that such trajectories do exist for suitable initial states (Assumption~\ref{as:r0}).
\end{itemize}
\end{theorem}

We briefly discuss each point of Theorem~\ref{thm:Adam-rankone} as follows.
In the subcritical regime, Theorem~\ref{thm:Adam-rankone} introduces the
cutoff $\overline W<2$, below which the contraction is certified. For the standard choice
$(\beta_1,\beta_2)=(0.9,0.999)$ and large $w_{\max}$ it is $\overline W\approx1.991$. Such a cutoff arises because in the presence of momentum, the Adam dynamics is not contracting under a single norm in the regime $w_t<2$, and we have to introduce an adaptive Lyapunov function (Appendix~\ref{app:positive-sub}) to analyze the contraction behavior.
In the supercritical regime, the sign of $\xp_th_t$ separates two sharply different behaviors.  An
aligned step with $\xp_th_t>0$ preserves alignment and expands $\abs{\xp_t}$ by a factor at
least $w_{t+1}-1>1$.  Consecutive misaligned steps instead imply contraction of $\abs{\xp_t}$ even though the dynamics is unstable, and this is a genuine and important phenomenon for the \emph{positive} momentum.

The full details of Theorem~\ref{thm:Adam-rankone} are presented in Appendix~\ref{app:positive-sub}, where we also show that the condition $\beta_2>\beta_1^2$ cannot be dropped: for $\eps=0$, $\beta_1\ge\sqrt2-1$ and $\beta_2\le\beta_1^2$ there are orbits of exact period four along which $w_t<2$ throughout (Proposition~\ref{prop:four-cycle}).

\subsection{Separating sharpness from the active curvature}
\label{sec:separation}

On the rank-one quadratic the preconditioned sharpness and the active curvature coincide. However, as we observed in our experiments, their behaviors can diverge beyond such a simple setting. In the following, we demonstrate such separation in a simple yet representative setting: the diagonal quadratic, with
$H=\diag(\lambda_1,\dots,\lambda_n)$. On this objective, the Adam dynamics is $n$ decoupled 1-dimensional dynamics, each with its own normalized
curvature $w_{i,t}:=c\eta\lambda_i/(\sqrt{v_{i,t}}+\eps)$, and
\begin{equation}
\textstyle
 P_t=\diag({w_{1,t+1}},\dots,{w_{n,t+1}}),\qquad
 \lambda_{\max}(P_t)=\max_i {w_{i,t+1}},\qquad
 W_{\mathrm{act},t}=\sum_{i=1}^n\hat z_{i,t}^2\,{w_{i,t+1}}.
 \label{eq:P-diag}
\end{equation}
For each coordinate $i$, the behavior of $w_{i,t}$ is characterized in Section~\ref{sec:theory-pinned}.  The normalized sharpness $\lambda_{\max}(P_t)=\max_i {w_{i,t+1}}$ is therefore the maximum of $n$ coordinates that each fluctuate around the threshold $2$, so one may expect the maximum to stay above $2$, as we indeed observe in Figure~\ref{fig:diag}
(Appendix~\ref{app:dense}). The following example makes this intuition concrete.

\paragraph{Sharpness can be conservative.}
Consider the following illustrative example:
Let $H=\diag(1,(c\eta)^{-1}A\eps)$ with $A>2$
and let Adam start at $x_0=(1,0)$ with $m_0=0, v_0=1$.  The dynamics of the second coordinate is trivial as $x_{2,t}=g_{2,t}\equiv 0$, and hence $w_{2,t}=\frac{A}{\beta_2^{t/2}\eps^{-1}+1}\to A$ converges exponentially as $t\to \infty$. In particular, the normalized sharpness $\lambda_{\max}(P_t)\geq {w_{2,t+1}}$, and hence it stays well above the threshold $2$ for large $t$.
The example is extreme, but it illustrates how directions that carry little gradient can dominate the sharpness. In our experiments (Figure~\ref{fig:nn-beta1} of Appendix~\ref{app:nn-sharpness} and Table~\ref{tab:family} of Appendix~\ref{app:family}) this effect is most pronounced when $\beta_1$ is small (e.g., RMSProp with $\beta_1=0$).

\paragraph{Active curvature towards the edge.}
As the simple example above illustrates, the ``edge of stability'' behavior in fact has to do with the coordinates that carry gradient. This intuition leads to the following theorem, which extends the two regimes of Theorem~\ref{thm:Adam-rankone} to the active curvature.

\begin{theorem}[Active curvature on diagonal quadratics]
\label{prop:main-diag}
Let $H=\diag(\lambda_1,\dots,\lambda_n)\succ0$, $\eps>0$, let
$w_{i,\max}:=c\eta\lambda_i/\eps>2$ for every $i$, and let $x_t\ne0$ for
all $t$.  Let $W_{\mathrm{act},t}$ be built from $D_t$ as in
\eqref{eq:P-diag}.
\begin{enumerate}[leftmargin=1.6em,itemsep=1pt,topsep=2pt]
\item \textbf{Supercritical.}  Let $0<\beta_1<1$ and $\beta_1^2<\beta_2<1$, and assume that one-dimensional Adam with these parameters admits a potential (Definition~\ref{def:potential}). By Lemma~\ref{lem:storage}, this holds at the default parameters $(\beta_1,\beta_2)=(0.9,0.999)$.  If $W_{\mathrm{act},t}\ge2+\delta$ for some
$\delta>0$ and all sufficiently large $t$, then
$\sum_t\norm{x_t}^2<\infty$; in particular $x_t\to0$ and
$w_{i,t}\to w_{i,\max}$.  Consequently, every trajectory with
$\limsup_t\norm{x_t}>0$ has $\liminf_tW_{\mathrm{act},t}\le2$.
\item \textbf{Subcritical, dimension 2.}  Let $n=2$, $0<\beta_1<1$ and
$\beta_1^2<\beta_2<1$.  There is an explicit cutoff
$\overline W_{\rm diag}(\beta_1,\beta_2)\in(0,2)$, defined in
Appendix~\ref{app:main-diag}, such that for every
$q<\overline W_{\rm diag}(\beta_1,\beta_2)$ the bound
$W_{\mathrm{act},t}\le q$ cannot hold for all sufficiently large $t$; hence
$\limsup_tW_{\mathrm{act},t}\ge\overline W_{\rm diag}(\beta_1,\beta_2)$.
For the standard choice $(\beta_1,\beta_2)=(0.9,0.999)$,
$\overline W_{\rm diag}\ge1.9$.
\end{enumerate}
\end{theorem}

Both regimes can be interpreted similarly to Theorem~\ref{thm:Adam-rankone}, though their proofs are much more involved because the active curvature aggregates the coordinate-wise curvature in a highly non-linear way. We refer to Appendix~\ref{app:main-diag} for a detailed discussion on our techniques. We also remark that for RMSProp ($\beta_1=0$) stronger results are established in the subsequent section.

\subsection{General objectives: the active curvature is the one-step loss criterion}
\label{sec:general}

In the following, we describe how to generalize our theory of the active curvature to general objectives, focusing on the case $\beta_1=0$ (RMSProp, $c=1$). The update then reduces to $x_{t+1}=x_t-\eta z_t$ with
$z_t:=D_{t}^{-1}g_t$, and it is straightforward to see
\begin{equation}
    f(x_{t+1})=f(x_t)-\eta g_t^\top z_t+\frac{\eta^2}{2}\,z_t\trans\overline H_tz_t, \qquad \overline H_t:=2\int_0^1(1-s)\,\nabla^2f(x_t-s\eta z_t)\,ds.
\end{equation}
Therefore, instead of investigating $W_{\mathrm{act},t}$ directly, it is more straightforward to study the \emph{zeroth-order active curvature} defined as
\begin{equation}
 W'_{\mathrm{act},t}
 :=\frac{2\bigl(f(x_{t+1})-f(x_t)\bigr)}{\eta g_t^\top D_{t}^{-1}g_t}+2=\eta\,\frac{g_t\trans D_t^{-1}\overline H_tD_t^{-1}g_t}{g_t\trans D_t^{-1}g_t}.
 \label{eq:Wact-loss}
\end{equation}
To interpret this, we first note that $W'_{\mathrm{act},t}=W_{\mathrm{act},t}$ when $f$ is a quadratic function, and for general $f$, the two active curvatures differ only in whether
the Hessian is read at $x_t$ or averaged along the realized step.
Further,
\begin{equation}
  f(x_{t+1})-f(x_t)=\frac{\eta}{2}\,g_t^\top D_{t}^{-1}g_t\bigl(W'_{\mathrm{act},t}-2\bigr).
\end{equation}
The threshold $2$ is therefore exact for every twice differentiable $f$: a step
raises the loss if and only if $W'_{\mathrm{act},t}>2$.
\begin{theorem}
\label{thm:main-general}
Let $\beta_1=0$ and $\eps>0$. Suppose that (a) $f\in C^2$ is bounded below with
$\sup_x\norm{\nabla^2f(x)}_{\rm op}\le\Lambda<\infty$, (b)
$\norm{\nabla f(x)}\to\infty$ as $\norm{x}\to\infty$, and (c) every
stationary point of $f$ is nondegenerate.
Then, along every Adam trajectory that does not reach a stationary point in finite
time, we have
\begin{equation}
 \liminf_{t\to\infty}W'_{\mathrm{act},t}\ \le\ 2+C_f \eta^{-1}\eps,
 \qquad
 2-C_f \eta^{-1}\eps\ \le\ \limsup_{t\to\infty}W'_{\mathrm{act},t},
 \label{eq:main-two-sided}
\end{equation}
where $C_f$ is a constant only depending on $f$ (the proof is in Appendix~\ref{app:general-path-passage}).
\end{theorem}

While the above result is stated for $W'_{\mathrm{act},t}$, the two
active curvatures differ only by a small relative error.  If $\nabla^2f$ is
Lipschitz, then
$\abs{W'_{\mathrm{act},t}-W_{\mathrm{act},t}}\le\frac{C_f'\eta}{\varkappa_t}\,W_{\mathrm{act},t}$,
where $C_f'$ depends only on $f$ and $\beta_2$, and $\varkappa_t$ is the
Rayleigh quotient of the raw Hessian $\nabla^2f(x_t)$, with no
preconditioner, along the update direction
(\eqref{eq:current-directional} in Appendix~\ref{app:general-cur}).  The
error is therefore $O(\eta)$ whenever $\varkappa_t$
stays bounded away from zero (Proposition~\ref{prop:defect} in
Appendix~\ref{app:general-cur}).

\section{Gradient reversal at EoS}
\label{sec:reversal}\label{sec:flip}

Oscillatory or period-two behavior near a stability boundary has
appeared in several prior analyses
\citep{pmlr-v202-chen23b,song2023trajectoryalignmentunderstandingedge,kalra2025universalsharpnessdynamicsneural,pmlr-v336-mulayoff26a}.
In this section, we further identify the phenomenon of \emph{gradient reversal}, which is predicted by our theoretical analysis in Section~\ref{sec:diag} and confirmed by our experiments (Section~\ref{sec:obs-flip}). In the following, we first illustrate how our theory predicts such phenomena on quadratic objectives and how it is closely connected with the active curvature.

\paragraph{Gradient reversal on rank-one quadratics.} On the rank-one quadratic, we can show that in the \emph{supercritical} regime (i.e., above the edge), the gradient at every step must reverse sign as long as $\xp_th_t>0$ (with $h_t$ as in Section~\ref{sec:rank-1}). We formalize this observation in the following lemma.

\begin{lemma}[Gradient and momentum reversal in the supercritical regime]
\label{lem:main-flip-momentum}
Let $0<\beta_1<1$ and consider the rank-one quadratic as in Section~\ref{sec:rank-1}.
Suppose that for step $T$, it holds that
$\xp_Th_T>0$ and $W_{\mathrm{act},t}>2$ for every $t\in[T,T+n]$, then
\begin{equation}
    \cos(g_t,g_{t+1})=-1, ~{t\in[T,T+n]}, \qquad
    \cos(m_t,m_{t+1})=-1, ~{t\in[T+1,T+n]}.
\end{equation}
\end{lemma}
The proof is in Appendix~\ref{app:flip-momentum}.
While rank-one quadratics are very special in that the gradients always belong to the 1-dimensional subspace $\Span(u)$, the behavior indicated by Lemma~\ref{lem:main-flip-momentum} is in fact very close to what we observed in training neural networks with full-batch gradients (Section~\ref{sec:obs-flip}).

\paragraph{Active curvature as a lens for gradient reversal.} The active curvature of Section~\ref{sec:diag} also explains gradient reversal, to some extent, beyond rank one. More specifically, on quadratic objectives with $\beta_1=0$ (RMSProp), we can express
\begin{equation}
    \label{eq:flipping-Wact}
    \textstyle
    \frac{\langle g_t,g_{t+1}\rangle_{D_t^{-1}}}{\norm{g_t}_{D_t^{-1}}^2}=1-W_{\mathrm{act},t}.
\end{equation}
We note that the LHS resembles the cosine between the preconditioned directions $D_t^{-1/2}g_t$ and $D_t^{-1/2}g_{t+1}$. In particular, $W_{\mathrm{act},t}>1$ implies that  $\langle g_t,g_{t+1}\rangle_{D_t^{-1}}<0$ (``weak reversal''). Furthermore, we can in fact express
\begin{equation}
\textstyle
    \cos(D_t^{-1/2}g_t, D_t^{-1/2}g_{t+1})=\frac{1-W_{\mathrm{act},t}}{\sqrt{(1-W_{\mathrm{act},t})^2+V_t}},
    \label{eq:flipping-cos}
\end{equation}
where $V_t=\norm{(P_t-W_{\mathrm{act},t}I)\hat{z}_t}^2$ can be regarded as the variance of the preconditioned curvature under the distribution induced by the normalized vector $\hat{z}_t$. Therefore, gradient reversal is implied when ${W_{\mathrm{act},t}}-1\gg \sqrt{V_t}$. On a general objective $f$ with Lipschitz Hessian, the analogue of Eq.~\eqref{eq:flipping-Wact} holds up to an $O(\eta)$ error: for $\beta_1=0$, reading the Hessian along the realized step instead of at $x_t$ gives
$\langle g_t,g_{t+1}\rangle_{D_t^{-1}}=(1-W_{\mathrm{act},t}-\hat e_t)\,\norm{g_t}_{D_t^{-1}}^2$
with $\abs{\hat e_t}\le\frac{C\eta}{\varkappa_t}\,W_{\mathrm{act},t}$, where $C$ depends only on $f$ and $\beta_2$ and $\varkappa_t$ is the curvature of $f$ along the step, as in Section~\ref{sec:general}. The error is thus $O(\eta)$ whenever $\varkappa_t$ stays bounded away from zero, and Eq.~\eqref{eq:flipping-cos} holds verbatim with the step-averaged preconditioned Hessian in place of $P_t$. This is the same Hessian-averaging argument that relates $W'_{\mathrm{act},t}$ to $W_{\mathrm{act},t}$ (Proposition~\ref{prop:flip-general} in Appendix~\ref{app:general-cur}).

\subsection{Gradient reversal across scales}
\label{sec:obs-flip}

We therefore measure the cosine between consecutive full-batch gradients across a broad range of neural-network training settings. Our experiments include fully connected networks of different sizes, ResNets, ViTs, LSTMs, GPT-2 medium, and several variants of the Adam family, while varying learning rates, $\beta_1$, $\beta_2$, the stabilizer, activation, dataset, training-set size, and batch size (Figures~\ref{fig:wact-flip} and~\ref{fig:flip}; the full set of settings is in Appendix~\ref{app:reversal}). Across these substantially different settings, gradient reversal appears repeatedly and often persists for long stretches of training. Its occurrence across architectures, scales, optimizers, and hyperparameters suggests that the phenomenon is not specific to a particular model or parameter choice.
Figure~\ref{fig:wact-flip} places the reversal next to the active
curvature in six of these settings: in four of them the gradient reverses on exactly the stretches on which $W_{\mathrm{act}}$ sits at the threshold (the other two are discussed in Appendix~\ref{app:reversal}).

\begin{figure}[t]
\centering
\includegraphics[width=\textwidth]{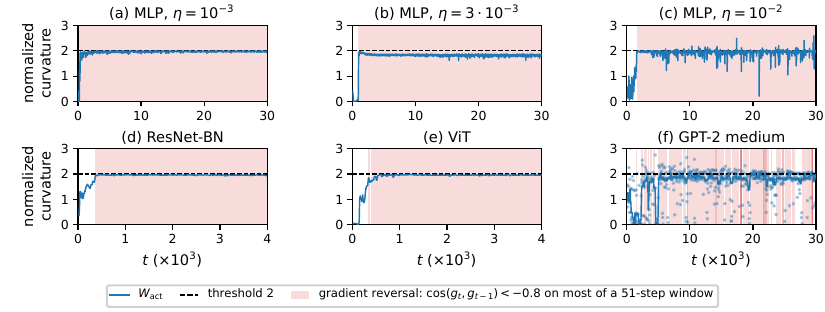}
\caption{Active curvature and gradient reversal in six full-batch Adam
runs with $\beta_1=0.9$, $\beta_2=0.999$.  (a--c) The width-$200$ fc-tanh
MLP of Appendix~\ref{app:flip-w200}.
(d,~e)
The ResNet with BatchNorm and the ViT of Appendix~\ref{app:zoo}.
(f) GPT-2
medium of Appendix~\ref{app:gpt2}.
Blue: $W_{\mathrm{act}}$; dashed line at $2$.  Red shading: the stretches
on which consecutive full-batch gradients reverse.
}
\label{fig:wact-flip}
\end{figure}

A striking feature of these runs is that strong gradient reversal does not prevent optimization progress. Even when consecutive gradients are nearly antiparallel for many steps, the training loss can continue to decrease smoothly (Figure~\ref{fig:flip}, where the stretches of reversal are shaded under the loss curve; Figure~\ref{fig:flip-lr} of Appendix~\ref{app:flip-w200}, bottom row against top row). This suggests that much of the motion is a fast, approximately period-two oscillation, superimposed on a smaller non-reversing component that continues to move the model toward lower loss.

Several observations support this interpretation. More directly, if consecutive iterates lie on opposite sides of a slowly moving center, their midpoint should cancel much of the alternating motion. Indeed, evaluating the loss at the midpoint of consecutive iterates produces a smoother and typically lower loss trajectory than evaluating it at the raw iterates (Figure~\ref{fig:averaging}; Figure~\ref{fig:flip-lr} of Appendix~\ref{app:flip-w200}, top row, dashed). Exponential moving averages show the same behavior: they suppress the rapid oscillation and expose a smoother trajectory underneath (Figure~\ref{fig:averaging}, dotted; the three learning rates and the numbers are in Appendix~\ref{app:flip-w200}).

\begin{figure}[t]
\centering
\includegraphics[width=\textwidth]{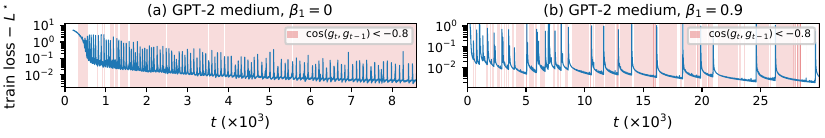}
\caption{Period-two orbits at the edge in GPT-2 medium, trained full-batch on $64$ fixed sequences with label-smoothed
cross-entropy, $\eta=3\cdot10^{-4}$: (a) $\beta_1=0$, $\beta_2=0.95$;
(b) $\beta_1=0.9$, $\beta_2=0.999$.  Blue: excess training loss $L-L^\star$
over the entropy of the smoothed target.  Red shading: the
stretches on which the full-batch gradient reverses. Training details are in
Appendix~\ref{app:gpt2}.
}
\label{fig:flip}
\end{figure}

Taken together, these experiments suggest a simple picture of Adam's late-stage dynamics. A large component of the motion can be spent repeatedly moving back and forth in an approximately period-two fashion, while a much slower component carries the net optimization progress. The smooth decrease of the averaged loss is consistent with the model continuing to descend through this slowly moving center even while the raw gradients reverse from step to step (Figure~\ref{fig:orbit-height}: the height of the orbit stays constant while the loss at its center falls).

\section{Conclusion}

We studied why Adam operates near the edge of stability and identified its
second-moment adaptation as a negative-feedback mechanism that continually
pushes the dynamics back toward the critical regime. This perspective leads
to the notion of active curvature, which captures the curvature encountered
along the optimization trajectory and remains close to the stability
boundary across a wide range of settings. Our analysis further predicts
gradient reversal as a characteristic signature of this regime, a phenomenon
that we observe broadly across architectures and Adam-family optimizers.
Together, these results provide a unified dynamical picture of Adam near the
edge of stability and suggest that its late-stage behavior is governed by
persistent oscillation around a slowly evolving optimization trajectory.
Extending the momentum analysis beyond quadratics, and the theory to
stochastic gradients, remains open.

\section*{AI disclosure}

In this work, we used AI tools to polish the paper and write code for generating the figures based on experimental data and verifying the numerical claims in our proof.
 We have reviewed all AI-assisted work carefully. We take responsibility for the final content of this work,
including text, claims and artifacts produced with the aid of generative AI.

\bibliographystyle{plainnat}
\bibliography{references}

\clearpage
\tableofcontents
\appendix
\section{Related work}
\label{app:related}

\paragraph{Edge of stability.}
The edge-of-stability (EoS) phenomenon was systematically documented by
\citet{cohen2021gradient}, with earlier observations of sharpness dynamics
along SGD trajectories by \citet{jastrzębski2019relationsharpestdirectionsdnn}
and precursors in the catapult phase
\citep{lewkowycz2020catapult}; the role of dynamical stability in
selecting minima was highlighted by \citet{wu2018sgd}.
Subsequent work has studied its underlying mechanisms, including implicit
regularization \citep{pmlr-v162-arora22a}, self-stabilization through
higher-order geometry \citep{damian2023selfstabilization}, progressive
sharpening \citep{li2022analyzingsharpnessgdtrajectory}, and EoS behavior in
simplified models
\citep{agarwala2023secondorder,ahn2023learningthresholdneuronsedge,
zhu2023understanding,pmlr-v202-chen23b}; for rank-one matrix
factorization, \citet{moon2026state} analyzes the edge with a
state-dependent Lyapunov function, the kind of argument our subcritical
analysis also uses (Appendix~\ref{app:positive-sub}).
Other works connect EoS to bifurcation and nonlinear oscillatory dynamics
\citep{song2023trajectoryalignmentunderstandingedge,
pmlr-v336-mulayoff26a,kalra2025universalsharpnessdynamicsneural}.
Related work has also studied directional, gradient-aware, or path-wise
notions of curvature and smoothness that characterize optimization behavior
along the trajectory rather than solely through the largest Hessian
eigenvalue
\citep{ahn2022understandingunstableconvergencegradient,
mishkin2024directional,lee2023a}.
In stochastic optimization, \citet{andreyev2025edgestochasticstabilityrevisiting}
showed that minibatch SGD operates near an edge characterized by
\emph{batch sharpness}, which measures curvature along stochastic gradient
directions rather than through the full-batch Hessian alone.
These results primarily concern gradient descent or SGD.
For adaptive methods, \citet{cohen2024adaptivegradientmethodsedge} identified
an analogous adaptive EoS characterized by the preconditioned Hessian.
Our work instead studies the exact finite-step Adam dynamics and identifies
the active curvature along the optimization direction as the quantity
regulated by adaptive preconditioning, first rigorously on quadratic
objectives and then beyond quadratics theoretically and empirically.

\paragraph{Adam.}
AdaGrad \citep{JMLR:v12:duchi11a} introduced coordinatewise adaptive learning
rates based on accumulated gradients, and \citet{kingma2015adam} later
introduced Adam by combining adaptive second-moment scaling with momentum.
Despite its empirical success, Adam can fail to converge
\citep{reddi2019convergenceadam}, motivating
convergence analyses under additional assumptions
\citep{chen2018on,zhang2023adamconvergemodificationupdate,dereich2025asymptoticstabilitypropertiespriori}.
Adam has also been studied through continuous-time dynamical models
\citep{JMLR:v21:18-808,barakat2020convergencedynamicalbehavioradam}.
More recent work has
investigated Adam's behavior on degenerate objectives and its implicit
effect on sharpness
\citep{bai2026understandingadamconvergencehighly,
li2025adamreducesuniqueform}.
Little work characterizes how Adam's adaptive preconditioner drives its sharpness relative to a finite-step stability
threshold.

\paragraph{Adam at the edge of stability.}
The works most closely related to ours study Adam directly through stability
and dynamical perspectives.
\citet{bai2026adaptivepreconditionerstriggerloss} explain loss spikes through
the evolution of the adaptive preconditioner, with their theoretical analysis
focusing on a one-dimensional quadratic setting with $\beta_1=0$.
\citet{cohen2025understanding} derive a central-flow description of the
time-averaged oscillatory dynamics near the edge, while
\citet{regis2026rodflowmodeladam} develop a continuous-time model for
Adam in the EoS regime.
From a discrete dynamical perspective, \citet{Bock_2019} showed that Adam can
admit non-convergent limit cycles, including quadratic examples, while
\citet{bock2021localconvergenceadaptivegradient} analyzed local convergence
through linear stability near fixed points. Recently,
\citet{dereich2025asymptoticstabilitypropertiespriori} established a priori
bounds and asymptotic stability properties for Adam on strongly convex
quadratic objectives.
In contrast, we study the exact finite-step dynamics across both subcritical
and supercritical regimes, identify the active curvature as the relevant
stability quantity, and connect its regulation near the edge to gradient
reversal, with the predicted dynamics further examined beyond quadratics and
across neural-network training settings.

\section{Experimental evidence across settings}
\label{app:experiments}

This appendix reports $\lambda_{\max}(P_t)$ and $W_{\mathrm{act}}$ on
quadratics (\ref{app:dense}), fully connected networks
(\ref{app:nn-sharpness}), residual, attention and recurrent networks
(\ref{app:zoo}) and GPT-2 medium (\ref{app:gpt2}), and then varies the
optimizer (\ref{app:family}), the learning-rate schedule
(\ref{app:lrdecay}) and the batch size (\ref{app:noise}).

\subsection{Quadratics}
\label{app:dense}

\begin{figure}[t]
\centering
\includegraphics[width=\textwidth]{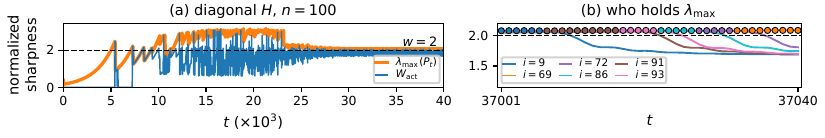}
\caption{The sharpness and the active curvature on a diagonal
quadratic.  Diagonal quadratic in $n=100$ dimensions, Adam with
$\beta_1=0.9$, $\beta_2=0.999$, $\eta=10^{-3}$.  (a) The sharpness
$\lambda_{\max}(P_t)=\max_i{w_{i,t+1}}$ (orange) and the active curvature
$W_{\mathrm{act}}$ (blue).  (b) Forty consecutive plateau steps, one color
per coordinate $i$: the line is its normalized curvature $w_{i,t+1}$ and the
dots mark the steps on which it attains the sharpness.  Spectrum,
initialization and the readings are described below.}
\label{fig:diag}
\end{figure}

The objective is $f(x)=\tfrac12x\trans Hx$ with
$H=\diag(\lambda_1,\dots,\lambda_n)$, $n=100$, and $\lambda_i$ log-uniformly
spaced on $[10^{-3},1]$.  The initial point is $x_0\sim\mathcal N(0,I_n)$ with
$m_0=v_0=0$. We run Adam with
 $\beta_1=0.9$,
$\beta_2=0.999$, $\eta=10^{-3}$, $\eps=0$ for $40{,}000$ steps in double precision.
As we discuss in Section~\ref{sec:diag},
because $H$ is
diagonal the recursion decouples: coordinate $i$ runs the one-dimensional
problem of curvature $\lambda_i$, and $P_t=c\eta D_t^{-1/2}HD_t^{-1/2}$ is
diagonal with entries $w_{i,t+1}=c\eta\lambda_i/\sqrt{v_{i,t+1}}$.

Panel (a) of Figure~\ref{fig:diag} draws every tenth step of the
two readings over the whole run.  Panel (b) is the window of steps
$37{,}001$--$37{,}040$: each coordinate that holds the maximum at some step
of the window receives a color, its $w_{i,t}$ is drawn as a line, and the
steps on which it is the holder carry a dot of that color. This figure explains why the sharpness can remain above the threshold while $W_{\mathrm{act}}$ does not: the top eigendirection can be continually handed off between coordinates whose gradients have become small, whereas $W_{\mathrm{act}}$ weights each direction by the gradient mass it currently carries.

\subsection{Fully connected networks: \texorpdfstring{learning rate, momentum and loss}{learning rate, momentum and loss}}
\label{app:nn-sharpness}

\begin{figure}[t]
\centering
\includegraphics[width=\textwidth]{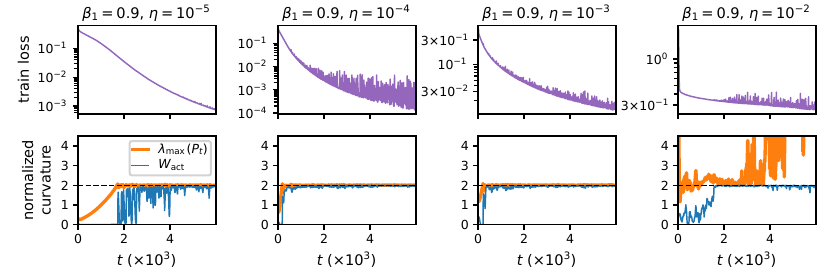}
\caption{Width-$200$ network ($3072{-}200{-}200{-}10$, $6.6\cdot10^{5}$
parameters), full-batch MSE, $\beta_1=0.9$, $6000$ steps, at four learning
rates: training loss (purple) above $\lambda_{\max}(P_t)$ (orange) and
$W_{\mathrm{act}}$ (blue), $31$-step rolling medians, threshold $2$ dashed.
The curvature axis is cut at $4.5$; at $\eta=10^{-2}$ the sharpness
rises above it (Table~\ref{tab:heldout}).}
\label{fig:nn-lr}
\end{figure}

\begin{figure}[t]
\centering
\includegraphics[width=\textwidth]{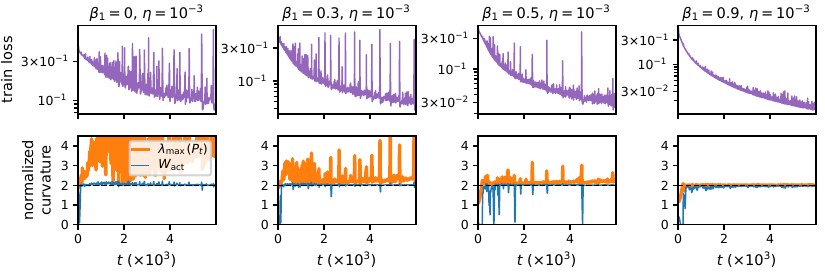}
\caption{The network and readings of Figure~\ref{fig:nn-lr}, MSE,
$\eta=10^{-3}$, at $\beta_1\in\{0,0.3,0.5,0.9\}$.  The $\beta_1=0.9$ panel
is the $\eta=10^{-3}$ run of Figure~\ref{fig:nn-lr}.}
\label{fig:nn-beta1}
\end{figure}

\begin{figure}[t]
\centering
\includegraphics[width=\textwidth]{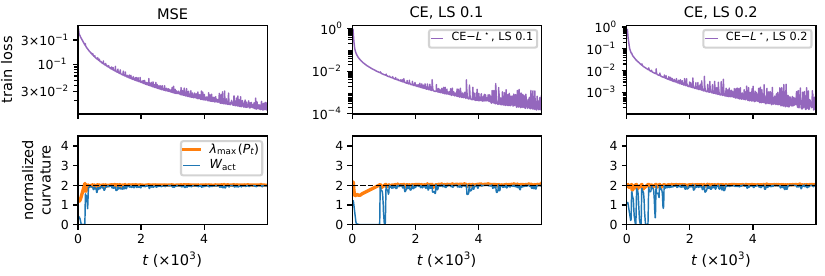}
\caption{The network and readings of Figure~\ref{fig:nn-lr}, $\beta_1=0.9$,
$\eta=10^{-3}$, under MSE and under cross-entropy with label smoothing (LS)
$0.1$ and $0.2$; the label-smoothed panels show the excess loss $L-L^\star$
over the entropy of the smoothed target.  The
MSE panel is the $\eta=10^{-3}$ run of Figures~\ref{fig:nn-lr}
and~\ref{fig:nn-beta1}.}
\label{fig:nn-loss}
\end{figure}

\begin{table}[t]
\caption{The runs of Figures~\ref{fig:nn-lr}--\ref{fig:nn-loss}: width-$200$ MLP on $5000$
CIFAR-10 images, full batch, $6000$ steps; medians of $\lambda_{\max}(P_t)$
and $W_{\mathrm{act}}$ over steps $3000$--$6000$ and the training
accuracy (median over steps $5800$--$5999$).  ``CE, LS $0.1$'' and ``CE, LS $0.2$'' are cross-entropy with
label smoothing $0.1$ and $0.2$.}
\label{tab:heldout}
\centering
\footnotesize
\setlength{\tabcolsep}{4pt}
\begin{tabular}{@{}lccccc@{}}
\toprule
loss & $\beta_1$ & $\eta$ & $\lambda_{\max}(P_t)$ & $W_{\mathrm{act}}$ & train acc\\
\midrule
MSE & 0 & $10^{-3}$ & 3.35 & 2.018 & 0.973 \\
MSE & 0.3 & $10^{-3}$ & 2.25 & 2.013 & 0.992 \\
MSE & 0.5 & $10^{-3}$ & 2.17 & 2.011 & 0.996 \\
MSE & 0.9 & $10^{-3}$ & 2.01 & 1.965 & 0.996 \\
MSE & 0.9 & $10^{-5}$ & 2.00 & 1.885 & 1.000 \\
MSE & 0.9 & $10^{-4}$ & 2.01 & 1.970 & 1.000 \\
MSE & 0.9 & $10^{-2}$ & 5.07 & 1.967 & 0.638 \\
CE, LS 0.1 & 0.9 & $10^{-3}$ & 2.03 & 1.957 & 1.000 \\
CE, LS 0.2 & 0.9 & $10^{-3}$ & 2.02 & 1.962 & 1.000 \\
 
\bottomrule
\end{tabular}
\end{table}

\paragraph{Implementation.}
\label{app:heldout}
Unless stated otherwise, the network runs of this appendix use the following protocol.  The network is an fc-tanh MLP on the first $5000$ CIFAR-10 images (standardized), trained with MSE against one-hot targets and full-batch gradients by uncorrected Adam with $\beta_2=0.999$ and $\eps=10^{-8}$, in float32 with seed $0$.  At every step $\lambda_{\max}(P_t)$ is a $25$-step warm-started Lanczos iteration on $P_t=c\eta D_t^{-1/2}H_tD_t^{-1/2}$, and $W_{\mathrm{act}}$ is exact, one Hessian-vector product along $\hat z_t$; both are drawn as $31$-step rolling medians without the first $20$ steps, where $v\approx0$ makes $P_0$ huge.  Figures~\ref{fig:nn-lr}--\ref{fig:nn-loss} and
Table~\ref{tab:heldout} use the width-$200$ network $3072{-}200{-}200{-}10$
($6.6\cdot10^{5}$ parameters) for $6000$ steps and change one factor at a
time from $\beta_1=0.9$, $\eta=10^{-3}$, MSE:
$\eta\in\{10^{-5},10^{-4},10^{-3},10^{-2}\}$,
$\beta_1\in\{0,0.3,0.5,0.9\}$, and cross-entropy with label smoothing $0.1$
and $0.2$, reported as the excess $L-L^\star$ over the entropy of the
smoothed target ($L^\star=0.500$ and $0.867$).  The left column of
Figure~\ref{fig:mechanism} is the $\beta_1=0$, $\eta=10^{-3}$ run, drawn as
centered $100$-step rolling medians because it jitters at every step.  Table
entries are medians over steps $3000$--$6000$; the training accuracy is the median over steps $5800$--$5999$.

\paragraph{Summary.}
All seven MSE runs reach the edge: $W_{\mathrm{act}}$ settles at $2$ (at $\eta=10^{-5}$ only after step $2100$, with median $1.89$),
while the sharpness ranges from $2$ to well above it
(Table~\ref{tab:heldout}).  The sharpness sits at $2$ only when the gradient
occupies the top eigenspace ($\beta_1=0.9$, $\eta\le10^{-3}$); at
$\beta_1=0$ or $\eta=10^{-2}$ it drifts upward while $W_{\mathrm{act}}$ stays
at $2$ (Figures~\ref{fig:nn-lr} and~\ref{fig:nn-beta1}).
Label-smoothed cross-entropy gives the same picture.

\subsection{Residual, attention and recurrent networks}
\label{app:zoo}

 The vision networks are a ResNet ($128$-channel stem, two basic
blocks of $128$ and $256$ channels, BatchNorm, $1.22$M parameters) and a ViT
(patch $4$, width $128$, four heads, two pre-LN blocks, $4.1\cdot10^{5}$
parameters) on the first $1000$ CIFAR-10 images with one-hot MSE.  The LSTM
has two layers of width $128$ and a tied embedding ($6.70$M parameters,
$6.43$M in the embedding) and is trained on $256$ FineWeb sequences of $64$
GPT-2 tokens with label-smoothed cross-entropy (smoothing $0.1$).  All runs are full batch, uncorrected Adam with
$\eta=3\cdot10^{-4}$, $\beta_2=0.999$, $\eps=10^{-8}$, a $100$-step warm-up
and $4000$ steps, at $\beta_1=0$ and $0.9$. Attention, LayerNorm and the LSTM cell are written in elementary operations, since the
fused kernels have no second derivative.  The full batch is processed in two fixed chunks of $500$ images, and BatchNorm couples the examples within a chunk, so $H_t$ is the Hessian of that chunked objective.

\begin{figure}[t]
\centering
\includegraphics[width=\textwidth]{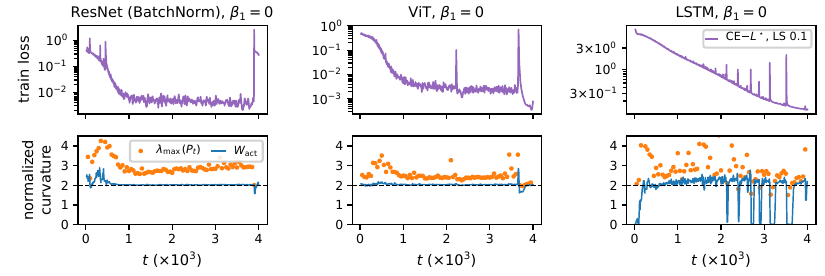}
\caption{BatchNorm ResNet and ViT (full-batch MSE on $1000$ CIFAR-10
images) and LSTM at $\beta_1=0$,
$\eta=3\cdot10^{-4}$, $\beta_2=0.999$, $4000$ steps.  Top: training loss
(purple).  Bottom: $\lambda_{\max}(P_t)$ (orange, every $50$ steps) and the
$31$-step rolling median of $W_{\mathrm{act}}$ (blue), dashed line at $2$.}
\label{fig:zoo-b0}
\end{figure}

\begin{figure}[t]
\centering
\includegraphics[width=\textwidth]{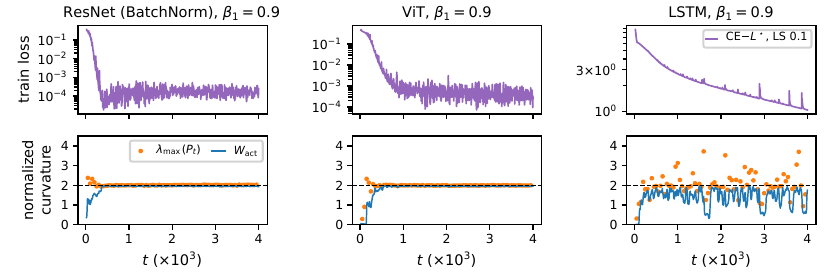}
\caption{As Figure~\ref{fig:zoo-b0} at $\beta_1=0.9$.}
\label{fig:zoo-b09}
\end{figure}

\begin{table}[t]
\caption{Second-half medians (steps $2000$--$3999$) of the two curvature
readings for the six runs of Figures~\ref{fig:zoo-b0} and~\ref{fig:zoo-b09};
loss is the training loss over the last $400$ steps (MSE for the vision
networks, excess cross-entropy over the floor $1.4075$ for the LSTM).}
\label{tab:zoo-arch}
\centering
\scriptsize
\setlength{\tabcolsep}{3pt}
\begin{tabular}{@{}lrcccr@{}}
\toprule
architecture & params & $\beta_1$ & $\lambda_{\max}(P_t)$ & $W_{\mathrm{act}}$ & loss\\
\midrule
ResNet-BN & 1.22M & 0 & 2.89 & 2.023 & 0.00406 \\
ResNet-BN & 1.22M & 0.9 & 2.02 & 1.957 & 0.000164 \\
ViT & 413k & 0 & 2.39 & 2.021 & 0.000777 \\
ViT & 413k & 0.9 & 1.99 & 1.955 & 0.000377 \\
LSTM & 6.70M & 0 & 2.46 & 2.102 & 0.139 \\
LSTM & 6.70M & 0.9 & 1.95 & 1.466 & 1.11 \\
 
\bottomrule
\end{tabular}
\end{table}

\subsection{GPT-2 medium}
\label{app:gpt2}

GPT-2 medium ($355$M parameters, $24$ layers), the transformer of
Section~\ref{sec:obs-flip}, is trained full batch on $64$ FineWeb sequences
of $1024$ tokens with label-smoothed cross-entropy (smoothing $0.1$; the vocabulary is padded to $50{,}304$ tokens, so the entropy floor is $1.4076$ rather than the LSTM's $1.4075$).  The optimizer is PyTorch AdamW without weight decay
or clipping, $\eta=3\cdot10^{-4}$ after a $200$-step warm-up, at
$(\beta_1,\beta_2)=(0,0.95)$ for $8600$ steps and at $(0.9,0.999)$ and
$(0.9,0.95)$ for $30{,}000$ steps (Figure~\ref{fig:gpt2-wact}).  $W_{\mathrm{act}}$ of \eqref{eq:Wact} is
exact and taken every $10$ steps at $\beta_1=0$ and every $50$ at
$\beta_1=0.9$; $\lambda_{\max}(P_t)$ is taken every $100$ or $250$ steps by
ten warm-started power iterations; at $\beta_1=0.9$ these converge except on some spike steps, at $\beta_1=0$ about a quarter of the readings have not converged.  The runs re-train the
two runs of Figure~\ref{fig:flip} and a $\beta_2=0.95$ twin; their reversal
is in Appendix~\ref{app:reversal-gpt2}.

\begin{figure}[t]
\centering
\includegraphics[width=\textwidth]{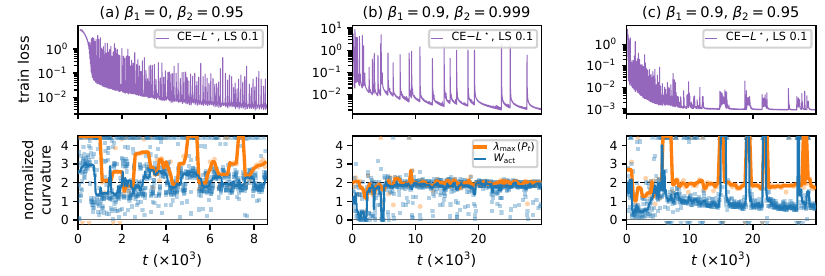}
\caption{GPT-2 medium, full batch, over the whole of each run: (a)
$\beta_1=0$, $\beta_2=0.95$, $8600$ steps; (b) $\beta_1=0.9$,
$\beta_2=0.999$, and (c) $\beta_1=0.9$, $\beta_2=0.95$, $30{,}000$ steps.
Top: training loss, excess over the label-smoothing floor $1.4076$
(purple).  Bottom: the exact full-batch sharpness $\lambda_{\max}(P_t)$
(orange; one reading every $100$ steps in (a), every $250$ in (b,~c)) and
active curvature $W_{\mathrm{act}}$ (blue; every $10$ steps in (a), every
$50$ in (b,~c)); faint markers are the individual readings, lines their
rolling medians over about $500$ steps ($750$ for the sharpness in (b,~c)); dashed line at $2$; readings
outside $[-0.25,4.5]$ are drawn as triangles at the edge.}
\label{fig:gpt2-wact}
\end{figure}

\subsection{Ablation study}

The runs so far use Adam at a constant learning rate and full batch.
We now vary the optimizer (Appendix~\ref{app:family}), the learning-rate
schedule (Appendix~\ref{app:lrdecay}) and the batch size
(Appendix~\ref{app:noise}).  The first two leave the regulation intact: eight
of the ten variants hold $W_{\mathrm{act}}$ at $2$, and a run at the
edge stays there through a decay.  The effect of the batch size depends on
$\beta_1$.

\subsubsection{The threshold across the Adam family}
\label{app:family}

The variants have different stability thresholds even with a frozen
preconditioner, so we normalize each by its own constant, which puts its
edge at $2$.

\paragraph{The ten updates.}
Every variant runs the same loop: from the full-batch gradient
$g_t=\nabla f(x_t)$ it advances a second-moment state, which gives the
diagonal preconditioner $D_t$, then a momentum state, then the iterate.
Uncorrected Adam is the recursion~\eqref{eq:adam},
\begin{align*}
 m_{t+1}&=\beta_1m_t+(1-\beta_1)g_t,
 &v_{t+1}&=\beta_2v_t+(1-\beta_2)\,g_t\odot g_t,\\
 D_t&=\diag\bigl(\sqrt{v_{t+1}}\bigr)+\eps I,
 &x_{t+1}&=x_t-\eta D_t^{-1}m_{t+1},
\end{align*}
and each of the nine others replaces only the lines written below; every
line not written is the corresponding line above.  Throughout $m_0=v_0=0$,
$\beta_1=0.9$ (momentum $\mu=0.9$ for the standard-parameterization
variant), $\beta_2=0.999$, $\eps=10^{-8}$, weight decay $\lambda=0.1$ (in this subsection only, $\lambda$ is the
weight-decay coefficient), and
all operations on vectors are coordinatewise.
\begin{itemize}
\item \textbf{Adam, bias-corrected.}  Both corrections are folded into the
preconditioner, as in \citet{cohen2024adaptivegradientmethodsedge}:
\begin{equation*}
 D_t=\bigl(1-\beta_1^{t+1}\bigr)
 \left[\diag\left(\sqrt{\frac{v_{t+1}}{1-\beta_2^{t+1}}}\right)+\eps I\right].
\end{equation*}
\item \textbf{AdamW.}  Decoupled weight decay:
\begin{equation*}
 x_{t+1}=(1-\eta\lambda)\,x_t-\eta D_t^{-1}m_{t+1}.
\end{equation*}
\item \textbf{Adam with $L_2$ regularization.}  Coupled weight decay: the
gradient
\begin{equation*}
 \widetilde g_t=\nabla f(x_t)+\lambda x_t
 =\nabla\left(f(x_t)+\frac{\lambda}{2}\norm{x_t}^2\right)
\end{equation*}
replaces $g_t$ in both moments, so the objective is
$f+\tfrac{\lambda}{2}\norm{\cdot}^2$ and the curvature read is
$\widetilde H_t=H_t+\lambda I$.
\item \textbf{Padam.}  A weaker preconditioner:
\begin{equation*}
 D_t=\diag\bigl(v_{t+1}^{\,p}\bigr)+\eps I,\qquad p=\tfrac14.
\end{equation*}
\item \textbf{Nadam.}  $m_t$ is updated as usual, but the parameter update uses $\widetilde m_{t+1}$:

\begin{equation*}
 \widetilde m_{t+1}=\beta_1m_{t+1}+(1-\beta_1)g_t,
 \qquad
 x_{t+1}=x_t-\eta D_t^{-1}\widetilde m_{t+1}.
\end{equation*}
\item \textbf{RMSProp.}  No momentum:
\begin{equation*}
 x_{t+1}=x_t-\eta D_t^{-1}g_t.
\end{equation*}
\item \textbf{RMSProp with momentum.}  The momentum sits after the
preconditioner and in the standard parameterization:
\begin{equation*}
 b_{t+1}=\mu b_t+D_t^{-1}g_t,
 \qquad
 x_{t+1}=x_t-\eta\,b_{t+1}.
\end{equation*}
\item \textbf{AMSGrad.}  The second moment ratchets:
\begin{equation*}
 \bar v_{t+1}=\max\{\bar v_t,v_{t+1}\},
 \qquad
 D_t=\diag\bigl(\sqrt{\bar v_{t+1}}\bigr)+\eps I,
\end{equation*}
the maximum taken coordinatewise, so that $D_{t+1}^{-1}\preceq D_t^{-1}$.
\item \textbf{Adagrad.}  The second moment accumulates and there is no
momentum:
\begin{equation*}
 v_{t+1}=v_t+g_t\odot g_t,
 \qquad
 x_{t+1}=x_t-\eta D_t^{-1}g_t.
\end{equation*}
\end{itemize}
The learning rate is $\eta=10^{-3}$, except for Padam and Adagrad, whose
weaker preconditioners need $\eta=10^{-2}$ to train.  The two readings are
then taken from \eqref{eq:Wact} and $\lambda_{\max}(P_t)$ with that
variant's own $D_t$.

With the preconditioner $D$ frozen, each variant is preconditioned
gradient descent with its own momentum, and its threshold for
$\lambda_{\max}(D^{-1/2}HD^{-1/2})$ is that of the unpreconditioned method.
We set $c:=2/(\eta\cdot\text{threshold})$, so that the normalized threshold
is $2$; here $c$ is the variant's own constant, and for Adam it is the
$c$ of Section~\ref{sec:setup-main}, $c_{\mathrm{Adam}}=(1-\beta_1)/(1+\beta_1)$.
Table~\ref{tab:family} lists $c$, the threshold and the readings of each
variant.

\begin{table}[t]
\caption{Ten variants, with $\beta_1=0.9$ or momentum $\mu=0.9$ where
the variant has a momentum term: the
normalization constant $c$ and the raw threshold of each, and their
readings on the width-$200$ network (fc-tanh $3072{-}200{-}200{-}10$,
$5000$ CIFAR-10 images, MSE, $3000$ steps, last quarter): means of the normalized
$\lambda_{\max}$ and $W_{\mathrm{act}}$, and the fraction of plateau steps
with $W_{\mathrm{act}}\in(1.8,2.1)$ (``locked'').}
\label{tab:family}
\centering
\scriptsize
\setlength{\tabcolsep}{3pt}
\begin{tabular}{@{}lllccc@{}}
\toprule
variant & $c$ & threshold & $\lambda_{\max}$ & $W_{\mathrm{act}}$ & locked\\
\midrule
Adam (uncorrected) & $(1-\beta_1)/(1+\beta_1)$ & $38/\eta$ & 2.01 & 1.95 & 100\% \\
Adam, bias-corrected & $c_{\mathrm{Adam}}$ & $38/\eta$ & 2.02 & 1.94 & 99\% \\
AdamW & $c_{\mathrm{Adam}}/(1-\eta\lambda/2)$ & $38(1-\eta\lambda/2)/\eta$ & 2.00 & 1.92 & 94\% \\
Adam + $L_2$ & $c_{\mathrm{Adam}}$ & $38/\eta$ & 1.87 & 1.76 & 43\% \\
Padam ($p=1/4$) & $c_{\mathrm{Adam}}$ & $38/\eta$ & 2.01 & 1.97 & 100\% \\
Nadam & $c_{\mathrm{Adam}}(1+2\beta_1)$ & $13.57/\eta$ & 2.09 & 2.00 & 100\% \\
RMSProp & $1$ & $2/\eta$ & 3.89 & 2.07 & 77\% \\
RMSProp + momentum & $1/(1+\mu)$ & $3.8/\eta$ & 3.80 & 2.04 & 37\% \\
AMSGrad & $c_{\mathrm{Adam}}$ & $38/\eta$ & 2.00 & 1.98 & 100\% \\
Adagrad & $1$ & $2/\eta$ & 2.08 & 2.03 & 100\% \\
 
\bottomrule
\end{tabular}
\end{table}

\begin{figure}[t]
\centering
\includegraphics[width=\textwidth]{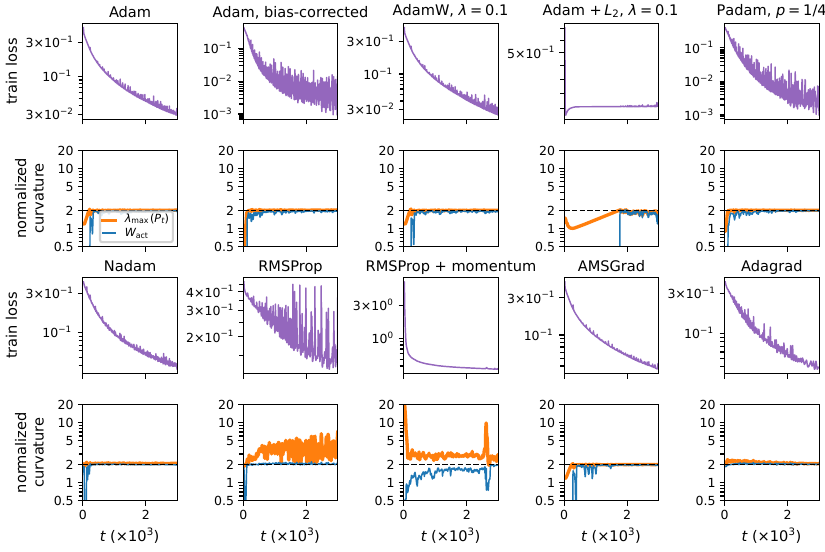}
\caption{The ten variants of Table~\ref{tab:family} on the width-$200$ MLP
($5000$ CIFAR-10 images, MSE, full batch, $3000$ steps), one column per
optimizer: training loss (purple) above $\lambda_{\max}(P_t)$ (orange) and
$W_{\mathrm{act}}$ (blue), $31$-step rolling medians, each normalized by the
variant's own constant $c$ so that the threshold (dashed) is $2$ for all.
AdamW and $L_2$ use decay $0.1$; Padam and Adagrad run at $\eta=10^{-2}$,
the others at $10^{-3}$.}
\label{fig:family}
\end{figure}

\subsubsection{Learning-rate decay}
\label{app:lrdecay}

Figure~\ref{fig:lrdecay} decays $\eta$ exponentially over steps
$2000$--$5000$ on width-$512$ MLPs at $\beta_1=0$ and $0.9$, with both
readings normalized by the instantaneous $c\eta_t$.  Runs that are at the
edge before the decay are back at $W_{\mathrm{act}}=2$ after it; at $\beta_1=0$ they hold $2$ throughout the decay, while at $\beta_1=0.9$ with $\eta:10^{-3}\to2\cdot10^{-4}$ the reading dips briefly inside the window.
The exception is $\beta_1=0.9$ with $\eta:10^{-2}\to10^{-3}$, which had not
reached the edge: $W_{\mathrm{act}}$ approaches $2$ and then drops well below
it, climbing back only slowly after the decay ($1.9$ at the end), while the constant-rate twin (not shown) settles at the edge.  Decay thus changes
when the edge is reached, not whether the regulation holds once it is
reached; at constant rate every $\eta\in[10^{-5},10^{-2}]$ reaches the
edge on the width-$200$ network of Table~\ref{tab:heldout}.

\begin{figure}[t]
\centering
\includegraphics[width=\textwidth]{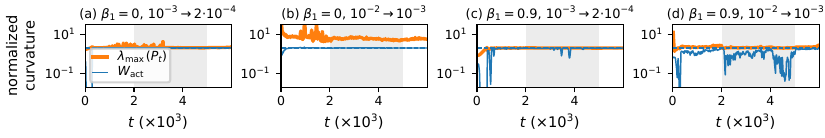}
\caption{Width-$512$ MLPs with an exponential learning-rate decay over the
shaded window (steps $2000$--$5000$), $31$-step rolling medians of
$\lambda_{\max}(P_t)$ (orange) and $W_{\mathrm{act}}$ (blue) normalized by
the instantaneous $c\eta_t$. }
\label{fig:lrdecay}
\end{figure}

\subsubsection{Batch size}
\label{app:noise}

Figures~\ref{fig:noise-mlp-b09}--\ref{fig:noise-gpt2-b09} subsample
a fixed training set to a fraction $B/N$ at each step: a width-$200$ MLP
($N=5000$, MSE, $30{,}000$ steps) and GPT-2 medium ($N=64$ sequences,
label-smoothed cross-entropy, $8000$ steps).  Table~\ref{tab:noise} gives
medians over steps $20{,}000$--$30{,}000$ (MLP) and $4000$--$8000$ (GPT-2).
Each run has two readings.  The full-set reading uses $g_t$ and $H_t$ as in
Section~\ref{sec:setup-main} and is comparable with the full-batch runs.
The minibatch reading uses the step's own gradient $g_{B,t}$ and Hessian
$H_{B,t}$ in the same formulas,
\[
 c\eta\,\frac{g_{B,t}^{\top}D_t^{-1}H_{B,t}D_t^{-1}g_{B,t}}{g_{B,t}^{\top}D_t^{-1}g_{B,t}}
 \qquad\text{and}\qquad
 \lambda_{\max}\bigl(c\eta\,D_t^{-1/2}H_{B,t}D_t^{-1/2}\bigr).
\]
On the MLP the full-set sharpness is a $15$-step Lanczos iteration every
$20$ steps; on GPT-2 it is recomputed only at saved checkpoints.

\begin{figure}[t]
\centering
\includegraphics[width=\textwidth]{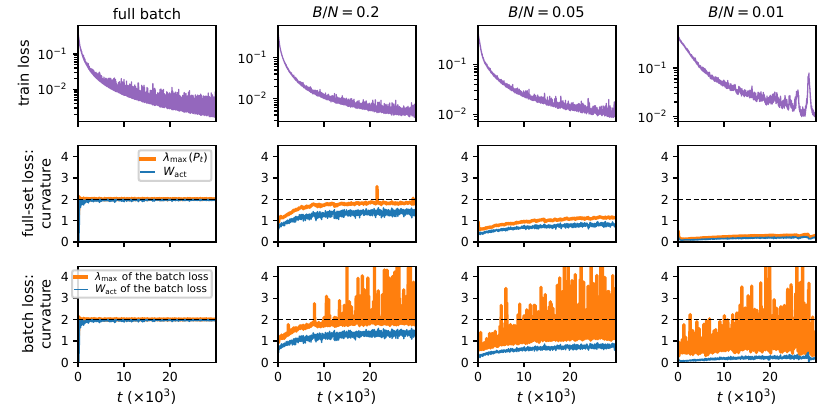}
\caption{Batch-size scan on the width-$200$ MLP, $\beta_1=0.9$,
$\eta=10^{-3}$, MSE, $30{,}000$ steps: minibatch training, in which Adam
sees at each step the gradient of a random subset of the fixed training set
($N=5000$ images) of size $B$, for $B/N=1$ (full batch), $0.2$, $0.05$ and
$0.01$ (four of the six fractions of Table~\ref{tab:noise}).  Top: the
training loss on all $N$ images (purple).  Middle: the two readings on the
\emph{full-set} loss of that minibatch trajectory, the sharpness
$\lambda_{\max}(P_t)$ of the full-set Hessian (orange, Lanczos every $20$
steps) and the active curvature $W_{\mathrm{act}}$ from the full-set
gradient and Hessian (blue).  Bottom: the same two readings on the loss of
the step's own \emph{minibatch}, its Hessian and its gradient
($\lambda_{\max}$ orange; $W_{\mathrm{act}}$
blue, every step).}
\label{fig:noise-mlp-b09}
\end{figure}

\begin{figure}[t]
\centering
\includegraphics[width=\textwidth]{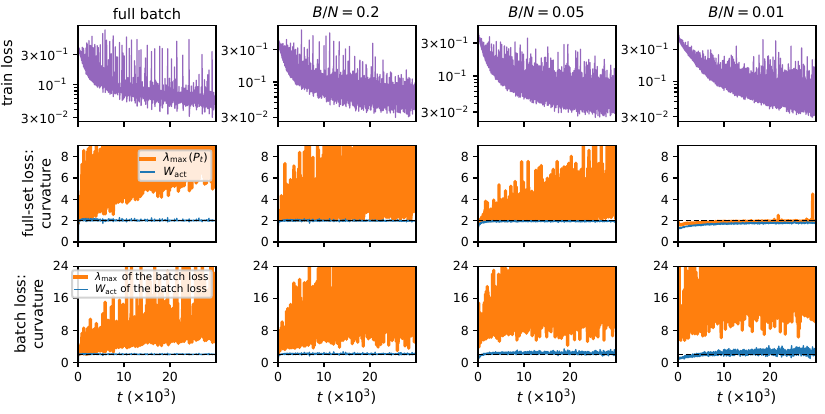}
\caption{Batch-size scan on the width-$200$ MLP at $\beta_1=0$ (otherwise
as Figure~\ref{fig:noise-mlp-b09}: minibatch training on random subsets of
the fixed $N=5000$ images, $B/N=1$, $0.2$, $0.05$ and $0.01$, $\eta=10^{-3}$,
MSE, $30{,}000$ steps).  Top: training loss on all $N$ images (purple).
Middle: sharpness and active curvature
(blue) of the full-set loss along the minibatch trajectory.  Bottom: the same two readings on the loss
of the step's own minibatch.}
\label{fig:noise-mlp-b0}
\end{figure}

\begin{figure}[t]
\centering
\includegraphics[width=\textwidth]{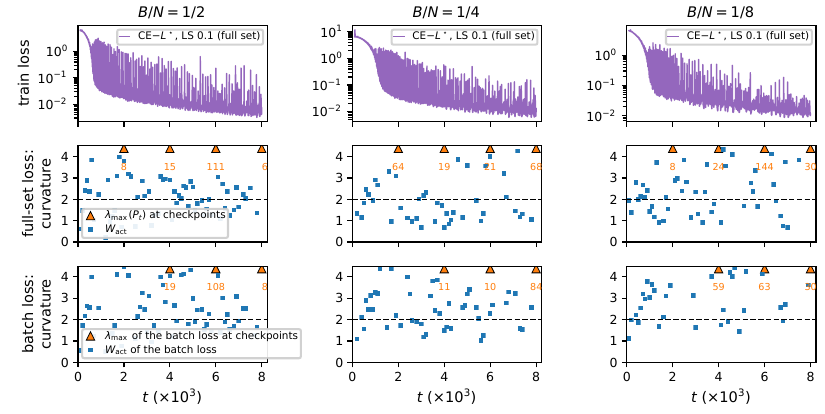}
\caption{Batch-size scan on GPT-2 medium, $\beta_1=0$, $\beta_2=0.95$,
$8000$ steps: minibatch training in which Adam sees at each step the
gradient of $B=32$, $16$ or $8$ of the fixed $64$ training sequences
($B/N=\tfrac12,\tfrac14,\tfrac18$; the full-batch run, $B/N=1$, is
Figure~\ref{fig:gpt2-wact}(a)).  Top: the training loss on all $64$
sequences, as excess over the label-smoothing floor $1.4076$ (purple).
Middle: the readings on the \emph{full-set} loss of the minibatch
trajectory, the exact active curvature $W_{\mathrm{act}}$ from the full-set
gradient and Hessian-vector product (blue squares, one reading every $100$
steps) and the sharpness $\lambda_{\max}(P_t)$ of the full-set Hessian,
which was not recorded along the run and is recomputed at the saved
checkpoints $x_{2000}$, $x_{4000}$, $x_{6000}$ and $x_{8000}$ (orange circles where it fits the axis, orange triangles at the top
edge with the value written below them where it does not).  Bottom: the
same two readings on the loss of the step's own \emph{minibatch}:
$W_{\mathrm{act}}$ of the batch loss (blue squares, one sample every $100$
steps) and its sharpness recomputed at $x_{4000}$, $x_{6000}$ and
$x_{8000}$ on two random batches of the run's size each (orange, median of
the two).  Dashed lines at $2$.}
\label{fig:noise-gpt2-b0}
\end{figure}

\begin{figure}[t]
\centering
\includegraphics[width=\textwidth]{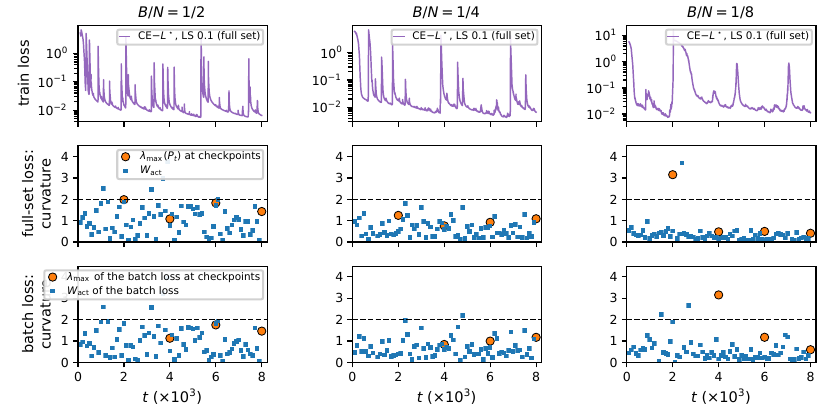}
\caption{Batch-size scan on GPT-2 medium at $\beta_1=0.9$, $\beta_2=0.999$
(otherwise as Figure~\ref{fig:noise-gpt2-b0}: minibatch training on $32$,
$16$ or $8$ of the fixed $64$ sequences per step, $8000$ steps; the
full-batch run is Figure~\ref{fig:gpt2-wact}(b)).  Top: training loss on
all $64$ sequences, excess over the floor $1.4076$ (purple).  Middle:
full-set readings, exact $W_{\mathrm{act}}$ (blue squares, every $100$
steps) and the full-set sharpness recomputed at the checkpoints
$x_{2000}$--$x_{8000}$ (orange circles).  Bottom: the readings on the
step's own minibatch loss, $W_{\mathrm{act}}$ of the batch loss (blue
squares, every $100$ steps) and its sharpness at the checkpoints
$x_{4000}$--$x_{8000}$ (orange circles).
Dashed lines at $2$.}
\label{fig:noise-gpt2-b09}
\end{figure}

\begin{table}[t]
\caption{Batch-fraction scan: late-window medians of the active curvature
and of the sharpness.  MLP: width $200$, $N=5000$, steps
$20{,}000$--$30{,}000$.  GPT-2 medium: fixed $64$ sequences, exact full-set
$W_{\mathrm{act}}$, steps $4000$--$8000$.  The full-batch GPT-2 rows are the
runs of Appendix~\ref{app:gpt2} over the same window; on the subsampled runs
the sharpness was not recorded during training and is recomputed at their
saved checkpoints $x_{4000}$, $x_{6000}$ and $x_{8000}$.  Minibatch-loss columns: the same two readings taken on
the loss of a minibatch of the run's size, with its own gradient and
Hessian.  On the MLP both are logged along the run ($W_{\mathrm{act}}$ every
step, $\lambda_{\max}$ by Lanczos every $20$ steps, on the step's own
batch).}
\label{tab:noise}
\centering
\footnotesize
\setlength{\tabcolsep}{4pt}
\begin{tabular}{@{}llccccc@{}}
\toprule
 & & & \multicolumn{2}{c}{full-set loss} & \multicolumn{2}{c}{minibatch loss}\\
model & $\beta_1$ & $B/N$ & $W_{\mathrm{act}}$ & $\lambda_{\max}(P_t)$ & $W_{\mathrm{act}}$ & $\lambda_{\max}$\\
\midrule
GPT-2 & 0 & full & 2.14 & 2.76 & 2.14 & 2.76 \\
GPT-2 & 0 & 1/2 & 2.11 & 15.09 & 2.59 & 19.18 \\
GPT-2 & 0 & 1/4 & 3.58 & 21.12 & 3.36 & 10.80 \\
GPT-2 & 0 & 1/8 & 4.22 & 29.69 & 6.36 & 58.58 \\
GPT-2 & 0.9 & full & 1.82 & 1.91 & 1.82 & 1.91 \\
GPT-2 & 0.9 & 1/2 & 0.94 & 1.43 & 0.63 & 1.46 \\
GPT-2 & 0.9 & 1/4 & 0.46 & 0.92 & 0.57 & 1.00 \\
GPT-2 & 0.9 & 1/8 & 0.18 & 0.48 & 0.29 & 1.17 \\
MLP & 0 & full & 2.02 & 8.16 & 2.02 & 8.16 \\
MLP & 0 & 0.2 & 2.00 & 5.87 & 2.05 & 15.11 \\
MLP & 0 & 0.05 & 1.96 & 3.62 & 2.22 & 19.92 \\
MLP & 0 & 0.01 & 1.80 & 1.95 & 2.53 & 20.60 \\
MLP & 0.9 & full & 1.98 & 2.01 & 1.98 & 2.01 \\
MLP & 0.9 & 0.7 & 1.71 & 1.98 & 1.72 & 2.04 \\
MLP & 0.9 & 0.4 & 1.55 & 1.93 & 1.54 & 1.99 \\
MLP & 0.9 & 0.2 & 1.39 & 1.82 & 1.36 & 1.88 \\
MLP & 0.9 & 0.05 & 0.81 & 1.12 & 0.74 & 1.36 \\
MLP & 0.9 & 0.01 & 0.21 & 0.31 & 0.22 & 1.08 \\
 
\bottomrule
\end{tabular}
\end{table}

\section{Gradient reversal across settings}
\label{app:reversal}

Section~\ref{sec:reversal} predicts that at the edge the full-batch
gradient reverses at every step, $\cos(g_t,g_{t-1})\to-1$, and
Section~\ref{sec:obs-flip} shows this on the width-$200$ MLP and GPT-2
medium (Figures~\ref{fig:wact-flip}, \ref{fig:flip} and~\ref{fig:flip-lr}).  This appendix gives
the protocol of those runs (\ref{app:flip-w200}) and the reversal on the
small network under sixteen changes (\ref{app:flip-robust}), the other
architectures (\ref{app:reversal-zoo}), GPT-2 (\ref{app:reversal-gpt2}), the
Adam family (\ref{app:reversal-family}) and the batch-size scan
(\ref{app:reversal-noise}) of Appendix~\ref{app:experiments}.  Cosines are
between consecutive full-batch quantities, curves are $101$-step rolling
medians, and ``reversal steps'' is the fraction of steps with
$\cos(g_t,g_{t-1})<-0.5$.  In Figure~\ref{fig:wact-flip}, $W_{\mathrm{act}}$ is a $31$-step rolling median in (a--e) and, in (f), the exact reading every $50$ steps with its rolling median over about $550$ steps; (f) uses the bias-corrected AdamW of Appendix~\ref{app:gpt2}, the other five runs uncorrected Adam.

\subsection{Fully connected networks: the runs of Section~\ref{sec:obs-flip}}
\label{app:flip-w200}

Figures~\ref{fig:flip-lr} and~\ref{fig:averaging} use the width-$200$
network of Appendix~\ref{app:nn-sharpness} ($5000$ CIFAR-10 images, MSE,
full batch) with uncorrected Adam, $\beta_2=0.999$, $\eps=10^{-8}$, seed
$0$, for $30{,}000$ steps at a constant learning rate.  Along each run we
record the loss at the iterate, at the midpoint $\tfrac12(x_t+x_{t-1})$ and
at the EMA iterate (decay $0.99$), and the cosines of consecutive gradients,
momenta and displacements.

\begin{figure}[t]
\centering
\includegraphics[width=\textwidth]{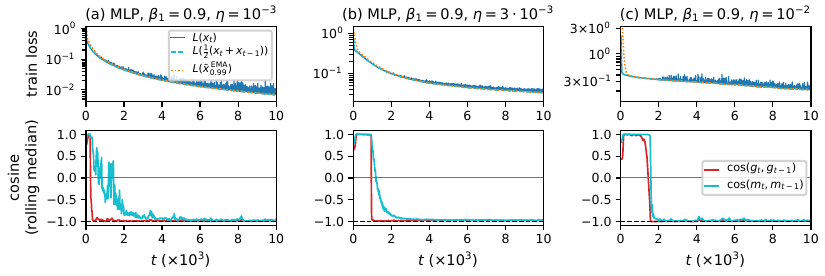}
\caption{Width-$200$ fc-tanh MLP on CIFAR-10, full-batch MSE,
$\beta_1=0.9$, $\beta_2=0.999$, at $\eta=10^{-3}$, $3\cdot10^{-3}$ and
$10^{-2}$ (first $10{,}000$ steps).  Top: training loss at the iterate
$L(x_t)$ (blue, solid), at the midpoint of consecutive iterates (cyan,
dashed) and at the EMA iterate (orange, dotted).  Bottom:
$\cos(g_t,g_{t-1})$ (red) and $\cos(m_t,m_{t-1})$ (cyan). The run protocol and the
numbers are in Appendix~\ref{app:flip-w200}.}
\label{fig:flip-lr}
\end{figure}

\begin{figure}[t]
\centering
\includegraphics[width=\textwidth]{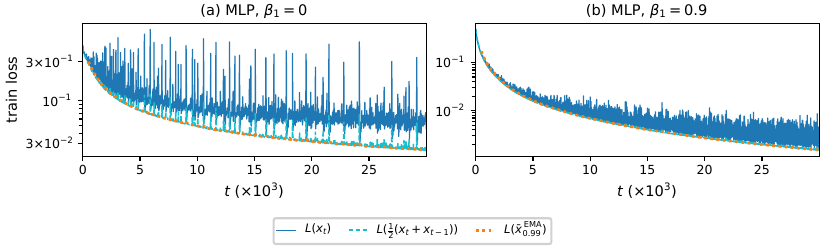}
\caption{Width-$200$
fc-tanh MLP on CIFAR-10, full-batch MSE, $\eta=10^{-3}$, at $\beta_1=0$ (a)
and $\beta_1=0.9$ (b): loss at the iterate $x_t$ (blue, solid), at the
midpoint of consecutive iterates (cyan, dashed) and at the EMA of the
iterate with rate $0.99$ (orange, dotted).  The
orbit height $E_{\mathrm{osc}}=L(x_t)-L(\mathrm{mid})$ settles to a constant, at once at $\beta_1=0$ and after about $3000$ steps at $\beta_1=0.9$, while
the loss keeps falling (Figure~\ref{fig:orbit-height}); details are in Appendix~\ref{app:flip-w200}.}
\label{fig:averaging}
\end{figure}

\paragraph{Figure~\ref{fig:averaging}: the fixed-height orbit.}
Panels (a) and (b) are the $\eta=10^{-3}$ runs at $\beta_1=0$ and
$0.9$ (EMA drawn from step $500$; it starts at the first iterate and lags early on).  Over the last
$10{,}000$ steps the midpoint and the EMA iterate reach about half the loss
at the iterate at $\beta_1=0$ and about two thirds at $\beta_1=0.9$.  The
orbit height $E_{\mathrm{osc}}=L(x_t)-L(\tfrac12(x_t+x_{t-1}))$
(Figure~\ref{fig:orbit-height}) is flat from step $1000$ on at $\beta_1=0$ and from about step $3000$ on at $\beta_1=0.9$, as is the gradient norm,
while the loss keeps falling.

\begin{figure}[t]
\centering
\includegraphics[width=0.5\textwidth]{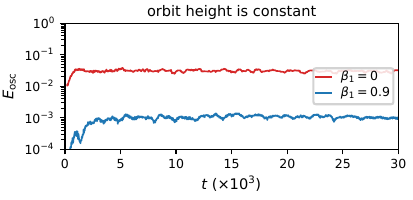}
\caption{The orbit height
$E_{\mathrm{osc}}=L(x_t)-L(\tfrac12(x_t+x_{t-1}))$ of the two runs of
Figure~\ref{fig:averaging}, as a $501$-step rolling median: $\beta_1=0$
(red) and $\beta_1=0.9$ (blue).  It is flat from step $1000$ ($\beta_1=0$) and from about step $3000$ ($\beta_1=0.9$) on, while the
loss at the iterate keeps falling, by a factor of $6$ or more.}
\label{fig:orbit-height}
\end{figure}

\subsection{Fully connected network: sixteen one-factor changes}
\label{app:flip-robust}

Figures~\ref{fig:flip-scan}--\ref{fig:flip-scan-net} and
Table~\ref{tab:scan} use the small network $3072{-}64{-}64{-}10$ (tanh,
$5000$ CIFAR-10 images, MSE, full batch, $4000$ steps) at $\beta_1=0.9$,
$\beta_2=0.999$, $\eta=10^{-3}$, $\eps=10^{-8}$, seed $0$, and change one
factor at a time: $\beta_1$ (Figure~\ref{fig:flip-scan}), $\beta_2$ and
$\eps$ (Figure~\ref{fig:flip-scan-b2}), $\eta$ (Figure~\ref{fig:flip-scan-lr}),
and the activation, dataset, training-set size and depth
(Figure~\ref{fig:flip-scan-net}).  In all sixteen settings
the median gradient cosine over steps $2000$--$4000$ is strongly negative,
close to $-1$.

\begin{figure}[t]
\centering
\includegraphics[width=\textwidth]{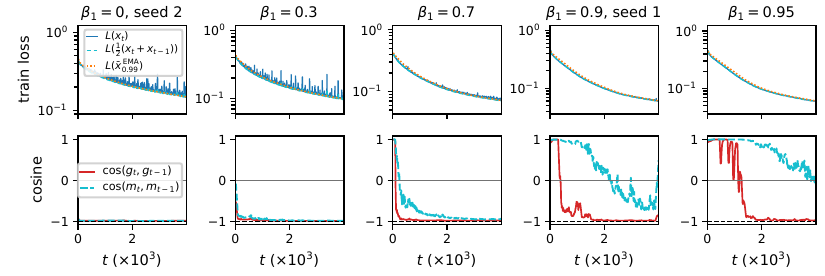}
\caption{Momentum.  The small network at the reference setting of
Appendix~\ref{app:flip-robust} with $\beta_1\in\{0,0.3,0.7,0.9,0.95\}$; the
$\beta_1=0$ and $\beta_1=0.9$ runs use seeds $2$ and $1$, the others seed
$0$.  Top: training loss at the
iterate (blue, solid), at the midpoint of consecutive iterates
$\tfrac12(x_t+x_{t-1})$ (cyan, dashed) and at the EMA of the iterate with
rate $0.99$ (orange, dotted).  Bottom: $101$-step rolling medians of $\cos(g_t,g_{t-1})$ (red) and
$\cos(m_t,m_{t-1})$ (cyan, dashed); the dashed line is $-1$.}
\label{fig:flip-scan}
\end{figure}

\begin{figure}[t]
\centering
\includegraphics[width=\textwidth]{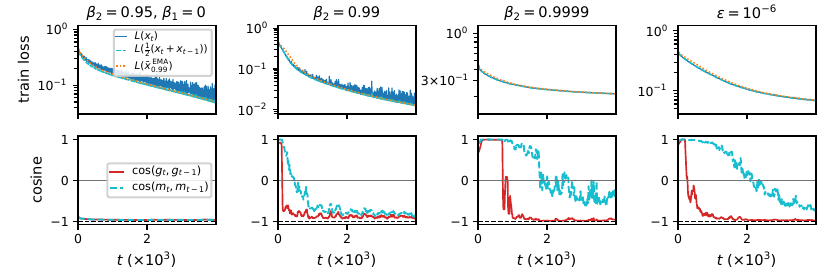}
\caption{Second-moment constants.  The network and readings of
Figure~\ref{fig:flip-scan} with $\beta_2=0.95$ (at $\beta_1=0$), $0.99$
and $0.9999$, and with $\eps=10^{-6}$.}
\label{fig:flip-scan-b2}
\end{figure}

\begin{figure}[t]
\centering
\includegraphics[width=\textwidth]{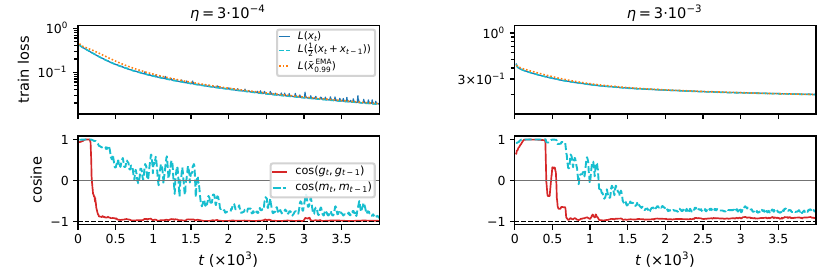}
\caption{Learning rate.  The network and readings of
Figure~\ref{fig:flip-scan} at $\eta=3\cdot10^{-4}$ and $3\cdot10^{-3}$.}
\label{fig:flip-scan-lr}
\end{figure}

\begin{figure}[t]
\centering
\includegraphics[width=\textwidth]{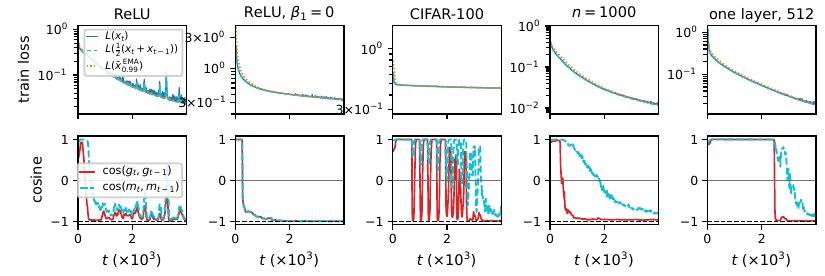}
\caption{Network and data.  The readings of Figure~\ref{fig:flip-scan} with
ReLU (at $\beta_1=0.9$ and $0$), on CIFAR-100, with $n=1000$ training
images, and with a single hidden layer of $512$.}
\label{fig:flip-scan-net}
\end{figure}

\begin{table}[t]
\caption{The sixteen settings of
Figures~\ref{fig:flip-scan}--\ref{fig:flip-scan-net}, grouped as the
figures: medians over steps $2000$--$4000$ of $\cos(g_t,g_{t-1})$,
$\cos(g_t,g_{t-2})$ and $\cos(m_t,m_{t-1})$, the fraction of steps with
$\cos(g_t,g_{t-1})<-0.5$, and the training loss (median over steps $3800$--$3999$).}
\label{tab:scan}
\centering
\footnotesize
\setlength{\tabcolsep}{4pt}
\begin{tabular}{@{}lcccrrrrr@{}}
\toprule
setting & $\beta_1$ & $\beta_2$ & $\eta$ & $\cos(g_t,g_{t-1})$ & $\cos(g_t,g_{t-2})$ & $\cos(m_t,m_{t-1})$ & reversal steps & loss \\
\midrule
seed 2, $\beta_1{=}0$ & 0 & 0.999 & 0.001 & -0.993 & +0.990 & -0.993 & 100\% & 0.156 \\
$\beta_1{=}0.3$ & 0.3 & 0.999 & 0.001 & -0.992 & +0.982 & -0.985 & 100\% & 0.102 \\
$\beta_1{=}0.7$ & 0.7 & 0.999 & 0.001 & -0.982 & +0.938 & -0.941 & 100\% & 0.0742 \\
seed 1 & 0.9 & 0.999 & 0.001 & -0.982 & +0.943 & -0.402 & 99\% & 0.0611 \\
$\beta_1{=}0.95$ & 0.95 & 0.999 & 0.001 & -0.971 & +0.908 & +0.480 & 98\% & 0.0608 \\
\midrule
$\beta_2{=}0.95$, $\beta_1{=}0$ & 0 & 0.95 & 0.001 & -0.963 & +0.904 & -0.963 & 100\% & 0.062 \\
$\beta_2{=}0.99$ & 0.9 & 0.99 & 0.001 & -0.891 & +0.553 & -0.801 & 89\% & 0.0152 \\
$\beta_2{=}0.9999$ & 0.9 & 0.9999 & 0.001 & -0.977 & +0.938 & -0.347 & 98\% & 0.216 \\
$\eps{=}10^{-6}$ & 0.9 & 0.999 & 0.001 & -0.972 & +0.905 & -0.466 & 100\% & 0.0696 \\
\midrule
$\eta{=}3{\cdot}10^{-4}$ & 0.9 & 0.999 & 0.0003 & -0.985 & +0.944 & -0.753 & 97\% & 0.0196 \\
$\eta{=}3{\cdot}10^{-3}$ & 0.9 & 0.999 & 0.003 & -0.927 & +0.796 & -0.734 & 99\% & 0.201 \\
\midrule
ReLU & 0.9 & 0.999 & 0.001 & -0.910 & +0.697 & -0.838 & 88\% & 0.0264 \\
ReLU, $\beta_1{=}0$ & 0 & 0.999 & 0.001 & -0.997 & +0.999 & -0.997 & 100\% & 0.34 \\
CIFAR-100 & 0.9 & 0.999 & 0.001 & -0.942 & +0.933 & -0.223 & 73\% & 0.458 \\
$n{=}1000$ & 0.9 & 0.999 & 0.001 & -0.963 & +0.866 & -0.645 & 100\% & 0.0126 \\
one layer, 512 & 0.9 & 0.999 & 0.001 & -0.969 & +0.989 & +0.006 & 74\% & 0.0204 \\
 
\bottomrule
\end{tabular}
\end{table}

\subsection{Residual, attention and recurrent networks}
\label{app:reversal-zoo}

At $\beta_1=0$ all three networks reverse at every step once at the
edge (Figures~\ref{fig:zoo-flip-b0} and~\ref{fig:zoo-flip-b09},
Table~\ref{tab:zoo-flip}): $\cos(g_t,g_{t-1})$ is close to $-1$ and
$\cos(g_t,g_{t-2})$ close to $+1$; on the LSTM, loss spikes interrupt the
orbit.  At $\beta_1=0.9$ the vision networks reverse and their momentum
inherits it, while the LSTM, still off the edge, reverses on only part of its
steps and its momentum not at all.

\begin{figure}[t]
\centering
\includegraphics[width=\textwidth]{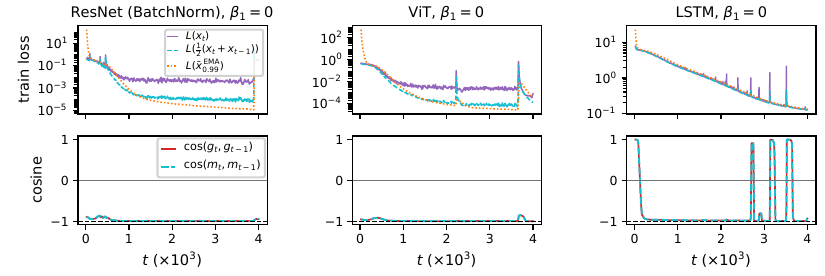}
\caption{Reversal on the three runs of Figure~\ref{fig:zoo-b0}
($\beta_1=0$).  Top: training loss at the iterate (purple, solid), at the
midpoint of consecutive iterates (cyan, dashed) and at the EMA of the
iterate with rate $0.99$ (orange, dotted); MSE for the vision networks, excess cross-entropy over the
floor $1.4075$ for the LSTM.  Bottom: $101$-step rolling medians of
$\cos(g_t,g_{t-1})$ (red, solid) and $\cos(m_t,m_{t-1})$ (cyan, dashed);
dashed line at $-1$.  At $\beta_1=0$ the momentum is the gradient,
$m_{t+1}=g_t$, so the two curves coincide.}
\label{fig:zoo-flip-b0}
\end{figure}

\begin{figure}[t]
\centering
\includegraphics[width=\textwidth]{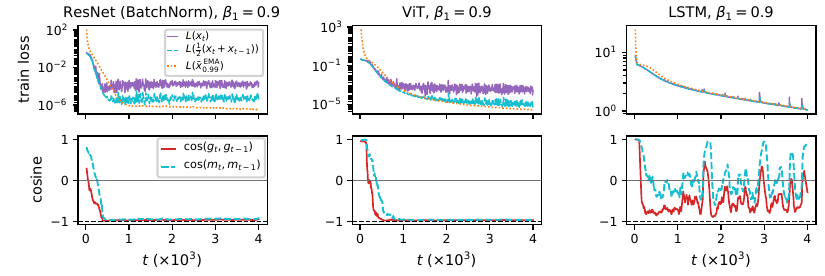}
\caption{As Figure~\ref{fig:zoo-flip-b0} for the three runs of
Figure~\ref{fig:zoo-b09} ($\beta_1=0.9$).}
\label{fig:zoo-flip-b09}
\end{figure}

\begin{table}[t]
\caption{Second-half medians (steps $2000$--$3999$) of the cosines for the
six runs of Figures~\ref{fig:zoo-b0} and~\ref{fig:zoo-b09}, and the fraction of steps with
$\cos(g_t,g_{t-1})<-0.5$.}
\label{tab:zoo-flip}
\centering
\scriptsize
\setlength{\tabcolsep}{4pt}
\begin{tabular}{@{}lcrrrr@{}}
\toprule
architecture & $\beta_1$ & $\cos(g_t,g_{t-1})$ & $\cos(g_t,g_{t-2})$ & $\cos(m_t,m_{t-1})$ & reversal steps\\
\midrule
ResNet-BN & 0 & -0.998 & +0.994 & -0.998 & 100\% \\
ResNet-BN & 0.9 & -0.966 & +0.872 & -0.944 & 100\% \\
ViT & 0 & -0.996 & +0.990 & -0.996 & 100\% \\
ViT & 0.9 & -0.975 & +0.905 & -0.966 & 100\% \\
LSTM & 0 & -0.986 & +0.991 & -0.986 & 79\% \\
LSTM & 0.9 & -0.498 & -0.286 & -0.027 & 50\% \\
 
\bottomrule
\end{tabular}
\end{table}

\subsection{GPT-2 medium}
\label{app:reversal-gpt2}

Table~\ref{tab:gpt2-flip} and Figure~\ref{fig:gpt2-reversal} give the
reversal on the three runs of Appendix~\ref{app:gpt2};
Figure~\ref{fig:flip-gpt2-cos} gives the traces of the two runs of
Figure~\ref{fig:flip}.  At $\beta_1=0$ the reversal is clean: the gradient
cosine stays close to $-1$ and most steps reverse.  At $\beta_1=0.9$,
$\beta_2=0.999$ it is intermittent, and the momentum inherits it only
partly.

\begin{table}[t]
\caption{Gradient reversal on the three full-batch GPT-2 runs of
Appendix~\ref{app:gpt2}: median $\cos(g_t,g_{t-1})$
and $\cos(m_t,m_{t-1})$, and fraction of steps with $\cos(g_t,g_{t-1})<-0.5$
over every step of the window; at $\beta_1=0$ the momentum is the
gradient.}
\label{tab:gpt2-flip}
\centering
\scriptsize
\setlength{\tabcolsep}{4pt}
\begin{tabular}{@{}cclccc@{}}
\toprule
$\beta_1$ & $\beta_2$ & steps & $\cos(g_t,g_{t-1})$ & $\cos(m_t,m_{t-1})$ & reversal\\
\midrule
$0$ & $0.95$ & $2{,}000$--$4{,}000$ & $-0.97$ & $-0.97$ & $89\%$ \\
 &  & $4{,}000$--$8{,}000$ & $-0.98$ & $-0.98$ & $93\%$ \\
 &  & $8{,}000$--$8{,}600$ & $-0.99$ & $-0.99$ & $94\%$ \\
$0.9$ & $0.999$ & $2{,}000$--$8{,}000$ & $-0.90$ & $0.36$ & $61\%$ \\
 &  & $8{,}000$--$12{,}000$ & $-0.97$ & $-0.41$ & $78\%$ \\
 &  & $12{,}000$--$20{,}000$ & $-0.93$ & $-0.30$ & $75\%$ \\
 &  & $20{,}000$--$30{,}000$ & $-0.82$ & $-0.37$ & $73\%$ \\
$0.9$ & $0.95$ & $2{,}000$--$8{,}000$ & $-0.01$ & $0.71$ & $19\%$ \\
 &  & $8{,}000$--$12{,}000$ & $0.01$ & $0.36$ & $5\%$ \\
 &  & $12{,}000$--$20{,}000$ & $0.09$ & $0.49$ & $2\%$ \\
 &  & $20{,}000$--$30{,}000$ & $0.12$ & $0.47$ & $2\%$ \\
 
\bottomrule
\end{tabular}
\end{table}

\begin{figure}[t]
\centering
\includegraphics[width=\textwidth]{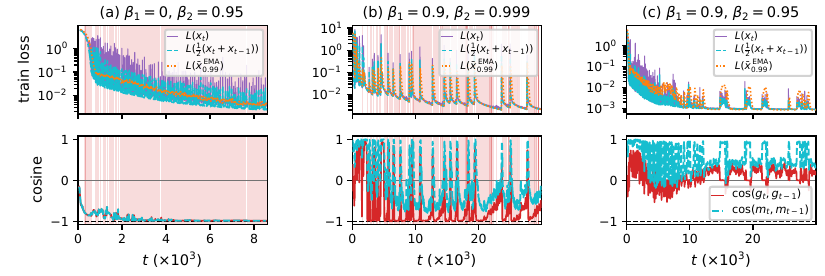}
\caption{Gradient reversal along the three full-batch GPT-2 runs of
Figure~\ref{fig:gpt2-wact}: (a) $\beta_1=0$, $\beta_2=0.95$, $8600$ steps;
(b) $\beta_1=0.9$, $\beta_2=0.999$, and (c) $\beta_1=0.9$, $\beta_2=0.95$,
$30{,}000$ steps.  Top: training loss, excess over the label-smoothing
floor $1.4076$, at the iterate (purple, solid), at the midpoint of
consecutive iterates (cyan, dashed) and at the EMA of the iterate with
rate $0.99$ (orange, dotted).  Bottom: $101$-step rolling medians of
$\cos(g_t,g_{t-1})$ (red, solid) and $\cos(m_t,m_{t-1})$ (cyan, dashed);
the dashed line is $-1$; at $\beta_1=0$ the two curves coincide.  Red shading in both
rows: stretches on which $\cos(g_t,g_{t-1})<-0.8$ on the majority of a
$51$-step window, as in Figure~\ref{fig:flip}; no stretch qualifies in
(c).}
\label{fig:gpt2-reversal}
\end{figure}

\begin{figure}[t]
\centering
\includegraphics[width=\textwidth]{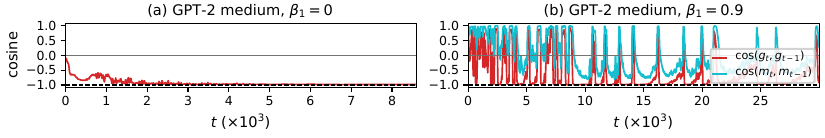}
\caption{Cosine traces of the two full-batch GPT-2 runs of
Figure~\ref{fig:flip}: $101$-step rolling medians of $\cos(g_t,g_{t-1})$
(red) and $\cos(m_t,m_{t-1})$ (cyan); the dashed line is $-1$.
(a) $\beta_1=0$, $\beta_2=0.95$
(b) $\beta_1=0.9$, $\beta_2=0.999$}
\label{fig:flip-gpt2-cos}
\end{figure}

\subsection{The Adam family}
\label{app:reversal-family}

Figures~\ref{fig:family-flip} and~\ref{fig:family-flip-b0} and
Table~\ref{tab:family-flip} give the reversal of the ten variants of
Appendix~\ref{app:family}, and of six of them rerun at $\beta_1=0$.  The
gradient reverses in every configuration that reaches the edge and stays there,
including Adam, AdamW, AMSGrad, Nadam, Padam, Adagrad and RMSProp.  The
exception is Adam with coupled $L_2$ decay, which stays below the threshold.
RMSProp with momentum reaches the edge only late and starts to reverse near
the end of the run.

\begin{figure}[t]
\centering
\includegraphics[width=\textwidth]{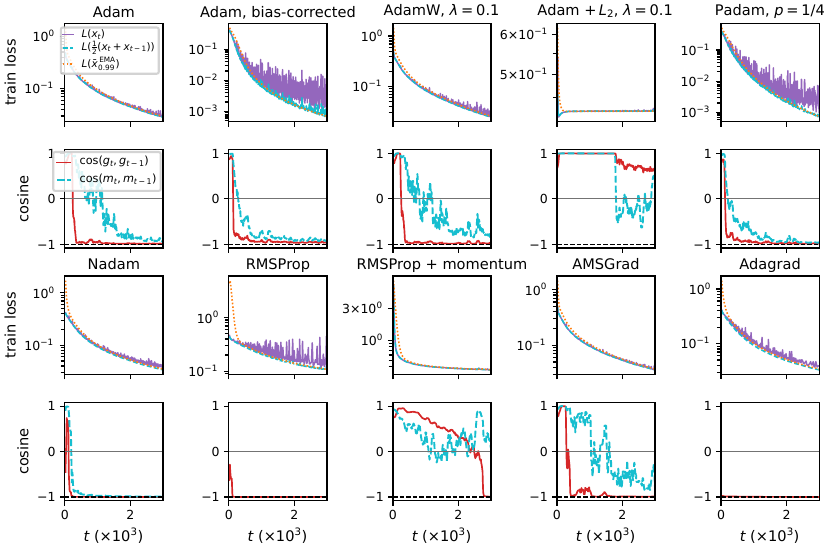}
\caption{Reversal on the ten configurations of Figure~\ref{fig:family}
($\beta_1=0.9$ or $\mu=0.9$), one column per optimizer in the order of
Table~\ref{tab:family}.  Top of each block: training loss at the iterate
(purple, solid), at the midpoint of consecutive iterates (cyan, dashed)
and at the EMA of the iterate with rate $0.99$ (orange, dotted).  Bottom:
$101$-step rolling medians of $\cos(g_t,g_{t-1})$ (red, solid) and
$\cos(m_t,m_{t-1})$ (cyan, dashed); dashed line at $-1$.  RMSProp and
Adagrad have no momentum and show the gradient cosine only.}
\label{fig:family-flip}
\end{figure}

\begin{figure}[t]
\centering
\includegraphics[width=\textwidth]{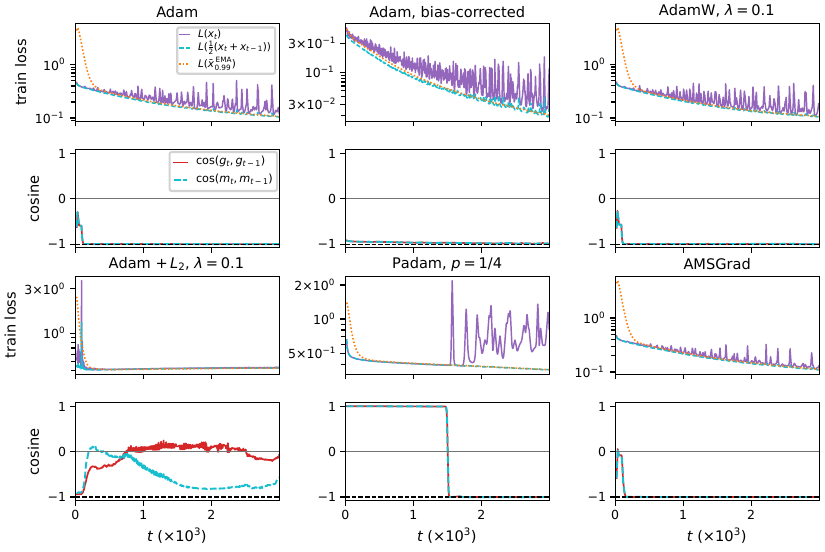}
\caption{As Figure~\ref{fig:family-flip} at $\beta_1=0$ for the six
variants that are run again there (Table~\ref{tab:family-flip}); $m_{t+1}=g_t$,
so the two cosine curves coincide.}
\label{fig:family-flip-b0}
\end{figure}

\begin{table}[t]
\caption{Reversal of the ten variants of Table~\ref{tab:family}, last
quarter of the run (steps $2250$--$2999$): medians of $\cos(g_t,g_{t-1})$,
$\cos(g_t,g_{t-2})$ and $\cos(m_t,m_{t-1})$, the fraction of steps with
$\cos(g_t,g_{t-1})<-0.5$ (``reversal''), and the means of the normalized
$W_{\mathrm{act}}$ and $\lambda_{\max}$, the statistics of
Table~\ref{tab:family}.}
\label{tab:family-flip}
\centering
\scriptsize
\setlength{\tabcolsep}{2.6pt}
\begin{tabular}{@{}lrrrrrrrrrrr@{}}
\toprule
 & \multicolumn{6}{c}{$\beta_1=0.9$ or $\mu=0.9$} & \multicolumn{5}{c}{$\beta_1=0$}\\
\cmidrule(lr){2-7}\cmidrule(lr){8-12}
variant & \shortstack{$\cos(g_t,$\\$g_{t-1})$} & \shortstack{$\cos(g_t,$\\$g_{t-2})$} & \shortstack{$\cos(m_t,$\\$m_{t-1})$} & reversal & $W_{\mathrm{act}}$ & $\lambda_{\max}$ & \shortstack{$\cos(g_t,$\\$g_{t-1})$} & \shortstack{$\cos(g_t,$\\$g_{t-2})$} & reversal & $W_{\mathrm{act}}$ & $\lambda_{\max}$\\
\midrule
Adam (uncorrected) & -0.987 & +0.947 & -0.909 & 99\% & 1.97 & 2.01 & -0.998 & +0.995 & 100\% & 2.05 & 3.70 \\
Adam, bias-corrected & -0.959 & +0.832 & -0.905 & 99\% & 1.94 & 2.02 & -0.983 & +0.982 & 100\% & 2.28 & 2.84 \\
AdamW & -0.972 & +0.872 & -0.753 & 94\% & 1.92 & 2.00 & -0.996 & +0.994 & 100\% & 2.09 & 3.69 \\
Adam + $L_2$ & +0.677 & +0.942 & -0.128 & 1\% & 1.76 & 1.87 & -0.119 & +0.844 & 3\% & 1.94 & 2.90 \\
Padam ($p=1/4$) & -0.982 & +0.930 & -0.957 & 100\% & 1.97 & 2.01 & -1.000 & +0.999 & 100\% & 2.00 & 2.06 \\
Nadam & -0.998 & +0.992 & -0.989 & 100\% & 2.00 & 2.09 & -- & -- & -- & -- & -- \\
RMSProp & -0.998 & +0.995 & -- & 100\% & 2.06 & 3.95 & -- & -- & -- & -- & -- \\
RMSProp + momentum & -0.231 & +0.322 & +0.392 & 37\% & 2.04 & 3.80 & -- & -- & -- & -- & -- \\
AMSGrad & -0.993 & +0.981 & -0.698 & 100\% & 1.98 & 2.00 & -0.998 & +0.998 & 100\% & 2.06 & 2.94 \\
Adagrad & -0.998 & +0.998 & -- & 100\% & 2.03 & 2.08 & -- & -- & -- & -- & -- \\
 
\bottomrule
\end{tabular}
\end{table}

\subsection{Batch size}
\label{app:reversal-noise}

Figures~\ref{fig:noise-flip-mlp-b0}--\ref{fig:noise-flip-gpt2-b09},
Table~\ref{tab:noise-cos} and Figure~\ref{fig:noise-cos} give the reversal on
the subsampled runs of Appendix~\ref{app:noise}, over the same windows.At $\beta_1=0$ the reversal survives every batch size on the MLP and.  At $\beta_1=0.9$ it weakens as the batch shrinks,
together with $W_{\mathrm{act}}$, and disappears at milder subsampling on
GPT-2 than on the MLP.  Noise alone weakens the reversal only slowly; noise with momentum removes it.

\begin{figure}[t]
\centering
\includegraphics[width=\textwidth]{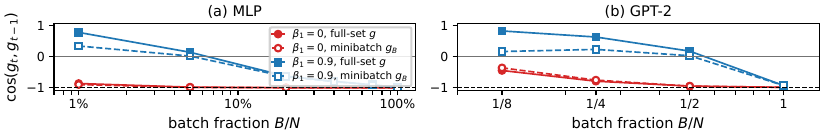}
\caption{Gradient cosine against the batch fraction $B/N$, late-window
medians, for (a) the width-$200$ MLP and (b) GPT-2 medium, at $\beta_1=0$
(red circles) and $\beta_1=0.9$ (blue squares).  Filled markers, solid
lines: consecutive full-set gradients; hollow markers, dashed lines:
consecutive minibatch gradients along the same runs (at $B/N=1$ the two
coincide).  Dashed horizontal line at $-1$.}
\label{fig:noise-cos}
\end{figure}

\begin{table}[t]
\caption{Batch-fraction scan, the reversal read on the full-set gradient
and on the minibatch gradient of the step: late-window medians of
$\cos(g_t,g_{t-1})$ and of $\cos(g_{B,t},g_{B,t-1})$, the fraction of steps
on which each is below $-0.5$ (``reversal''), and the alignment
$\cos(g_{B,t},g_t)$ of the batch gradient with the full gradient at the
same iterate.  MLP: width $200$, $N=5000$, steps $20{,}000$--$30{,}000$,
every column logged at every step.  GPT-2 medium: fixed $64$ sequences,
steps $4000$--$8000$.}
\label{tab:noise-cos}
\centering
\footnotesize
\setlength{\tabcolsep}{4pt}
\begin{tabular}{@{}llcccccc@{}}
\toprule
 & & & \multicolumn{2}{c}{full-set gradient} & \multicolumn{2}{c}{minibatch gradient} & \\
model & $\beta_1$ & $B/N$ & $\cos(g_t,g_{t-1})$ & reversal & $\cos(g_{B,t},g_{B,t-1})$ & reversal & $\cos(g_{B,t},g_t)$\\
\midrule
GPT-2 & 0 & full & -0.98 & 93\% & -0.98 & 93\% & +1.00 \\
GPT-2 & 0 & 1/2 & -0.94 & 90\% & -0.95 & 89\% & +0.95 \\
GPT-2 & 0 & 1/4 & -0.79 & 75\% & -0.75 & 70\% & +0.90 \\
GPT-2 & 0 & 1/8 & -0.44 & 43\% & -0.36 & 30\% & +0.62 \\
GPT-2 & 0.9 & full & -0.93 & 65\% & -0.93 & 65\% & +1.00 \\
GPT-2 & 0.9 & 1/2 & +0.18 & 28\% & +0.03 & 26\% & +0.88 \\
GPT-2 & 0.9 & 1/4 & +0.63 & 4\% & +0.24 & 3\% & +0.89 \\
GPT-2 & 0.9 & 1/8 & +0.83 & 0\% & +0.17 & 0\% & +0.48 \\
MLP & 0 & full & -1.00 & 100\% & -1.00 & 100\% & +1.00 \\
MLP & 0 & 0.2 & -1.00 & 100\% & -1.00 & 100\% & +1.00 \\
MLP & 0 & 0.05 & -0.98 & 100\% & -0.98 & 100\% & +0.99 \\
MLP & 0 & 0.01 & -0.86 & 99\% & -0.90 & 100\% & +0.95 \\
MLP & 0.9 & full & -0.99 & 100\% & -0.99 & 100\% & +1.00 \\
MLP & 0.9 & 0.7 & -0.90 & 99\% & -0.90 & 99\% & +1.00 \\
MLP & 0.9 & 0.4 & -0.79 & 92\% & -0.80 & 93\% & +0.99 \\
MLP & 0.9 & 0.2 & -0.63 & 69\% & -0.64 & 72\% & +0.99 \\
MLP & 0.9 & 0.05 & +0.14 & 3\% & +0.02 & 4\% & +0.95 \\
MLP & 0.9 & 0.01 & +0.78 & 0\% & +0.35 & 0\% & +0.76 \\
 
\bottomrule
\end{tabular}
\end{table}

\begin{figure}[t]
\centering
\includegraphics[width=\textwidth]{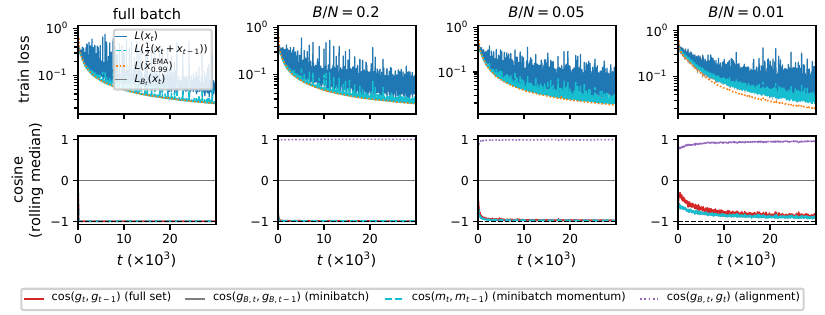}
\caption{Gradient reversal in the batch-size scan on the width-$200$ MLP,
$\beta_1=0$, the runs of Table~\ref{tab:noise-cos} ($\eta=10^{-3}$, MSE,
$30{,}000$ steps; at each step Adam sees a random subset of size $B$ of the
fixed $N=5000$ images, $B/N=1$, $0.2$, $0.05$, $0.01$).  Top: the training
loss on all $N$ images at the iterate (blue, solid), at the midpoint of
consecutive iterates (cyan, dashed) and at the EMA of the iterate with
rate $0.99$ (orange, dotted), and the loss of the step's own minibatch
(grey, thin).  Bottom: $101$-step rolling medians of the cosine of
consecutive full-set gradients $\cos(g_t,g_{t-1})$ (red), of consecutive
momenta $\cos(m_t,m_{t-1})$ (cyan, dashed; at $\beta_1=0$ the momentum is
the gradient, $m_{t+1}=g_t$, and the two curves coincide), of consecutive
minibatch gradients $\cos(g_{B,t},g_{B,t-1})$ (grey) and of the alignment
$\cos(g_{B,t},g_t)$ of the batch gradient with the full gradient at the
same iterate (purple, dotted).  Dashed line at $-1$.  In the full-batch
panel the minibatch is the full set and only the full-set curves are
drawn.}
\label{fig:noise-flip-mlp-b0}
\end{figure}

\begin{figure}[t]
\centering
\includegraphics[width=\textwidth]{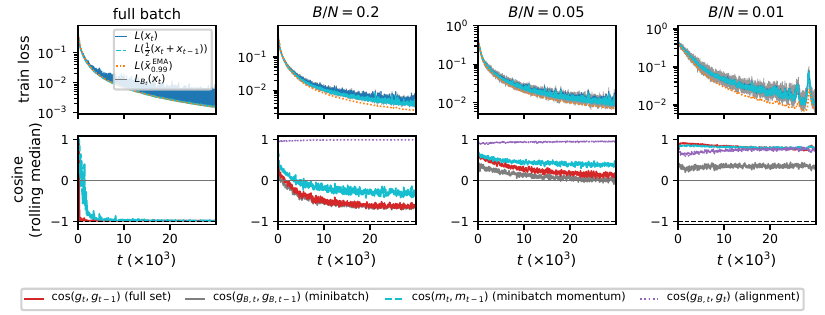}
\caption{As Figure~\ref{fig:noise-flip-mlp-b0} at $\beta_1=0.9$ (the same
fractions and readings).}
\label{fig:noise-flip-mlp-b09}
\end{figure}

\begin{figure}[t]
\centering
\includegraphics[width=\textwidth]{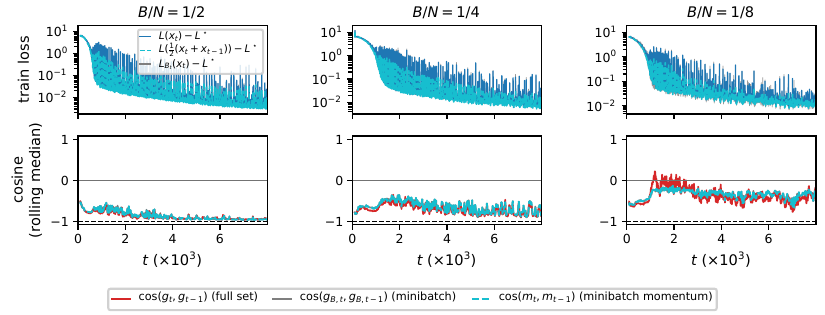}
\caption{Gradient reversal in the batch-size scan on GPT-2 medium,
$\beta_1=0$, $\beta_2=0.95$, $8000$ steps, the runs of
Table~\ref{tab:noise-cos} (at each step Adam sees $B=32$, $16$ or $8$ of
the fixed $64$ training sequences; the full-batch run is
Figure~\ref{fig:gpt2-reversal}(a)).  Top: the training loss on all $64$
sequences as excess over the label-smoothing floor $1.4076$, at the iterate
(blue, solid) and at the midpoint of consecutive iterates (cyan, dashed),
and the loss of the step's own minibatch (grey, thin);
$101$-step rolling medians of $\cos(g_t,g_{t-1})$ on the full set (red),
of $\cos(m_t,m_{t-1})$ (cyan, dashed; at $\beta_1=0$ it coincides with
the gradient cosine) and of $\cos(g_{B,t},g_{B,t-1})$ on the minibatches
(grey).  The alignment $\cos(g_{B,t},g_t)$ was not logged along these
runs; its checkpoint values are in Table~\ref{tab:noise-cos}.  Dashed
line at $-1$.}
\label{fig:noise-flip-gpt2-b0}
\end{figure}

\begin{figure}[t]
\centering
\includegraphics[width=\textwidth]{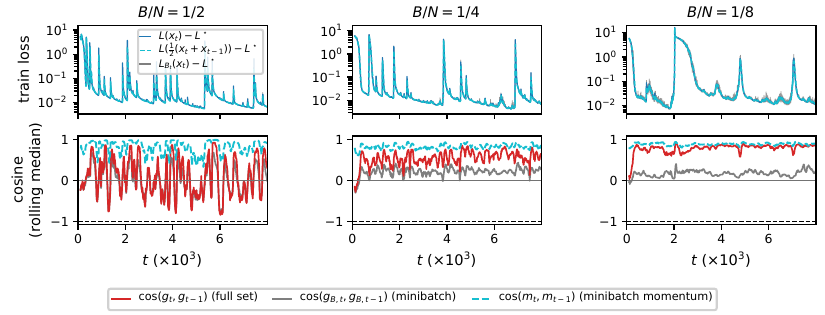}
\caption{As Figure~\ref{fig:noise-flip-gpt2-b0} at $\beta_1=0.9$,
$\beta_2=0.999$ }
\label{fig:noise-flip-gpt2-b09}
\end{figure}

\clearpage

\section{Additional preliminaries}
\label{app:frozen}

Freezing the second-moment variable turns one Adam step into a fixed map on
position and momentum,
\begin{equation}
 \Amap_v:\begin{pmatrix}x\\ m\end{pmatrix}\longmapsto
 \begin{pmatrix}
  x-\eta\bigl(\diag(\sqrt v)+\eps I\bigr)^{-1}\bigl(\beta_1m+(1-\beta_1)\nabla f(x)\bigr)\\
  \beta_1m+(1-\beta_1)\nabla f(x)
 \end{pmatrix},
 \label{eq:frozen-op}
\end{equation}
whose linearization is governed by a preconditioned Hessian.  The following
lemma is stated for a general twice continuously
differentiable $f$, with $\mathsf S_t=\lambda_{\max}(D_t^{-1}H_t)$ and $\mathsf S^\star$ as in \eqref{eq:Sstar}.

\begin{lemma}[Frozen stability]
\label{lem:frozen-main}
Suppose $H_t\succ0$.  The linearization of $\Amap_{v_{t+1}}$ at $(x_t,m_t)$ is
Schur stable if and only if $\mathsf S_t<\mathsf S^\star$.
\end{lemma}

\begin{proof}
Fix $t$ and write $D:=D_t$, $H:=H_t$ and
$\widetilde H:=D^{-1/2}HD^{-1/2}$, which is symmetric positive definite and
has the same eigenvalues as $D^{-1}H$.  In the coordinates
$x'=D^{1/2}\Delta x$ and $m'=D^{-1/2}\Delta m$, the linearization of
\eqref{eq:frozen-op} at $(x_t,m_t)$ reads
\begin{equation}
 m'_+=\beta_1m'+(1-\beta_1)\widetilde Hx',\qquad
 x'_+=x'-\eta m'_+ .
 \label{eq:frozen-linearization}
\end{equation}
In the eigenbasis of $\widetilde H$ this decouples into $2\times2$ blocks
\begin{equation}
 \begin{pmatrix}1-\eta(1-\beta_1)\sigma&-\eta\beta_1\\
 (1-\beta_1)\sigma&\beta_1\end{pmatrix},
 \label{eq:frozen-block}
\end{equation}
one per eigenvalue $\sigma>0$ of $\widetilde H$, with characteristic
polynomial
\begin{equation}
 z^2-\bigl(1+\beta_1-\eta(1-\beta_1)\sigma\bigr)z+\beta_1 .
 \label{eq:frozen-charpoly}
\end{equation}
For a real quadratic $z^2+a_1z+a_0$ the Jury criterion places both roots
inside the open unit disc exactly when $\abs{a_0}<1$ and
$\abs{a_1}<1+a_0$.  Here $a_0=\beta_1\in[0,1)$, and the second condition
reads $\abs{1+\beta_1-\eta(1-\beta_1)\sigma}<1+\beta_1$, that is
\begin{equation}
 0<\eta(1-\beta_1)\sigma<2(1+\beta_1).
 \label{eq:frozen-jury}
\end{equation}
This holds for every mode if and only if it holds for the largest,
$\sigma=\mathsf S_t$, which is $\mathsf S_t<\mathsf S^\star$.
\end{proof}

\paragraph{The rank-one model.}
\label{sec:setup}

Throughout Appendices~\ref{app:frozen}--\ref{sec:numerics} we fix the
rank-one objective of Section~\ref{sec:rank-1},
\begin{equation}
 f(x)=\tfrac12 x\trans Hx,\qquad H=\lambda uu\trans,\qquad \lambda>0,\qquad
 \norm{u}_2=1,
 \label{eq:model}
\end{equation}
so that $g_t=Hx_t=\lambda\xp_tu$, and run \eqref{eq:adam}.
\begin{assumption}[Standing initialization]
\label{as:r0}
$m_0\in\Span\{u\}$ and $v_0=\lambda^2\vpar_0\,u\odot u$ for some
$\vpar_0\ge0$.  The standard initialization $m_0=0$, $v_0=0$ of
Section~\ref{sec:rank-1} is the case $\vpar_0=0$.
\end{assumption}

Both properties propagate, because $g_t$ is a multiple of $u$:
$m_{t+1}=\beta_1m_t+(1-\beta_1)\lambda\xp_tu\in\Span\{u\}$, and
$v_{t+1}=\beta_2v_t+(1-\beta_2)\lambda^2(\xp_t)^2\,u\odot u$ is again
proportional to $u\odot u$.  Hence $m_t=\mpar_tu$ and
\begin{equation}
 v_t=\lambda^2\vpar_t\,u\odot u,\qquad
 \vpar_{t+1}=\beta_2\vpar_t+(1-\beta_2)(\xp_t)^2
 \label{eq:v-recursion}
\end{equation}
for every $t$, which is \eqref{eq:reduced-dynamics} (stated there for the
standard initialization $\vpar_0=0$).  The transverse position
$\xperp_t$ never feeds back.  Under Assumption~\ref{as:r0} the trajectory
is therefore described by the three scalars
$(\xp_t,\mpar_t,\vpar_t)\in\R\times\R\times[0,\infty)$, which evolve by
\eqref{eq:active-main} and \eqref{eq:v-recursion} with
\begin{equation}
 w_t=F(\vpar_t),\qquad
 F(\vpar)=c\eta\lambda\sum_{i=1}^d\frac{u_i^2}{\lambda\abs{u_i}\sqrt{\vpar}+\eps},
 \label{eq:F}
\end{equation}
as in \eqref{eq:w-rank-one}. The loss is
$f(x_t)=\lambda(\xp_t)^2/2$.  Every statement about the trajectory below is
a statement about this three-dimensional recursion.  Unrolling \eqref{eq:v-recursion} gives,
for all $T$ and $n\ge0$,
\begin{equation}
 \vpar_{T+n}=\beta_2^n\vpar_T
 +(1-\beta_2)\sum_{j=0}^{n-1}\beta_2^{\,n-1-j}(\xp_{T+j})^2 .
 \label{eq:unroll}
\end{equation}
When $u$ is a coordinate axis, $F(\vpar)=c\eta\lambda/(\lambda\sqrt\vpar+\eps)$
and the recursion is the one-dimensional Adam recursion.
\begin{lemma}[Properties of the response function]
\label{lem:F}
Let $\eps>0$.
\begin{enumerate}[leftmargin=1.6em,itemsep=1pt,topsep=2pt]
\item \textbf{Monotone bijection.}  $F$ is continuous and strictly
decreasing, with $F(0)=w_{\max}=c\eta\lambda/\eps$ and
$F(\vpar)\to0$ as $\vpar\to\infty$.  Hence for $0<\omega\le w_{\max}$ the level
\begin{equation}
 Q_\omega:=F^{-1}(\omega)\ \ge0
 \label{eq:Qw}
\end{equation}
is well defined, and $w_t\ge\omega\iff\vpar_t\le Q_\omega$,
$w_t\le\omega\iff\vpar_t\ge Q_\omega$.
\item \textbf{Scaling.}  For $\kappa\ge1$ and $\vpar\ge0$,
$F(\kappa^2\vpar)\ge F(\vpar)/\kappa$, with strict inequality when
$\kappa>1$.
\item \textbf{One-step restriction.}  For $\vpar\ge0$,
\begin{equation}
 \frac1{F(\beta_2\vpar)}-\frac1{w_{\max}}
 \ \ge\ \sqrt{\beta_2}\left(\frac1{F(\vpar)}-\frac1{w_{\max}}\right).
 \label{eq:F-restriction}
\end{equation}
\end{enumerate}
\end{lemma}

\begin{proof}
(i) is read off from \eqref{eq:F}.  (ii) Termwise,
$\lambda\abs{u_i}\kappa\sqrt\vpar+\eps\le\kappa(\lambda\abs{u_i}\sqrt\vpar+\eps)$,
strictly when $\kappa>1$ because $\eps>0$.
(iii) Put $y_i:=\eps/(\lambda\abs{u_i}\sqrt\vpar+\eps)\in(0,1]$ and
$\alpha_i:=u_i^2$, so that $\sum_i\alpha_i=1$ and
$F(\vpar)/w_{\max}=\sum_i\alpha_iy_i$.  Substituting
$\lambda\abs{u_i}\sqrt\vpar=\eps(1-y_i)/y_i$ gives the exact identity
\begin{equation*}
 \frac{\eps}{\lambda\abs{u_i}\sqrt{\beta_2\vpar}+\eps}
 =\Phi(y_i),\qquad
 \Phi(y):=\frac{y}{\sqrt{\beta_2}+(1-\sqrt{\beta_2})\,y},
\end{equation*}
and $\Phi$ is concave on $[0,1]$: for $\beta_2>0$,
$\Phi''(y)=-2\sqrt{\beta_2}(1-\sqrt{\beta_2})\bigl(\sqrt{\beta_2}+(1-\sqrt{\beta_2})y\bigr)^{-3}<0$,
while for $\beta_2=0$, $\Phi\equiv1$ and \eqref{eq:F-restriction} is
trivial.  Jensen's inequality with the weights $\alpha_i$ gives
$F(\beta_2\vpar)/w_{\max}=\sum_i\alpha_i\Phi(y_i)\le\Phi\bigl(F(\vpar)/w_{\max}\bigr)$,
and taking reciprocals is \eqref{eq:F-restriction}.
\end{proof}

The regime of interest is $w_{\max}>2$; otherwise no supercritical state is reachable at all.

\paragraph{The only nonzero preconditioned curvature.}
\label{app:rank-one-curvature}
The symmetric representative of the preconditioned Hessian is
\begin{equation}
 D_t^{-1/2}HD_t^{-1/2}=\lambda\bigl(D_t^{-1/2}u\bigr)\bigl(D_t^{-1/2}u\bigr)\trans,
 \label{eq:symmetric-preconditioned}
\end{equation}
which has rank one with unique nonzero eigenvalue
$\lambda u\trans D_t^{-1}u$ and eigenvector $D_t^{-1/2}u$.  This is also
the only nonzero eigenvalue of $D_t^{-1}H$, which is similar to
\eqref{eq:symmetric-preconditioned}, so
$\lambda_{\max}(P_t)=c\eta\lambda u\trans D_t^{-1}u={w_{t+1}}$.  The
preconditioned gradient $D_t^{-1/2}g_t=\lambda\xp_tD_t^{-1/2}u$ is a
multiple of that eigenvector whenever $\xp_t\ne0$, so the unit vector
$\hat z_t$ of \eqref{eq:Wact} is $\pm D_t^{-1/2}u/\norm{D_t^{-1/2}u}$ and
\begin{equation}
 W_{\mathrm{act},t}=\hat z_t\trans P_t\hat z_t
 =c\eta\lambda\,\frac{\bigl(u\trans D_t^{-1}u\bigr)^2}{u\trans D_t^{-1}u}
 =c\eta\lambda\,u\trans D_t^{-1}u={w_{t+1}} ,
 \label{eq:wact-rank-one}
\end{equation}
which is the coincidence \eqref{eq:w-rank-one} of the sharpness and the
active curvature used in Section~\ref{sec:rank-1}.

\subsection{Proof of Proposition~\ref{prop:main-cycle}}

\begin{proof}
By Lemma~\ref{lem:F}(i) and $2<w_{\max}$ the level $x_\star^2=F^{-1}(2)$ is
well defined and $x_\star>0$ is unique.  Start one step from
\begin{equation*}
 (\xp_0,\mpar_0,\vpar_0)=(x_\star,\,-c\lambda x_\star,\,x_\star^2).
\end{equation*}
The second moment does not move,
$\vpar_1=\beta_2x_\star^2+(1-\beta_2)(\xp_0)^2=x_\star^2$, so
\begin{equation*}
 w_1=F(\vpar_1)=F(x_\star^2)=2 .
\end{equation*}
The momentum recursion of \eqref{eq:active-main} gives, using
$1-\beta_1=c(1+\beta_1)$,
\begin{equation*}
 \mpar_1=\beta_1\mpar_0+(1-\beta_1)\lambda\xp_0
 =\lambda x_\star\bigl(-\beta_1c+1-\beta_1\bigr)=c\lambda x_\star ,
\end{equation*}
and the position update then gives
\begin{equation*}
 \xp_1=\xp_0-(c\lambda)^{-1}w_1\mpar_1=x_\star-2x_\star=-x_\star .
\end{equation*}
So the state after one step is exactly the second point
$(-x_\star,\,c\lambda x_\star,\,x_\star^2)$ of the asserted cycle. The orbit is therefore
$2$-periodic, with $\vpar_t\equiv x_\star^2$ and hence
$w_t\equiv2$, which by \eqref{eq:w-rank-one} is
$W_{\mathrm{act},t}=\lambda_{\max}(P_t)=2$ at every step.  Finally
$g_t=\lambda\xp_tu$ alternates in sign, so $\cos(g_t,g_{t+1})=-1$.
\end{proof}

\section{Proofs for RMSProp ($\beta_1=0$)}
\label{app:beta0}

Here we prove the results announced in Section~\ref{sec:rank-1} for the
zero-momentum dynamics: the finite passage times toward the threshold from
either side (Appendix~\ref{app:beta0-passage}, proving
Theorem~\ref{thm:RMSProp-rank1}), and the fact that the reversing band
$w>1$ is entered in finite time and never left
(Appendix~\ref{app:beta0-flip}).  Throughout this section $\beta_1=0$,
hence $c=1$, $\mpar_{t+1}=\lambda\xp_t$, and \eqref{eq:reduced-dynamics}
reads $\xp_{t+1}=(1-w_{t+1})\xp_t$ with $w_{t+1}=F(\vpar_{t+1})$.

\paragraph{Notation.}
For $n\ge1$ and $a,b\ge0$, let
\begin{equation}
 G_n(a,b):=\sum_{j=0}^{n-1}a^jb^{n-1-j}
 =\begin{cases}\dfrac{a^n-b^n}{a-b},&a\ne b,\\[6pt] na^{n-1},&a=b.\end{cases}
 \label{eq:G}
\end{equation}
For a time $T$, a subcritical target $\widehat w<2$ and a supercritical
margin $\delta>0$, set
\begin{equation}
 \sigma_{\widehat w}(T):=\inf\{n\ge0:w_{T+n}\ge\widehat w\},\quad
 \tau_\delta(T):=\inf\{n\ge1:w_{T+n}<2+\delta\},\quad
 A_\delta:=(1+\delta)^2,
 \label{eq:times}
\end{equation}
with $\inf\emptyset=+\infty$, and
\begin{equation}
\begin{gathered}
 M_T:=\max\{\vpar_T,(\xp_T)^2\},\qquad
 \rho_T:=\max\bigl\{\abs{1-F(M_T)},\abs{1-\widehat w}\bigr\},\\
 U_n(T):=\beta_2^n\vpar_T+(1-\beta_2)(\xp_T)^2\,G_n(\rho_T^2,\beta_2),\qquad
 L_n(T):=\beta_2^n\vpar_T+(1-\beta_2)(\xp_T)^2\,G_n(A_\delta,\beta_2).
\end{gathered}
 \label{eq:UL}
\end{equation}
Here $F(M_T)$ is a lower bound on the curvature while $\vpar_t\le M_T$,
$\rho_T$ is the resulting contraction factor of $\abs{\xp_t}$ below the
target, and $U_n(T)$, $L_n(T)$ are the levels of $\vpar_{T+n}$ after
$n$ contracting steps and after $n$ expanding steps.
  Recall
$Q_{\widehat w}=F^{-1}(\widehat w)$ and $Q_\delta:=F^{-1}(2+\delta)$ from
\eqref{eq:Qw}.

\subsection{Proof of Theorem~\ref{thm:RMSProp-rank1}}
\label{app:beta0-passage}

\paragraph{Subcritical stage.}
If $w_T\ge\widehat w$ then $\sigma_{\widehat w}(T)=0$; we consider
$w_T<\widehat w$.  While $w_t$ stays below $\widehat w<2$ the position
contracts at a rate bounded away from one, the contraction drains
$\vpar_t$, and a drained $\vpar_t$ cannot support a small $w_t$.

\begin{theorem}[Finite passage to a strict subcritical target]
\label{thm:sub-beta0}
Assume $\beta_1=0$, $w_{\max}>2$ and $w_T<\widehat w<2$.  Then
\begin{equation}
 N_{\rm sub}:=\min\bigl\{n\ge1:\,U_n(T)\le Q_{\widehat w}\bigr\}<\infty
 \qquad\text{and}\qquad
 1\le\sigma_{\widehat w}(T)\le N_{\rm sub} .
 \label{eq:Nsub-imp}
\end{equation}
Moreover, with $\gamma_T:=\max\{\rho_T^2,\beta_2\}<1$,
\begin{equation}
 N_{\rm sub}\le\overline N_{\rm sub}:=
 \begin{dcases}
  \left\lceil\frac{\log\bigl(C_T/Q_{\widehat w}\bigr)}{\log(1/\gamma_T)}\right\rceil,
  \quad C_T:=\vpar_T+\frac{(1-\beta_2)(\xp_T)^2}{\abs{\beta_2-\rho_T^2}},
  &\beta_2\ne\rho_T^2,\\
  1+\left\lceil\frac{\log\bigl(C_T/Q_{\widehat w}\bigr)}{\log(1/\sqrt{\gamma_T})}\right\rceil,
  \quad C_T:=\vpar_T+\frac{(1-\beta_2)(\xp_T)^2}{(1-\sqrt{\gamma_T})^2},
  &\beta_2=\rho_T^2 .
 \end{dcases}
 \label{eq:sub-explicit}
\end{equation}
\end{theorem}

\begin{proof}
Since $w_T<\widehat w<2$ and $w_T=F(\vpar_T)$,
Lemma~\ref{lem:F}(i) gives
\[
\vpar_T>Q_{\widehat w}>0.
\]
Recall $M_T=\max\{\vpar_T,(\xp_T)^2\}$. Since $F$ is decreasing,
\[
0<F(M_T)\le w_T<\widehat w<2.
\]
Hence
\[
\rho_T
:=\max_{w\in[F(M_T),\,\widehat w]}|1-w|<1.
\]

Now suppose that the trajectory stays below $\widehat w$ for the next $n$
steps:
\[
w_{T+j}<\widehat w,
\qquad j=0,\ldots,n-1.
\]
We show inductively that, throughout this interval,
\begin{equation}
\vpar_{T+j}\le M_T,
\qquad
|\xp_{T+j}|\le \rho_T^j|\xp_T|.
\label{eq:stopped-sub}
\end{equation}
Indeed, if these bounds hold at step $j-1$, then
$\vpar_{T+j}$ is a convex combination of $\vpar_{T+j-1}$ and
$(\xp_{T+j-1})^2$, both bounded by $M_T$. Thus
$\vpar_{T+j}\le M_T$, so
\[
F(M_T)\le w_{T+j}<\widehat w.
\]
By the definition of $\rho_T$,
$|1-w_{T+j}|\le\rho_T$, and the scalar recursion gives
\[
|\xp_{T+j}|
\le \rho_T|\xp_{T+j-1}|
\le \rho_T^j|\xp_T|.
\]

Unrolling the second-moment recursion therefore yields
\begin{equation}
\vpar_{T+n}
\le
\beta_2^n\vpar_T
+(1-\beta_2)(\xp_T)^2
\sum_{j=0}^{n-1}\beta_2^{n-1-j}\rho_T^{2j}
=:U_n(T).
\label{eq:q-upper}
\end{equation}
Because $\beta_2<1$ and $\rho_T<1$, we have $U_n(T)\to0$.
Hence there exists a finite $N_{\rm sub}$ such that
\[
U_{N_{\rm sub}}(T)\le Q_{\widehat w}.
\]
If the trajectory has not crossed $\widehat w$ before that time, then
\eqref{eq:q-upper} gives
\[
\vpar_{T+N_{\rm sub}}\le Q_{\widehat w},
\]
and therefore, by Lemma~\ref{lem:F}(i),
\[
w_{T+N_{\rm sub}}\ge\widehat w.
\]
Thus the trajectory must leave the region $w<\widehat w$ within at most
$N_{\rm sub}$ steps.

It remains only to bound $N_{\rm sub}$ explicitly.
Let
\[
\gamma_T:=\max\{\beta_2,\rho_T^2\}<1.
\]
If $\beta_2\neq\rho_T^2$, the geometric sum in
\eqref{eq:q-upper} satisfies
\[
U_n(T)\le C_T\gamma_T^n,
\]
so it is sufficient to choose
\[
n\ge
\frac{\log(C_T/Q_{\widehat w})}
     {\log(1/\gamma_T)}.
\]
If $\beta_2=\rho_T^2=\gamma_T$, the sum contributes an additional factor
$n$, which can be bounded by a slower geometric rate, giving
\[
U_n(T)\le C_T\gamma_T^{(n-1)/2}.
\]
Hence it suffices that
\[
n-1\ge
\frac{\log(C_T/Q_{\widehat w})}
     {\log(1/\sqrt{\gamma_T})}.
\]
This proves the claimed finite-time bound.
\end{proof}

\paragraph{Supercritical stage.}
As long as $w_t$ remains above $2+\delta$ the position expands geometrically, which forces $\vpar_t$ upward and eventually makes such a large value of $w_t$ impossible.

\begin{lemma}[Geometric expansion forces exit]
\label{lem:expansion-exit}
Fix $\delta>0$ with $2+\delta\le w_{\max}$, $T$ with $\xp_T\ne0$, and any
$0\le\beta_1<1$.  Then
\begin{equation}
\begin{gathered}
 N_{\rm sup}:=\min\bigl\{n\ge1:\,L_n(T)>Q_\delta\bigr\}
 \ \le\ \overline N_{\rm sup},\\
 \overline N_{\rm sup}:=1+
 \left\lfloor
 \frac{\log\bigl(1+(A_\delta-\beta_2)Q_\delta/((1-\beta_2)(\xp_T)^2)\bigr)}
 {\log A_\delta}\right\rfloor<\infty .
\end{gathered}
 \label{eq:explicit-super}
\end{equation}
If moreover, for some $n\ge1$,
\begin{equation}
 w_{T+k}\ge2+\delta
 \quad\text{and}\quad
 \abs{\xp_{T+k}}\ge(1+\delta)\abs{\xp_{T+k-1}},
 \qquad k=1,\dots,n ,
 \label{eq:expansion-hyp}
\end{equation}
then $n<N_{\rm sup}$.
\end{lemma}
\begin{proof}
Suppose the trajectory remains in the supercritical region
\[
w_{T+j}\ge 2+\delta,
\qquad j=0,\ldots,n.
\]
Under \eqref{eq:expansion-hyp}, the scalar state expands geometrically:
\[
(\xp_{T+j})^2\ge A_\delta^{\,j}(\xp_T)^2,
\qquad j=0,\ldots,n,
\]
where $A_\delta>1$.

Unrolling the second-moment recursion therefore gives, for every $k\le n$,
\begin{equation}
\vpar_{T+k}
\ge
{\beta_2^{\,k}\vpar_T+}(1-\beta_2)(\xp_T)^2
\sum_{j=0}^{k-1}
\beta_2^{\,k-1-j}A_\delta^j
{=L_k(T)}.
\label{eq:q-lower}
\end{equation}
Since $A_\delta>\beta_2$ and $\vpar_T\ge0$, summing the geometric series gives
\[
L_k(T)
{\ge}
(1-\beta_2)(\xp_T)^2
\frac{A_\delta^k-\beta_2^k}{A_\delta-\beta_2}
\ge
(1-\beta_2)(\xp_T)^2
\frac{A_\delta^k-1}{A_\delta-\beta_2}.
\]
Because $A_\delta>1$, the lower bound grows without bound with $k$.

On the other hand, as long as $w_{T+k}\ge 2+\delta$,
Lemma~\ref{lem:F}(i) implies
\[
\vpar_{T+k}\le Q_\delta.
\]
Hence the trajectory can remain in the region $w\ge2+\delta$ only while
\[
L_k(T)\le Q_\delta.
\]

By the definition of $\overline N_{\rm sup}$, the lower bound above exceeds
$Q_\delta$ at $k=\overline N_{\rm sup}$. Therefore the trajectory must leave
the supercritical region before that time, and thus
\[
N_{\rm sup}\le \overline N_{\rm sup}.
\]
\end{proof}
\begin{theorem}[Finite supercritical exit]
\label{thm:super-beta0}
Assume $\beta_1=0$, $w_T\ge2+\delta$ with $\delta>0$, and $\xp_T\ne0$.  Then
$\tau_\delta(T)\le N_{\rm sup}\le\overline N_{\rm sup}<\infty$.
\end{theorem}

\begin{proof}
Note $2+\delta\le w_T\le w_{\max}$.  If $\tau_\delta(T)>n$, then
$w_{T+k}\ge2+\delta$ for $k=1,\dots,n$, and \eqref{eq:reduced-dynamics} gives
$\abs{\xp_{T+k}}=(w_{T+k}-1)\abs{\xp_{T+k-1}}\ge(1+\delta)\abs{\xp_{T+k-1}}$,
so \eqref{eq:expansion-hyp} holds and Lemma~\ref{lem:expansion-exit} gives
$n<N_{\rm sup}$.  Hence $\tau_\delta(T)\le N_{\rm sup}$.
\end{proof}

\paragraph{The two-sided statement.}
Theorems~\ref{thm:sub-beta0} and~\ref{thm:super-beta0} are the two halves
of Theorem~\ref{thm:RMSProp-rank1}.  Because the two levels may be chosen
arbitrarily close to $2$, they combine into the restoring statement made
after it.

\subsection{Eventual reversal at every step}
\label{app:beta0-flip}

The growth threshold is $2$; the sign threshold is $1$.  The band between
them is never re-crossed downward, so from a finite time on the component
along $u$, and with it the gradient, reverses at every step.

\begin{lemma}[Eventual reversal at every step]
\label{lem:flip-beta0}
Assume $\beta_1=0$, $w_{\max}>1$ and $\xp_t\ne0$ for all $t$.
\begin{enumerate}[leftmargin=1.6em,itemsep=1pt,topsep=2pt]
\item \textbf{Forward invariance.}  If $w_t>1$ for some $t\ge1$, then
$w_{t+1}>1$.
\item \textbf{Finite entry.}  There is a finite $T\ge1$ with $w_T>1$.  If
$w_{\max}>2$ and $w_1<1$, then $T\le1+\overline N_{\rm sub}$, where
$\overline N_{\rm sub}$ is the bound \eqref{eq:sub-explicit} of
Theorem~\ref{thm:sub-beta0} at time $1$ and level $\widehat w=1$.
\item \textbf{Reversal.}  For all $t\ge T$: $w_t>1$, $\xp_{t-1}\xp_t<0$ and $\cos(g_{t-1},g_t)=-1$.
\end{enumerate}
\end{lemma}

\begin{proof}
(i) Fix $t\ge1$ with $w_t>1$.  By \eqref{eq:v-recursion} at time $t-1$,
$(1-\beta_2)(\xp_{t-1})^2\le\vpar_t$, and by
\eqref{eq:reduced-dynamics}, $(\xp_t)^2=(w_t-1)^2(\xp_{t-1})^2$.  Hence
\begin{equation}
 \vpar_{t+1}=\beta_2\vpar_t+(1-\beta_2)(\xp_t)^2
 \le\bigl[\beta_2+(w_t-1)^2\bigr]\vpar_t\le w_t^2\,\vpar_t ,
 \label{eq:q-one-step-growth}
\end{equation}
the last step because $w_t^2-\beta_2-(w_t-1)^2=2w_t-1-\beta_2>0$.  By
Lemma~\ref{lem:F}(i) and then (ii) with $\kappa=w_t>1$,
\begin{equation}
 w_{t+1}=F(\vpar_{t+1})\ \ge\ F(w_t^2\vpar_t)\ >\ \frac{F(\vpar_t)}{w_t}=1 .
 \label{eq:w-invariance}
\end{equation}

(ii) By \eqref{eq:reduced-dynamics}, $w_t=1$ for some $t\ge1$ would force
$\xp_t=0$; so $w_t\ne1$ for all $t\ge1$.  If $w_1>1$ take $T=1$.
Otherwise suppose $w_t<1$ for all $t\ge1$.  Then
$\abs{\xp_{t+1}}=(1-w_{t+1})\abs{\xp_t}\le\abs{\xp_t}$, so as in
\eqref{eq:stopped-sub} $\vpar_t\le M_1$ and $w_t\ge F(M_1)>0$ for all
$t\ge1$, whence $\abs{\xp_t}\le(1-F(M_1))^{t-1}\abs{\xp_1}\to0$; by
\eqref{eq:unroll}, $\vpar_t\to0$ and $w_t\to F(0)=w_{\max}>1$,
contradicting $w_t<1$.  When $w_{\max}>2$, Theorem~\ref{thm:sub-beta0} at
time $1$ with $\widehat w=1$ gives $1\le n\le\overline N_{\rm sub}$ with
$w_{1+n}\ge1$, hence $w_{1+n}>1$.

(iii) Parts (i) and (ii) give $w_t>1$ for all $t\ge T$, so
\eqref{eq:reduced-dynamics} reads $\xp_t=-(w_t-1)\xp_{t-1}$ with
$w_t-1>0$; hence $\xp_{t-1}\xp_t<0$, and since $g_t=\lambda\xp_tu$,
$\cos(g_{t-1},g_t)=\operatorname{sign}(\xp_{t-1}\xp_t)=-1$.
\end{proof}

\section{Proofs for general Adam (Theorem~\ref{thm:Adam-rankone})}
\label{app:positive-sub}

Assume $0<\beta_1<1$ throughout this section.

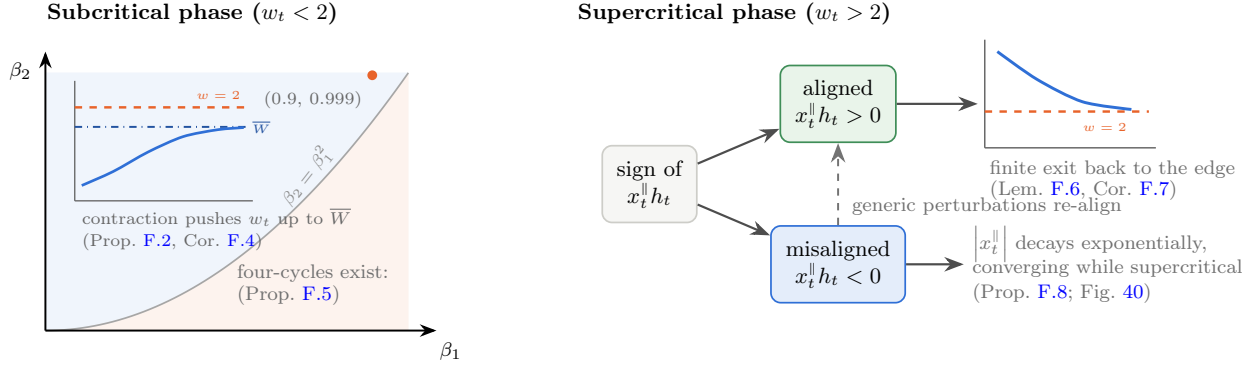
\begin{figure}[t]
  \centering
  \resizebox{\linewidth}{!}{%
\begin{tikzpicture}[
  font=\small,
  box/.style={rounded corners=4pt, align=center, inner sep=6pt, draw, line width=0.7pt},
  sup/.style={box, fill=supfill, draw=supline},
  sub/.style={box, fill=subfill, draw=subline},
  neu/.style={box, fill=neufill, draw=neuline},
  karr/.style={-{Stealth[length=2.8mm]}, line width=0.9pt, black!70},
  darr/.style={-{Stealth[length=2.8mm]}, line width=0.8pt, black!55, dashed}
]
\begin{scope}
  \node[align=center,anchor=west] at (-0.1,4.55) {\textbf{Subcritical phase ($w_t<2$)}};
  \fill[cblue!7]  (0,0) \foreach \u in {0.1,0.2,...,1.0}{ -- (5.2*\u,{3.7*\u*\u}) } -- (0,3.7) -- cycle;
  \fill[corange!8] (0,0) \foreach \u in {0.1,0.2,...,1.0}{ -- (5.2*\u,{3.7*\u*\u}) } -- (5.2,0) -- cycle;
  \draw[black!35, line width=0.7pt]
    plot[domain=0:1, samples=60] ({5.2*\x},{3.7*\x*\x});
  \node[black!45, rotate=52, anchor=south] at (3.95,2.05) {\scriptsize $\beta_2=\beta_1^2$};
  \draw[-{Stealth[length=2.4mm]}, line width=0.8pt] (0,0) -- (5.6,0)
      node[below left=1pt and -14pt]{\footnotesize $\beta_1$};
  \draw[-{Stealth[length=2.4mm]}, line width=0.8pt] (0,0) -- (0,4.0)
      node[left=2pt,pos=0.93]{\footnotesize $\beta_2$};
  \fill[corange] (4.68,3.66) circle (0.07);
  \node[anchor=north east, black!60] at (4.66,3.56) {\scriptsize $(0.9,\,0.999)$};
  \begin{scope}[shift={(0.42,1.86)}]
    \draw[line width=0.55pt, black!60] (0,0) -- (2.55,0);
    \draw[line width=0.55pt, black!60] (0,0) -- (0,1.72);
    \draw[dashed, corange, line width=1pt] (0,1.34) -- (2.45,1.34);
    \node[anchor=west, corange] at (1.62,1.52) {\tiny $w=2$};
    \draw[dash dot, cblue!75!black, line width=0.8pt] (0,1.06) -- (2.45,1.06);
    \node[anchor=west, cblue!75!black] at (2.42,1.06) {\tiny $\overline W$};
    \draw[cblue, line width=1.1pt]
      plot[smooth] coordinates {(0.10,0.22) (0.55,0.42) (1.05,0.70) (1.55,0.92) (2.05,1.02) (2.42,1.05)};
    \node[anchor=west, black!60] at (-0.02,-0.24)
      {\scriptsize contraction pushes $w_t$ up to $\overline W$};
    \node[anchor=west, black!60] at (-0.02,-0.58)
      {\scriptsize (Prop.~\ref{prop:lyapunov}, Cor.~\ref{cor:positive-sub-passage})};
  \end{scope}
  \node[align=left, anchor=west] at (2.62,0.66)
    {\subtxt{four-cycles exist:}\\[-2pt]
     \subtxt{(Prop.~\ref{prop:four-cycle})}};
\end{scope}
\begin{scope}[shift={(7.6,0.2)}]
  \node[align=center,anchor=west] at (-0.1,4.35) {\textbf{Supercritical phase ($w_t>2$)}};
  \node[neu] (root) at (1.05,1.9) {sign of\\ $\xp_th_t$};
  \node[sup] (al) at (3.75,3.05) {aligned\\ $\xp_th_t>0$};
  \node[sub] (mis) at (3.75,0.75) {misaligned\\ $\xp_th_t<0$};
  \begin{scope}[shift={(5.85,2.42)}]
    \draw[line width=0.6pt, black!60] (0,0) -- (2.45,0);
    \draw[line width=0.6pt, black!60] (0,0) -- (0,1.55);
    \draw[dashed, corange, line width=1pt] (0,0.52) -- (2.35,0.52);
    \node[anchor=west, corange] at (1.28,0.30) {\tiny $w=2$};
    \draw[cblue, line width=1.2pt]
      plot[smooth] coordinates {(0.18,1.38) (0.75,0.98) (1.4,0.66) (2.1,0.55)};
  \end{scope}
  \node[anchor=west] at (5.80,2.12)
    {\subtxt{finite exit back to the edge}};
  \node[anchor=west] at (5.80,1.82)
    {\subtxt{(Lem.~\ref{lem:sign-geometry}, Cor.~\ref{cor:aligned-exit})}};
  \node[align=left, anchor=west] (mistext) at (5.55,0.75)
    {\subtxt{$\abs{\xp_t}$ decays exponentially,}\\[-1pt]
     \subtxt{converging while supercritical}\\[-1pt]
     \subtxt{(Prop.~\ref{prop:persistent}; Fig.~\ref{fig:exceptions})}};
  \draw[karr] (root) -- (al);
  \draw[karr] (root) -- (mis);
  \draw[karr] (al.east) -- (5.75,3.05);
  \draw[karr] (mis.east) -- (mistext.west);
  \draw[darr] (mis.north) to[out=90,in=-90]
    node[right=2pt,pos=0.22]{\subtxt{generic perturbations re-align}} (al.south);
\end{scope}
\end{tikzpicture}}
  \caption{Proof roadmap for $\beta_1>0$ (Theorem~\ref{thm:Adam-rankone}).
  Left: below the edge, an adaptive Lyapunov function certifies
  contraction up to a cutoff $\overline W<2$ when $\beta_2>\beta_1^2$
  (Proposition~\ref{prop:lyapunov},
  Corollary~\ref{cor:positive-sub-passage}), while strictly subcritical
  four-cycles exist for complementary parameter choices when $\eps=0$
  (Proposition~\ref{prop:four-cycle}).  Right: above the edge, the sign of
  $\xp_th_t$ separates aligned trajectories, which exit every fixed
  supercritical band in finite time
  (Lemma~\ref{lem:sign-geometry}, Corollary~\ref{cor:aligned-exit}), from
  exceptional persistently misaligned trajectories that converge while
  remaining uniformly supercritical (Proposition~\ref{prop:persistent}).}
  \label{fig:proof-roadmap}
\end{figure}

Figure~\ref{fig:proof-roadmap} summarizes the structure of the proofs in this section.

\paragraph{The normalized map.}
With the normalized momentum coordinate $h_t=(c\lambda)^{-1}\mpar_t+\xp_t$
of Section~\ref{sec:rank-1}, write
\begin{equation}
 z_t:=(\xp_t,h_t)\trans .
 \label{eq:h-z}
\end{equation}
The recursion \eqref{eq:active-main} is then exactly linear in $z_t$ once
$w_{t+1}$ is known:
\begin{equation}
 z_{t+1}=\Amap(w_{t+1})\,z_t,\qquad
 \Amap(w):=\begin{pmatrix}1-w&-\beta_1w\\ 2-w&\beta_1(1-w)\end{pmatrix}.
 \label{eq:A-map}
\end{equation}
Indeed, since $1-\beta_1=c(1+\beta_1)$,
\begin{equation}
 \mpar_{t+1}=\beta_1\mpar_t+(1-\beta_1)\lambda\xp_t
 =\beta_1c\lambda\,(h_t-\xp_t)+(1-\beta_1)\lambda\xp_t
 =c\lambda\,(\xp_t+\beta_1h_t),
 \label{eq:mu-from-h}
\end{equation}
so that
\begin{equation}
\begin{aligned}
 \xp_{t+1}&=\xp_t-\frac{w_{t+1}}{c\lambda}\,\mpar_{t+1}
 =(1-w_{t+1})\xp_t-\beta_1w_{t+1}h_t,\\
 h_{t+1}&=\frac{\mpar_{t+1}}{c\lambda}+\xp_{t+1}
 =(2-w_{t+1})\xp_t+\beta_1(1-w_{t+1})h_t ,
\end{aligned}
 \label{eq:A-map-entries}
\end{equation}
which is \eqref{eq:A-map}.  The
matrix $\Amap(w)$ is similar to the frozen block \eqref{eq:frozen-block}
with $\sigma=w/(c\eta)$, so its spectral radius is below one exactly when
$0<w<2$.  The matrices encountered along a trajectory, however, vary with
$w_t$ and need not commute, so stability of every individual matrix does
not control their product.

\paragraph{Roadmap.}
Appendix~\ref{app:lyapunov} proves the subcritical item of
Theorem~\ref{thm:Adam-rankone}: the restriction that the second-moment
recursion places on successive values of $w_t$ (Lemma~\ref{lem:jensen})
makes an adaptive Lyapunov functional contract below the cutoff
$\overline W$ (Proposition~\ref{prop:lyapunov}), which forces finite
passage (Corollary~\ref{cor:positive-sub-passage}); Appendix~\ref{app:four-cycle}
shows that the parameter condition $\beta_2>\beta_1^2$ cannot be dropped.
Appendix~\ref{app:super} proves the two supercritical items through the
sign geometry of $(\xp_t,h_t)$ (Lemma~\ref{lem:sign-geometry},
Corollary~\ref{cor:aligned-exit}), proves the momentum-reversal
Lemma~\ref{lem:main-flip-momentum} of the body
(Appendix~\ref{app:flip-momentum}), and exhibits the persistently
misaligned trajectories (Appendix~\ref{app:persistent}).

\subsection{Proof of the subcritical regime of Theorem~\ref{thm:Adam-rankone}}
\label{app:lyapunov}

In one dimension the subcritical argument opens with an exact identity
relating $w_{t+1}$ to $w_t$ through the scalar second moment.  Here $F$ is
a sum over coordinates and one side of it survives, as the one-step
restriction of Lemma~\ref{lem:F}(iii).

\begin{lemma}[Successive-sharpness restriction]
\label{lem:jensen}
For every $t$,
\begin{equation}
 w_{t+1}\ \le\
 \frac{w_t}{\sqrt{\beta_2}+(1-\sqrt{\beta_2})\,w_t/w_{\max}},
 \label{eq:w-transition}
\end{equation}
equivalently
\begin{equation}
 \frac1{w_{t+1}}-\frac1{w_{\max}}
 \ \ge\ \sqrt{\beta_2}\left(\frac1{w_t}-\frac1{w_{\max}}\right).
 \label{eq:reciprocal-transition}
\end{equation}
\end{lemma}

\begin{proof}
By \eqref{eq:v-recursion}, $\vpar_{t+1}\ge\beta_2\vpar_t$, so
$w_{t+1}=F(\vpar_{t+1})\le F(\beta_2\vpar_t)$ by
Lemma~\ref{lem:F}(i), and \eqref{eq:F-restriction} at
$\vpar=\vpar_t$ gives \eqref{eq:reciprocal-transition};
\eqref{eq:w-transition} is the same inequality solved for $w_{t+1}$.
\end{proof}

\paragraph{The adaptive Lyapunov functional.}
The family $\{\Amap(w):0<w<2\}$ admits no common Euclidean contraction
estimate.  We instead use a quadratic form whose weight on $h_t^2$ is
chosen as a function of the current $w_t$:
\begin{equation}
\begin{gathered}
 \theta(w):=\frac{w}{2-w}\quad(0<w<2),\qquad
 P(w):=\diag\bigl(1,\beta_1\theta(w)\bigr),\\
 \Psi_t:=z_t\trans P(w_t)z_t=(\xp_t)^2+\beta_1\theta(w_t)h_t^2 .
\end{gathered}
 \label{eq:Psi}
\end{equation}
Recall the cutoff $\overline W$ of Theorem~\ref{thm:Adam-rankone}, defined in \eqref{eq:Wbar-restated}.  The contraction rate on a band
$0<a\le w\le b<\overline W$ is expressed through
\begin{equation}
 K(b):=\frac{2-b\bigl(\sqrt{\beta_2}+2(1-\sqrt{\beta_2})/w_{\max}\bigr)}
 {\sqrt{\beta_2}\,(2-b)},
 \label{eq:Kb}
\end{equation}
\begin{equation}
 \Delta(a,b):=\frac{(1-\beta_1)\min\{a(2-a),\,b(2-b)\}\cdot
 \dfrac{\beta_1a}{2-a}\bigl(1-\beta_1K(b)\bigr)}
 {2\left(1+\dfrac{\beta_1b}{2-b}\right)\max\left\{1,\dfrac{\beta_1b}{2-b}\right\}} .
 \label{eq:Delta}
\end{equation}
The proof below shows that $K(b)$ bounds the ratio
$\theta(w_{t+1})/\theta(w_t)$ on the band, that $\beta_1K(b)<1$ is exactly
$b<\overline W$, and that $\Delta(a,b)$ is the resulting lower bound on
$1-\Psi_{t+1}/\Psi_t$.

\begin{proposition}[Uniform contraction below an explicit cutoff]
\label{prop:lyapunov}
Assume $0<\beta_1<1$, $w_{\max}>2$ and $\beta_2>\beta_1^2$.  Then
$0<\overline W<2$.  Fix $0<a\le b<\overline W$ and suppose
$a\le w_t,w_{t+1}\le b$.  Then $0<\Delta(a,b)<1$ and
\begin{equation}
 \Psi_{t+1}\ \le\ \bigl(1-\Delta(a,b)\bigr)\Psi_t .
 \label{eq:lyapunov-contraction}
\end{equation}
In particular, if $\Psi_t>0$ and $w_t,w_{t+1}<\overline W$, then
$\Psi_{t+1}<\Psi_t$.
\end{proposition}

\begin{proof}
We first record the restriction that the second-moment recursion places on
the pair $(w_t,w_{t+1})$.  Rearranging Lemma~\ref{lem:jensen},
\begin{equation}
 w_t\ \ge\ \frac{\sqrt{\beta_2}\,w_{t+1}}
 {1-(1-\sqrt{\beta_2})w_{t+1}/w_{\max}} .
 \label{eq:w-lower}
\end{equation}

We next solve the one-step quadratic inequality.  For $y,Y>0$ set
$P_y:=\diag(1,y)$ and
\begin{equation}
 \Emap(y,Y;w):=P_y-\Amap(w)\trans P_Y\Amap(w).
 \label{eq:Emap-def}
\end{equation}
Direct multiplication gives the exact identities
\begin{align}
 \Emap_{11}(y,Y;w)&=(2-w)\bigl(w-(2-w)Y\bigr),
 \label{eq:D11}\\
 \det\Emap(y,Y;w)&=\bigl(w-(2-w)Y\bigr)\bigl((2-w)y-\beta_1^2w\bigr).
 \label{eq:Ddet}
\end{align}
Because $0<w<2$, Sylvester's criterion therefore shows that
\begin{equation}
 \Amap(w)\trans P_Y\Amap(w)\prec P_y
 \iff
 Y<\theta(w)\ \text{ and }\ y>\beta_1^2\theta(w).
 \label{eq:LMI}
\end{equation}
For the Lyapunov functional \eqref{eq:Psi} the source and target weights
are $y=\beta_1\theta(w_t)$ and $Y=\beta_1\theta(w_{t+1})$.  The first
condition in \eqref{eq:LMI} is automatic because $\beta_1<1$, while the
second becomes
\begin{equation}
 \theta(w_t)>\beta_1\theta(w_{t+1}).
 \label{eq:theta-condition}
\end{equation}

Since $\theta$ is strictly increasing on $(0,2)$, \eqref{eq:w-lower} implies
\begin{equation*}
 \theta(w_t)\ \ge\ \theta\!\left(
 \frac{\sqrt{\beta_2}\,w_{t+1}}
 {1-(1-\sqrt{\beta_2})w_{t+1}/w_{\max}}\right),
\end{equation*}
and therefore, after direct simplification,
\begin{equation}
 \frac{\theta(w_{t+1})}{\theta(w_t)}\ \le\
 \frac{2-w_{t+1}\bigl(\sqrt{\beta_2}+2(1-\sqrt{\beta_2})/w_{\max}\bigr)}
 {\sqrt{\beta_2}\,(2-w_{t+1})} .
 \label{eq:theta-ratio}
\end{equation}
Write $A:=\sqrt{\beta_2}+2(1-\sqrt{\beta_2})/w_{\max}$, so that the
right-hand side of \eqref{eq:theta-ratio} is
$(2-Aw_{t+1})/\bigl(\sqrt{\beta_2}(2-w_{t+1})\bigr)$, whose derivative in
$w_{t+1}$ has the sign of $2-2A$; and $A<1$ exactly when $w_{\max}>2$.  The
right-hand side is thus increasing in $w_{t+1}$, so if
$w_t,w_{t+1}\le b<2$ then
\begin{equation}
 \frac{\theta(w_{t+1})}{\theta(w_t)}\ \le\ K(b).
 \label{eq:theta-ratio-b}
\end{equation}
The product $\beta_1K(b)$ is strictly below one exactly when
$b\bigl(\sqrt{\beta_2}-\beta_1A\bigr)<2(\sqrt{\beta_2}-\beta_1)$, that is
exactly when $b<\overline W$.  Since $\sqrt{\beta_2}>\beta_1$ and
$w_{\max}>2$,
\begin{equation}
 \sqrt{\beta_2}(1-\beta_1)-\frac{2\beta_1(1-\sqrt{\beta_2})}{w_{\max}}
 -(\sqrt{\beta_2}-\beta_1)
 =\beta_1(1-\sqrt{\beta_2})\left(1-\frac2{w_{\max}}\right)>0,
 \label{eq:W-positive}
\end{equation}
so the denominator of $\overline W$ in \eqref{eq:Wbar-restated} is positive and
strictly larger than $\sqrt{\beta_2}-\beta_1$, whence $0<\overline W<2$.

It remains to make the contraction factor explicit.  Set
\begin{equation}
 \Emap:=P(w_t)-\Amap(w_{t+1})\trans P(w_{t+1})\Amap(w_{t+1}).
 \label{eq:Emap-specific}
\end{equation}
Specializing \eqref{eq:D11}--\eqref{eq:Ddet}, or equivalently taking a Schur
complement, gives
\begin{equation}
 \Emap_{11}=(1-\beta_1)w_{t+1}(2-w_{t+1}),\qquad
 \Emap_{22}-\frac{\Emap_{12}^2}{\Emap_{11}}
 =\beta_1\bigl(\theta(w_t)-\beta_1\theta(w_{t+1})\bigr).
 \label{eq:Schur-specific}
\end{equation}
If $a\le w_t,w_{t+1}\le b<\overline W$, then
\begin{equation}
 \Emap_{11}\ge(1-\beta_1)\min\{a(2-a),b(2-b)\},\qquad
 \Emap_{22}-\frac{\Emap_{12}^2}{\Emap_{11}}
 \ \ge\ \frac{\beta_1a}{2-a}\bigl(1-\beta_1K(b)\bigr)\ >\ 0,
 \label{eq:D-lower}
\end{equation}
the second by \eqref{eq:theta-ratio-b} and $\theta(w_t)\ge\theta(a)$, and
the positivity because $\beta_1K(b)<1$.  Hence $\Emap\succ0$, and all
factors of $\Delta(a,b)$ in \eqref{eq:Delta} are strictly positive.
Moreover
\begin{equation*}
 P(w_t)-\Emap=\Amap(w_{t+1})\trans P(w_{t+1})\Amap(w_{t+1})\succ0,
\end{equation*}
because $P(w_{t+1})\succ0$ and $\det\Amap(w_{t+1})=\beta_1>0$.  Thus
$0\prec\Emap\prec P(w_t)$, and since $w_t\le b$ and $\theta$ is increasing,
\begin{equation}
 \Tr\Emap<\Tr P(w_t)\le1+\frac{\beta_1b}{2-b} .
 \label{eq:trace-bound}
\end{equation}
Also, by \eqref{eq:D-lower},
\begin{equation}
 \det\Emap=\Emap_{11}\left(\Emap_{22}-\frac{\Emap_{12}^2}{\Emap_{11}}\right)
 \ \ge\ (1-\beta_1)\min\{a(2-a),b(2-b)\}\,
 \frac{\beta_1a}{2-a}\bigl(1-\beta_1K(b)\bigr).
 \label{eq:det-bound}
\end{equation}
A positive definite $2\times2$ matrix satisfies
$\lambda_{\min}\ge\det/\Tr$, so \eqref{eq:Delta},
\eqref{eq:trace-bound} and \eqref{eq:det-bound} give
\begin{equation*}
 \lambda_{\min}(\Emap)\ \ge\ 2\Delta(a,b)\max\left\{1,\frac{\beta_1b}{2-b}\right\}.
\end{equation*}
On the other hand $P(w_t)\preceq\max\{1,\beta_1b/(2-b)\}I$, and therefore
\begin{equation}
 \Emap\ \succeq\ 2\Delta(a,b)\max\left\{1,\frac{\beta_1b}{2-b}\right\}I
 \ \succeq\ 2\Delta(a,b)\,P(w_t)\ \succeq\ \Delta(a,b)\,P(w_t).
 \label{eq:D-dominates}
\end{equation}
Since $\Emap\prec P(w_t)$, the first inequality in \eqref{eq:D-dominates}
forces $2\Delta(a,b)<1$, so in particular $0<\Delta(a,b)<1$.  Finally,
\begin{equation*}
 \Psi_{t+1}=z_t\trans\Amap(w_{t+1})\trans P(w_{t+1})\Amap(w_{t+1})z_t
 =\Psi_t-z_t\trans\Emap z_t
 \le\bigl(1-\Delta(a,b)\bigr)\Psi_t ,
\end{equation*}
which is \eqref{eq:lyapunov-contraction}.  The last assertion follows by
taking $a:=\min\{w_t,w_{t+1}\}$ and $b:=\max\{w_t,w_{t+1}\}$.
\end{proof}

\begin{remark}[Expansion of the cutoff]
\label{rem:cutoff}
Substituting $\sqrt{\beta_2}=1-\tfrac12(1-\beta_2)+O\bigl((1-\beta_2)^2\bigr)$
into \eqref{eq:Wbar-restated} gives, for fixed $\beta_1$ and $w_{\max}>2$,
\begin{equation}
 \overline W=2-\frac{\beta_1\bigl(1-2/w_{\max}\bigr)}{1-\beta_1}(1-\beta_2)
 +O\bigl((1-\beta_2)^2\bigr)\qquad(\beta_2\to1).
 \label{eq:W-expansion}
\end{equation}
For $(\beta_1,\beta_2)=(0.9,0.999)$ and large $w_{\max}$ this gives
$\overline W\approx1.991$: a restoring tendency close to, but not all the
way up to, $2$.
\end{remark}

The contraction forces finite passage, which is the subcritical item of
Theorem~\ref{thm:Adam-rankone}. Note that our argument can establish non-asymptotic upper bound on the time $\tau$, though we omit the precise upper bound for succinctness.

\begin{corollary}[Finite passage toward the subcritical cutoff]
\label{cor:positive-sub-passage}
Under the hypotheses of Proposition~\ref{prop:lyapunov}, fix
$b<\overline W$ and $T\ge0$ with $w_T\le b$.  Then $w_{T+\tau}>b$ for some
finite $\tau\ge0$.
\end{corollary}

\begin{proof}

Suppose, for contradiction, that $w_{T+n}\le b$ for every $n\ge0$.  Put
\begin{equation}
 M:=\max\{\vpar_T,\Psi_T\},\qquad a:=F(M),\qquad
 {\varsigma}:=1-\Delta(a,b)\in(0,1),
 \label{eq:passage-constants}
\end{equation}
so that $0<a\le w_T\le b$, because $F$ is decreasing and $\vpar_T\le M$.

\emph{Step 1: induction.}  We claim that for every $n\ge0$
\begin{equation}
 \vpar_{T+n}\le M\qquad\text{and}\qquad\Psi_{T+n}\le\varsigma^n\Psi_T .
 \label{eq:passage-induction}
\end{equation}
Both hold at $n=0$.  If they hold at $n$, then
$(\xp_{T+n})^2\le\Psi_{T+n}\le M$ by \eqref{eq:Psi}, so
\eqref{eq:v-recursion} gives $\vpar_{T+n+1}\le\beta_2M+(1-\beta_2)M=M$;
hence $w_{T+n},w_{T+n+1}\in[a,b]$, and Proposition~\ref{prop:lyapunov}
gives $\Psi_{T+n+1}\le\varsigma\Psi_{T+n}\le\varsigma^{n+1}\Psi_T$.

\emph{Step 2: the second moment vanishes.}  Inserting
\eqref{eq:passage-induction} into \eqref{eq:unroll}, with $G_n$ as in
\eqref{eq:G},
\begin{equation}
 \vpar_{T+n}\ \le\ \beta_2^n\vpar_T+(1-\beta_2)\,\Psi_T\,G_n(\varsigma,\beta_2)
 \ \xrightarrow[n\to\infty]{}\ 0,
 \label{eq:passage-v-to-zero}
\end{equation}
since $\varsigma,\beta_2<1$.  Hence
$w_{T+n}=F(\vpar_{T+n})\to F(0)=w_{\max}>2>b$, contradicting
$w_{T+n}\le b$.

\end{proof}

\subsubsection{Existence of strictly subcritical four-cycles}
\label{app:four-cycle}

The parameter condition $\beta_2>\beta_1^2$ of
Proposition~\ref{prop:lyapunov} is substantive.  This subsection is the one
place where the standing assumption $\eps>0$ is relaxed: we work in the
scale-free case $\eps=0$, where $w_t$ is still defined by
\eqref{eq:w-rank-one} but $w_{\max}=+\infty$, and assume $u_i\ne0$ for every $i$, so that $D_t$ is
invertible and every term of \eqref{eq:F} is finite for $\vpar>0$; there
\begin{equation}
 w_t=F(\vpar_t)=\frac{c\,\eta_u}{\sqrt{\vpar_t}},\qquad
 \eta_u:=\eta\norm{u}_1 ,
 \label{eq:scale-free-reduction}
\end{equation}
so the reduced recursion is exactly one-dimensional Adam with learning
rate $\eta_u$ and $\eps=0$, and the one-dimensional four-cycle mechanism
embeds.  We give the construction explicitly.

\begin{proposition}[Strictly subcritical four-cycles]
\label{prop:four-cycle}
Assume $\eps=0$ and
\begin{equation}
 \sqrt2-1\le\beta_1<1,\qquad 0\le\beta_2\le\beta_1^2 .
 \label{eq:cycle-region}
\end{equation}
Then for every $\eta>0$ there is an initial state satisfying
Assumption~\ref{as:r0} whose reduced state $(\xp_t,\mpar_t,\vpar_t)$
has prime period four and satisfies $w_t<2$ for all $t$.
\end{proposition}

\begin{figure}[t]
\centering
\includegraphics[width=0.68\textwidth]{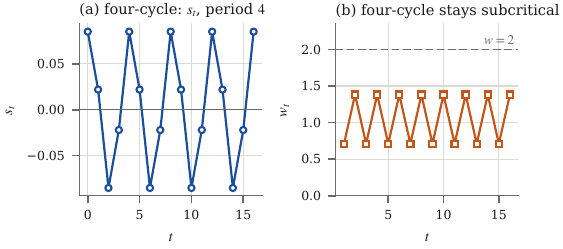}
\caption{The four-cycle of Proposition~\ref{prop:four-cycle} at
$\beta_1=0.5$, $\beta_2=0.2$, $\eps=0$, $\eta=0.1$: the constructed
$\zeta=0.25868$ gives an orbit of exact period four (a) whose two sharpness
values, $0.7095$ and $1.3825$, both stay strictly below $2$ (b).}
\label{fig:fourcycle}
\end{figure}

\begin{proof}
By \eqref{eq:scale-free-reduction} it suffices to exhibit the
one-dimensional four-cycle, which we do explicitly.
For $\zeta\in(0,1)$ define, only within this proof,
\begin{equation}
 R(\zeta):=\left[\frac{(1-\beta_1\zeta)(1+\zeta)}
 {(1-\zeta)(\beta_1+\zeta)}\right]^2,\qquad
 J(\zeta):=\frac{1-R(\zeta)\zeta^2}{R(\zeta)-\zeta^2} ,
 \label{eq:cycle-J}
\end{equation}
and let $\zeta_0$ be the first positive solution of
\begin{equation}
 \zeta\,\frac{(1-\beta_1\zeta)(1+\zeta)}{(1-\zeta)(\beta_1+\zeta)}=1 .
 \label{eq:zeta0}
\end{equation}
The unsquared ratio in $R$ exceeds one on $(0,1)$, because its numerator
minus its denominator is
\begin{equation}
 (1-\beta_1\zeta)(1+\zeta)-(1-\zeta)(\beta_1+\zeta)
 =(1-\beta_1)(1+\zeta^2)>0 ,
 \label{eq:cycle-ratio-positive}
\end{equation}
and the left-hand side of \eqref{eq:zeta0} vanishes at $\zeta=0$ and tends
to $+\infty$ as $\zeta\to1^-$, so $\zeta_0\in(0,1)$ exists.  On
$(0,\zeta_0]$ we have $R(\zeta)>1>\zeta^2$, so the denominator of $J$ is
positive and $J$ is continuous there.  Equation \eqref{eq:zeta0} says
$R(\zeta_0)\zeta_0^2=1$, hence
\begin{equation*}
 J(\zeta_0)=0,\qquad\text{while}\qquad R(0)=\beta_1^{-2}
 \ \text{ gives }\ J(0)=\beta_1^2 .
\end{equation*}
By the intermediate value theorem there is therefore
$\zeta\in(0,\zeta_0]$ with $J(\zeta)=\beta_2$ for every $\beta_2$ in
\eqref{eq:cycle-region}; in the boundary case $\beta_2=\beta_1^2$ the
identity $J'(0)=2\beta_1(1-\beta_1)^2>0$ supplies a second crossing at some
$\zeta>0$.

Fix such a $\zeta$ and set, again only within this proof,
\begin{equation}
 \nu_+:=\sqrt{\frac{1+\beta_2\zeta^2}{1+\beta_2}},\quad
 \nu_-:=\sqrt{\frac{\zeta^2+\beta_2}{1+\beta_2}},\quad
 \varphi_+:=\frac{1-\beta_1\zeta}{1-\zeta},\quad
 \varphi_-:=\frac{\beta_1+\zeta}{1+\zeta},\quad
 \kappa:=\frac{1-\beta_1}{1+\beta_1^2} .
 \label{eq:cycle-constants}
\end{equation}
The equation $J(\zeta)=\beta_2$ is exactly
$\varphi_+/\nu_+=\varphi_-/\nu_-$, so
\begin{equation}
 \varrho:=\eta_u\kappa\frac{\varphi_+}{\nu_+}
 =\eta_u\kappa\frac{\varphi_-}{\nu_-}>0
 \label{eq:cycle-amplitude}
\end{equation}
is well defined.  Initialize
\begin{equation*}
 \xp_0=\varrho,\qquad
 \mpar_0=-\lambda\kappa\varrho(\beta_1+\zeta),\qquad
 {\vpar_0=\varrho^2\nu_-^2}.
\end{equation*}
The identities
\begin{equation}
 \beta_2\nu_-^2+(1-\beta_2)=\nu_+^2,\qquad
 \beta_2\nu_+^2+(1-\beta_2)\zeta^2=\nu_-^2
 \label{eq:cycle-q-identities}
\end{equation}
verify the alternating second moments; the momentum recursion produces the
displayed $\mpar_t$ below; and $\varrho\nu_+=\eta_u\kappa\varphi_+$,
$\varrho\nu_-=\eta_u\kappa\varphi_-$ verify the position updates
$\varrho\mapsto\zeta\varrho$ and $\zeta\varrho\mapsto-\varrho$.  The
remaining two steps follow by half-turn symmetry.  Hence
\begin{equation}
\begin{array}{c|cccc}
 t\bmod4&0&1&2&3\\ \midrule
 \xp_t&\varrho&\zeta\varrho&-\varrho&-\zeta\varrho\\
 \mpar_t/\lambda&-\kappa\varrho(\beta_1+\zeta)&\kappa\varrho(1-\beta_1\zeta)&
 \kappa\varrho(\beta_1+\zeta)&-\kappa\varrho(1-\beta_1\zeta)\\
 {\vpar_t}&\varrho^2\nu_-^2&\varrho^2\nu_+^2&\varrho^2\nu_-^2&\varrho^2\nu_+^2
\end{array}
\label{eq:cycle-table}
\end{equation}
This also shows why the construction works for every $\eta>0$: changing
$\eta$ rescales $\varrho$, $\mpar_t$ and ${\sqrt{\vpar_t}}$ by the same
factor.

The two normalized sharpness values along the orbit are
\begin{equation}
 \frac{1+\beta_1^2}{1+\beta_1}\cdot\frac{1-\zeta}{1-\beta_1\zeta}
 \qquad\text{and}\qquad
 \frac{1+\beta_1^2}{1+\beta_1}\cdot\frac{1+\zeta}{\beta_1+\zeta} .
 \label{eq:cycle-w}
\end{equation}
The first is always below $2$; the second is below $2$ exactly when
\begin{equation}
 \zeta>\frac{1-2\beta_1-\beta_1^2}{1+2\beta_1-\beta_1^2},
 \label{eq:cycle-condition}
\end{equation}
whose right-hand side is nonpositive precisely when
$\beta_1\ge\sqrt2-1$.  Hence the orbit is strictly subcritical, and its
four values of $\xp_t$ are distinct, so the period is prime four.
\end{proof}

\subsection{Proof of the supercritical regime of Theorem~\ref{thm:Adam-rankone}}
\label{app:super}

In the supercritical region $w>2$ the sign of $\xp_th_t$ separates two
sharply different behaviors, recorded in the following lemma; the two
supercritical items of Theorem~\ref{thm:Adam-rankone} are its two parts
iterated in time.

\begin{lemma}[Supercritical phase geometry]
\label{lem:sign-geometry}
Suppose $w_{t+1}>2$.
\begin{enumerate}[leftmargin=1.6em,itemsep=1pt,topsep=2pt]
\item If $\xp_th_t>0$ then $\xp_t\xp_{t+1}<0$, $h_th_{t+1}<0$, hence
$\xp_{t+1}h_{t+1}>0$, and
$\abs{\xp_{t+1}}>(w_{t+1}-1)\abs{\xp_t}$.
\item If $\xp_th_t<0$ and $\xp_{t+1}h_{t+1}<0$ then
$\abs{\xp_{t+1}}<\abs{\xp_t}/(w_{t+1}-1)$.
\end{enumerate}
\end{lemma}

\begin{proof}
Write $w:=w_{t+1}>2$.

(i) Let $\sigma:=\operatorname{sign}(\xp_t)=\operatorname{sign}(h_t)$,
$a:=\abs{\xp_t}>0$ and $b:=\abs{h_t}>0$.  By \eqref{eq:A-map-entries},
\begin{equation}
 \xp_{t+1}=-\sigma\bigl[(w-1)a+\beta_1wb\bigr],\qquad
 h_{t+1}=-\sigma\bigl[(w-2)a+\beta_1(w-1)b\bigr].
 \label{eq:aligned-components}
\end{equation}
Both brackets are strictly positive because $w>2$, so both components
carry the sign $-\sigma$, and
\begin{equation}
 \abs{\xp_{t+1}}=(w-1)\abs{\xp_t}+\beta_1w\abs{h_t}
 >(w-1)\abs{\xp_t} .
 \label{eq:aligned-direct}
\end{equation}

(ii) If instead $h_t=-\xi\xp_t$ with $\xi>0$, then \eqref{eq:A-map-entries}
gives
\begin{equation}
 \xp_{t+1}=-\xp_t\bigl[(w-1)-\beta_1w\xi\bigr],\qquad
 h_{t+1}=-\xp_t\bigl[(w-2)-\beta_1(w-1)\xi\bigr],
 \label{eq:misaligned-components}
\end{equation}
so $\xp_{t+1}h_{t+1}<0$ holds exactly when the two brackets have opposite
signs, that is exactly on the strip
\begin{equation}
 \ell(w)<\xi<\upsilon(w),\qquad
 \ell(w):=\frac{w-2}{\beta_1(w-1)},\qquad
 \upsilon(w):=\frac{w-1}{\beta_1w},
 \label{eq:misaligned-strip}
\end{equation}
where $\ell$ and $\upsilon$ are used in this appendix only.  The strip is
nonempty because $w(w-2)<(w-1)^2$.  Within it,
$\abs{\xp_{t+1}}=\bigl((w-1)-\beta_1w\xi\bigr)\abs{\xp_t}$, and
$\xi>\ell(w)$ gives
\begin{equation}
 \abs{\xp_{t+1}}<\left(w-1-\frac{w(w-2)}{w-1}\right)\abs{\xp_t}
 =\frac{\abs{\xp_t}}{w-1},
 \label{eq:misaligned-direct}
\end{equation}
using $(w-1)^2-w(w-2)=1$.
\end{proof}

Under $\xp_th_t>0$ the first component of \eqref{eq:aligned-components} is
already reversed as soon as $w>1$; the hypothesis $w>2$ is what keeps the
second bracket positive, and hence keeps the cone $\xp_th_t>0$ invariant so
that the lemma can be iterated.  Iterating part (ii) gives the misaligned
item of Theorem~\ref{thm:Adam-rankone} directly: if $w_t>2$ and
$\xp_th_t<0$ for all $t\ge T$, then
$\abs{\xp_{t+1}}<\abs{\xp_t}/(w_{t+1}-1)<\abs{\xp_t}$ for every $t\ge T$.
Iterating part (i) gives the aligned item, since the expansion feeds the
second moments exactly as in the zero-momentum case.

\begin{corollary}[Aligned finite exit from every supercritical band]
\label{cor:aligned-exit}
Suppose $\xp_Th_T>0$ and $w_T\ge2+\delta$ with $\delta>0$.  Then, with
$\tau_\delta(T)$ as in \eqref{eq:times} and $N_{\rm sup}$,
$\overline N_{\rm sup}$ as in \eqref{eq:explicit-super},
\begin{equation}
 \tau_\delta(T)\ \le\ N_{\rm sup}\ \le\ \overline N_{\rm sup}
 =1+\left\lfloor
 \frac{\log\bigl(1+(A_\delta-\beta_2)Q_\delta/((1-\beta_2)(\xp_T)^2)\bigr)}
 {\log A_\delta}\right\rfloor\ <\ \infty .
 \label{eq:aligned-explicit}
\end{equation}
\end{corollary}

\begin{proof}
Note first that $\xp_T\ne0$, since $\xp_Th_T>0$, and that
$2+\delta\le w_T\le w_{\max}$.  Fix $n\ge1$ and suppose
$\tau_\delta(T)>n$, so that $w_{T+k}\ge2+\delta>2$ for $k=1,\dots,n$.  By
Lemma~\ref{lem:sign-geometry}(i), applied successively at
$t=T,\dots,T+n-1$, alignment is preserved at every such step and
\begin{equation*}
 \abs{\xp_{T+k}}\ \ge\ (w_{T+k}-1)\abs{\xp_{T+k-1}}
 \ \ge\ (1+\delta)\abs{\xp_{T+k-1}},
 \qquad k=1,\dots,n .
\end{equation*}
Thus \eqref{eq:expansion-hyp} holds, and Lemma~\ref{lem:expansion-exit}
gives $n<N_{\rm sup}$.  Hence $\tau_\delta(T)\le N_{\rm sup}\le
\overline N_{\rm sup}$, which is \eqref{eq:aligned-explicit}.

\end{proof}

\subsubsection{Proof of Lemma~\ref{lem:main-flip-momentum}}
\label{app:flip-momentum}

\begin{proof}[Proof of Lemma~\ref{lem:main-flip-momentum}]

By \eqref{eq:w-rank-one}, $W_{\mathrm{act},t}=w_{t+1}$, so the hypothesis
reads $w_{t+1}>2$ for every $t\in[T,T+n]$.

\emph{Step 1: the gradient.}  Lemma~\ref{lem:sign-geometry}(i) at $t=T$
gives $\xp_T\xp_{T+1}<0$ and $\xp_{T+1}h_{T+1}>0$, so its hypothesis holds
again at $t=T+1$.  By induction,
\begin{equation}
 \xp_th_t>0\qquad\text{and}\qquad\xp_t\xp_{t+1}<0\qquad(t\in[T,T+n]).
 \label{eq:flip-induction}
\end{equation}

\emph{Step 2: the momentum.}  By \eqref{eq:mu-from-h},
$\mpar_{t+1}=c\lambda\,(\xp_t+\beta_1h_t)$ with $c\lambda>0$ and
$\beta_1>0$.  Under \eqref{eq:flip-induction} the two summands have the
sign of $\xp_t$, so $\mpar_{t+1}$ has the sign of $\xp_t$ for
$t\in[T,T+n]$, and therefore
\begin{equation}
 \mpar_t\mpar_{t+1}<0\qquad(t\in[T+1,T+n]).
 \label{eq:flip-momentum-sign}
\end{equation}

\emph{Step 3: from scalars to cosines.}  Since $g_t=\lambda\xp_tu$ and,
under Assumption~\ref{as:r0}, $m_t=\mpar_tu$,
\begin{equation}
 \cos(g_t,g_{t+1})=\operatorname{sign}(\xp_t\xp_{t+1})=-1,\qquad
 \cos(m_t,m_{t+1})=\operatorname{sign}(\mpar_t\mpar_{t+1})=-1
 \label{eq:flip-cosines}
\end{equation}
on the ranges of \eqref{eq:flip-induction} and
\eqref{eq:flip-momentum-sign} respectively.

\end{proof}

\subsubsection{Existence of misaligned supercritical trajectories}
\label{app:persistent}

The contraction of Lemma~\ref{lem:sign-geometry}(ii) is available only while
misalignment survives, and the next proposition shows that it can survive
forever: there are trajectories that converge to the minimizer manifold
while the frozen map is unstable at every step.  Together with the
four-cycles of Appendix~\ref{app:four-cycle}, these are the one-dimensional
exceptions, and they survive in rank one.

\begin{proposition}[Supercritical convergence by persistent misalignment]
\label{prop:persistent}
Assume $0<\beta_1<1$, $w_{\max}>2$ and $\eps>0$.  Fix a level
$\vpar_0>0$ with $w_0=F(\vpar_0)>2$.  Then there exist $\xp_0\ne0$
and $\mpar_0$, realized by an initialization satisfying
Assumption~\ref{as:r0}, for which
\begin{equation}
 \xp_th_t<0,\qquad w_t\ge w_0,\qquad
 \abs{\xp_t}\le(w_0-1)^{-t}\abs{\xp_0}\qquad(t\ge0),
 \label{eq:persistent-bounds}
\end{equation}
and along this trajectory $\xp_t\to0$, $\mpar_t\to0$, $\vpar_t\to0$ and
$w_t\to w_{\max}>2$.  The full vector $x_t$ converges to a minimizer in
$u^\perp$, not necessarily to the origin.
\end{proposition}

\begin{proof}

Choose
\begin{equation}
 0<\abs{\xp_0}\le\sqrt{\vpar_0},
 \label{eq:persistent-box}
\end{equation}
and, for a trial slope $\zeta_0>0$, set
\begin{equation}
 h_0=-\zeta_0\xp_0,\qquad
 \mpar_0=c\lambda(h_0-\xp_0)=-c\lambda(1+\zeta_0)\xp_0 .
 \label{eq:persistent-init}
\end{equation}
These data are realized by $x_0=\xp_0u$, $m_0=\mpar_0u$ and
$v_0=\lambda^2\vpar_0\,u\odot u$.

\emph{The prescribed box is preserved.}  Whenever misalignment survives one
step, Lemma~\ref{lem:sign-geometry}(ii) gives
$\abs{\xp_{t+1}}<\abs{\xp_t}/(w_{t+1}-1)<\abs{\xp_t}$, because $w_{t+1}>2$.
Starting from \eqref{eq:persistent-box}, induction and the convex
combination form of \eqref{eq:v-recursion} therefore give
\begin{equation}
 (\xp_t)^2\le\vpar_0,\qquad \vpar_t\le\vpar_0,
 \qquad w_t=F(\vpar_t)\ge w_0>2
 \label{eq:small-box-invariant}
\end{equation}
at every surviving time.  Thus the supercritical hypothesis of
Lemma~\ref{lem:sign-geometry}(ii) is preserved for as long as the
trajectory stays misaligned.

\emph{Selection of the slope.}  Removing the alternating signs
by
\begin{equation*}
 \tilde\xp_t:=\operatorname{sign}(\xp_0)(-1)^t\xp_t,\qquad
 \tilde h_t:=\operatorname{sign}(\xp_0)(-1)^th_t
\end{equation*}
turns \eqref{eq:A-map} into a nonnegative map, and on a surviving orbit
\begin{equation*}
 \zeta_t:=-\frac{h_t}{\xp_t}=-\frac{\tilde h_t}{\tilde\xp_t}>0 .
\end{equation*}
By \eqref{eq:misaligned-strip} the next step is misaligned exactly when
$\ell(w_{t+1})<\zeta_t<\upsilon(w_{t+1})$, and then
\eqref{eq:misaligned-components} gives
\begin{equation}
 \zeta_{t+1}=\Gamma_{w_{t+1}}(\zeta_t),\qquad
 \Gamma_w(\zeta):=\frac{\beta_1(w-1)\zeta-(w-2)}{(w-1)-\beta_1w\zeta} ,
 \label{eq:Gamma}
\end{equation}
with $\Gamma$ used in this proof only.  For fixed $w>2$ the map $\Gamma_w$
is continuous and strictly increasing on the strip, since
\begin{equation*}
 \Gamma_w'(\zeta)=\frac{\beta_1}{\bigl((w-1)-\beta_1w\zeta\bigr)^{2}}>0
\end{equation*}
after the cancellation $(w-1)^2-w(w-2)=1$, and its endpoint limits are
\begin{equation}
 \lim_{\zeta\downarrow\ell(w)}\Gamma_w(\zeta)=0,\qquad
 \lim_{\zeta\uparrow\upsilon(w)}\Gamma_w(\zeta)=+\infty .
 \label{eq:Gamma-endpoints}
\end{equation}

We now build nested intervals.  The value $w_1$ depends on
$\xp_0$ and ${\vpar_0}$ but not on $\zeta_0$, so
\begin{equation*}
 \mathcal J_1:=\bigl(\ell(w_1),\upsilon(w_1)\bigr)
\end{equation*}
is exactly the set of slopes surviving one step; the map
$\zeta_0\mapsto\zeta_1$ is continuous on it and, by
\eqref{eq:Gamma-endpoints}, maps it onto $(0,\infty)$.  Suppose
inductively that $\mathcal J_N$ is a nonempty open interval such that every
$\zeta_0\in\mathcal J_N$ survives through time $N$, the map
$\zeta_0\mapsto\zeta_N$ is continuous on it, and its limits at the left and
right endpoints are $0$ and $+\infty$.  The quantities $w_{N+1}$,
$\ell(w_{N+1})$ and $\upsilon(w_{N+1})$ depend continuously on $\zeta_0$,
and both boundaries increase with $w$ because
$\ell'(w)=\beta_1^{-1}(w-1)^{-2}>0$ and
$\upsilon'(w)=\beta_1^{-1}w^{-2}>0$; hence, by
\eqref{eq:small-box-invariant},
\begin{equation}
 \ell(w_{N+1})\ge\ell(w_0)>0,\qquad
 \upsilon(w_{N+1})\le\upsilon(w_{\max})<\frac1{\beta_1} .
 \label{eq:uniform-strip}
\end{equation}
Thus the graph of $\zeta_N$ starts below the lower boundary in
\eqref{eq:misaligned-strip} and ends above the upper one, so a component
$\mathcal J_{N+1}$ can be chosen on which
$\ell(w_{N+1})<\zeta_N<\upsilon(w_{N+1})$ and whose left and right
endpoints meet the lower and the upper boundary respectively; then
$\overline{\mathcal J_{N+1}}\subset\mathcal J_N$, and
\eqref{eq:Gamma-endpoints} restores the endpoint limits $0$ and $+\infty$
for $\zeta_{N+1}$.  This completes the induction.  The nonempty compact
sets $\overline{\mathcal J_N}$ are nested, so we may choose
\begin{equation*}
 \zeta_0^\star\in\bigcap_{N\ge1}\overline{\mathcal J_N}\ \subset\
 \bigcap_{N\ge1}\mathcal J_N ,
\end{equation*}
and with this choice misalignment persists at every finite time.

\emph{Convergence estimates.}  Iterating
Lemma~\ref{lem:sign-geometry}(ii) and using $w_t\ge w_0$ gives the
geometric bound in \eqref{eq:persistent-bounds}.  Since
$\zeta_t<\upsilon(w_{t+1})<1/\beta_1$ and
$\mpar_t=-c\lambda(1+\zeta_t)\xp_t$, also $\mpar_t\to0$ geometrically;
\eqref{eq:v-recursion} then gives ${\vpar_t\to0}$, and
\eqref{eq:w-rank-one} gives $w_t\to w_{\max}$.  Finally $m_t=\mpar_tu$ and
$\norm{D_t^{-1}u}_2\le1/\eps$, so
$\sum_t\norm{x_{t+1}-x_t}_2<\infty$ and $x_t$ converges; its limit
satisfies $u\trans x_\infty=\lim_t\xp_t=0$ and is therefore a minimizer.
\end{proof}

\section{Numerical illustration of Section~\ref{sec:rank-1}}
\label{sec:numerics}

This appendix illustrates the results of Section~\ref{sec:rank-1} by
simulating the rank-one model directly.  All runs use $d=3$,
$u=(\tfrac23,\tfrac23,\tfrac13)$ and $\lambda=1.5$, with $\eps=0.025$ except in the scale-free four-cycle run ($\eps=0$).  Figure~\ref{fig:beta0} shows the $\beta_1=0$
statements of Appendix~\ref{app:beta0}.  In panel~(c) the trajectory does
not merely oscillate: it locks onto the exact orbit of
Proposition~\ref{prop:main-cycle}, on which $w_t=2$ identically.
Figure~\ref{fig:momentum} shows the positive-momentum results of
Appendices~\ref{app:positive-sub} and~\ref{app:super}, with $\beta_1=0.9$,
$\beta_2=0.99$, $w_{\max}=2.84$, for which $\overline W=1.9722$.
Figure~\ref{fig:fourcycle} (Appendix~\ref{app:four-cycle}) shows the
four-cycle and Figure~\ref{fig:exceptions} the misaligned orbit.

\begin{figure}[t]
\centering
\includegraphics[width=\textwidth]{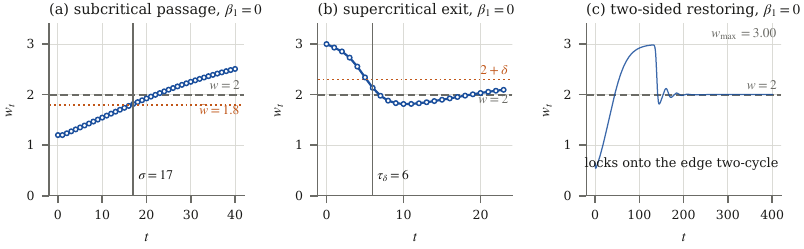}
\caption{Zero momentum ($\eta=0.05$, $\beta_2=0.9$, $w_{\max}=3$).
(a) From $w_0=1.2$ the target $\widehat w=1.8$ is reached at
$\sigma=17$, against the certificate $N_{\rm sub}=19$ of
Theorem~\ref{thm:sub-beta0}.
(b) From $w_0=w_{\max}$ the band $w\ge2.3$ is left at $\tau_\delta=6$,
against $N_{\rm sup}=10$ of Theorem~\ref{thm:super-beta0}.
(c) A longer run crosses $w=2$ from both sides and then settles on the exact
edge two-cycle of Proposition~\ref{prop:main-cycle}.}
\label{fig:beta0}
\end{figure}

\begin{figure}[t]
\centering
\includegraphics[width=\textwidth]{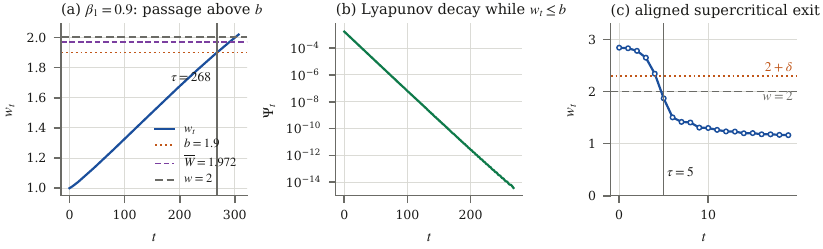}
\caption{Positive momentum ($\eta=0.9$, $\beta_1=0.9$, $\beta_2=0.99$,
$w_{\max}=2.84$, $\overline W=1.9722$).
(a) Starting at $w_0=1$, the normalized sharpness leaves the band
$w\le b=1.9$ at $t=268$, as forced by
Corollary~\ref{cor:positive-sub-passage}.
(b) The Lyapunov functional $\Psi_t$ falls over that stretch, an observed per-step factor $0.905$;
the certified factor is $1-\Delta(a,b)$ with
$a=w_0=1$, $b=1.9$ and $\Delta=1.93\cdot10^{-6}$.
(c) An aligned supercritical start leaves $w\ge2.3$ at $\tau=5$, against the
closed-form bound $16$ of Corollary~\ref{cor:aligned-exit}.}
\label{fig:momentum}
\end{figure}

\begin{figure}[t]
\centering
\includegraphics[width=0.5\textwidth]{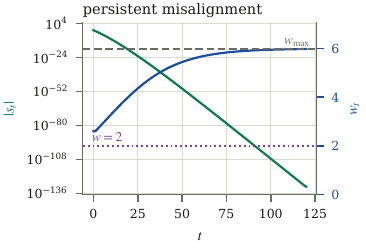}
\caption{The misaligned orbit of Proposition~\ref{prop:persistent} at
$\beta_1=0.5$, $\beta_2=0.9$, $\eta=0.3$, $w_{\max}=6$: over $120$ steps
$\abs{\xp_t}$ falls below $10^{-130}$ while $w_t$ rises from $2.6$ to
$5.98$, so the iterate converges to the minimizer manifold without ever
leaving the supercritical region.}
\label{fig:exceptions}
\end{figure}

\section{Proof of Theorem~\ref{prop:main-diag}}
\label{app:main-diag}

\paragraph{Setting.}
Let $H=\diag(\lambda_1,\dots,\lambda_n)\succ0$ and $\eps>0$.  Since $H$ and
$D_t$ are diagonal, \eqref{eq:adam} decouples: coordinate $i$ is the
rank-one dynamics of Appendix~\ref{app:frozen} with $u=e_i$ and
$\lambda=\lambda_i$, that is
\begin{equation}
\begin{gathered}
 m_{i,t+1}=\beta_1m_{i,t}+(1-\beta_1)\lambda_ix_{i,t},\qquad
 v_{i,t+1}=\beta_2v_{i,t}+(1-\beta_2)\lambda_i^2x_{i,t}^2,\\
 x_{i,t+1}=x_{i,t}-\frac{w_{i,t+1}}{c\lambda_i}\,m_{i,t+1},
\end{gathered}
 \label{eq:diag-coord}
\end{equation}
with $w_{i,t}=c\eta\lambda_i/(\sqrt{v_{i,t}}+\eps)$ as in
\eqref{eq:P-diag} and $w_{i,\max}=c\eta\lambda_i/\eps$.  Every statement of
Appendices~\ref{app:beta0}--\ref{app:positive-sub} therefore applies to
each coordinate separately, with $x_{i,t}$, $m_{i,t}$, $v_{i,t}/\lambda_i^2$,
$w_{i,t}$, $h_{i,t}:=(c\lambda_i)^{-1}m_{i,t}+x_{i,t}$ and $w_{i,\max}$ in
place of $\xp_t$, $\mpar_t$, $\vpar_t$, $w_t$, $h_t$ and $w_{\max}$; the only
coupling between coordinates is the hypothesis on $W_{\mathrm{act},t}$.
For $x_t\ne0$ put
\begin{equation}
 a_{i,t}:=\lambda_ix_{i,t}^2\,w_{i,t+1}\ \ge0,
 \qquad
 \sum_{i=1}^na_{i,t}>0,
 \label{eq:diag-weights}
\end{equation}
so that \eqref{eq:P-diag} reads
$W_{\mathrm{act},t}=\sum_ia_{i,t}w_{i,t+1}/\sum_ia_{i,t}$ and, for every
level $\omega$,
\begin{equation}
 \sum_{i=1}^na_{i,t}\bigl(w_{i,t+1}-\omega\bigr)
 =\Bigl(\sum_{i=1}^na_{i,t}\Bigr)\bigl(W_{\mathrm{act},t}-\omega\bigr).
 \label{eq:diag-excess}
\end{equation}
Throughout, $f(x)=\tfrac12\sum_i\lambda_ix_i^2$.

Since $v_{i,t}$ is a convex combination of $v_{i,T}$ and of
$\lambda_i^2x_{i,s}^2$, $T\le s<t$,
\begin{equation}
 v_{i,t}\ \le\ \max\Bigl\{v_{i,T},\ \lambda_i^2\max_{T\le s<t}x_{i,s}^2\Bigr\}
 \qquad(t>T).
 \label{eq:v-vs-past}
\end{equation}
In particular, along a bounded trajectory $v_{i,t}\le\overline V_i$ for
some constants $\overline V_i$, and every coordinate curvature is bounded
below: $w_{i,t}\ge\alpha_i:=c\eta\lambda_i/(\sqrt{\overline V_i}+\eps)>0$ for
all $t$.

\paragraph{The speed bound.}
A coordinate curvature cannot increase arbitrarily fast. Applying
Lemma~\ref{lem:jensen} to coordinate $i$ (equivalently, to the rank-one
model with $u=e_i$), we obtain
\begin{equation}
\begin{gathered}
 w_{i,t+1}
 \le
 \frac{w_{i,t}}
 {\sqrt{\beta_2}+(1-\sqrt{\beta_2})w_{i,t}/w_{i,\max}}
 \le
 \frac{w_{i,t}}{\sqrt{\beta_2}},\\
 \text{and hence}\qquad
 w_{i,t-j}\ge \beta_2^{j/2}w_{i,t},
 \qquad 0\le j\le t.
\end{gathered}
\label{eq:speed}
\end{equation}
Thus, if a coordinate reaches level $Q$ at time $t$, then $j$ steps
earlier it must have been at least $\beta_2^{j/2}Q$. In other words,
upward motion is geometrically limited, whereas downward motion need not be:
a large $x_{i,t}$ relative to $v_{i,t}$ can cause $w_{i,t}$ to drop sharply.

\subsection{Proof of the supercritical regime}
\subsubsection{Main argument}
\label{app:sup-main}

This subsection proves the supercritical item of
Theorem~\ref{prop:main-diag}.  The idea is simple.  We attach to each
coordinate a number $\Phi_{i,t}$, which we call its \emph{potential}.  The
potential has two properties.  It stays bounded.  And at every step it grows
by at least $x_{i,t}^2\,w_{i,t+1}(w_{i,t+1}-2)/(c\eta)^2$.

This lower bound is positive when $w_{i,t+1}>2$ and negative when
$w_{i,t+1}<2$.  Add the bounds over the coordinates, with the weights
$\lambda_i$.  By \eqref{eq:diag-excess}, the sum is a positive multiple of
$W_{\mathrm{act},t}-2$.  So if $W_{\mathrm{act},t}\ge2+\delta$ at every
step, the total potential grows at every step by at least a fixed multiple
of $\norm{x_t}^2$, implying $\sum_t\norm{x_t}^2<\infty$.

The rest of the subsection establishes this argument.
Proposition~\ref{prop:supercritical-abstract} states it with three
hypotheses: a floor on the curvatures, a bound on the total potential, and
the growth bound.  Lemma~\ref{lem:sup-bounds} gives the floor and shows that
trajectories stay bounded.  Lemma~\ref{lem:edge-identity} is an exact
identity that shows where the threshold $2$ comes from.
Lemma~\ref{lem:storage} constructs a potential at the default parameters.  All three are proved in
Appendix~\ref{app:sup-lemmas}.  The first two hold whenever
$\beta_1^2<\beta_2$.  Only the third uses specific parameters.

\paragraph{Setting and notation.}
The supercritical item assumes $0<\beta_1<1$, $\beta_1^2<\beta_2<1$, and that the scalar dynamics admits a potential in the sense of Definition~\ref{def:potential} below.  Lemma~\ref{lem:storage} proves this at $(\beta_1,\beta_2)=(0.9,0.999)$.  Put
$\gamma:=\beta_1/(1+\beta_1)$.

It is convenient to measure each coordinate in units of the step size.  Put
$y_{i,t}:=x_{i,t}/(c\eta)$, $\mu_{i,t}:=m_{i,t}/(c\eta\lambda_i)$,
$V_{i,t}:=v_{i,t}/(c\eta\lambda_i)^2$ and
$e_i:=\eps/(c\eta\lambda_i)=1/w_{i,\max}$.  Since $w_{i,\max}>2$, we have
$0<e_i<\tfrac12$.  In these units \eqref{eq:diag-coord} reads
\begin{equation}
\begin{gathered}
 \mu_{i,t+1}=\beta_1\mu_{i,t}+(1-\beta_1)y_{i,t},\qquad
 V_{i,t+1}=\beta_2V_{i,t}+(1-\beta_2)y_{i,t}^2,\\
 y_{i,t+1}=y_{i,t}-\frac{1+\beta_1}{1-\beta_1}\,w_{i,t+1}\mu_{i,t+1},
\end{gathered}
 \label{eq:normalised}
\end{equation}
with $w_{i,t+1}=(\sqrt{V_{i,t+1}}+e_i)^{-1}$.  When we look at one
coordinate only, we drop the index $i$.

\begin{definition}[Potential]
\label{def:potential}
Fix $\beta_1,\beta_2$.  A \emph{potential} for the scalar dynamics
\eqref{eq:normalised} is, for every $e\in(0,\tfrac12)$, a real function
$\Phi_t=\phi_e(y_{t-1},y_t,V_t)$ of the state with the following two
properties along every trajectory.  There is a constant $C$ with
$\Phi_t\le C\,y_t^2$ for all $t\ge1$.  And for all sufficiently large $t$,
$\Phi_{t+1}-\Phi_t\ge y_t^2\,w_{t+1}(w_{t+1}-2)$.
\end{definition}

\begin{proposition}[Abstract supercritical criterion]
\label{prop:supercritical-abstract}
Suppose that each coordinate $i$ carries a number $\Phi_{i,t}$ at each time
$t$, its potential, and put $\Phi^{\rm tot}_t:=\sum_i\lambda_i\Phi_{i,t}$.
Suppose that there are constants $\alpha>0$ and $\overline\Phi<\infty$ such
that, for all sufficiently large $t$ and every coordinate $i$:
\begin{enumerate}[leftmargin=1.6em,itemsep=1pt,topsep=2pt]
\item[\textup{(S1)}] \emph{Curvature floor.}  $w_{i,t}\ge\alpha$.
\item[\textup{(S2)}] \emph{Bounded potential.}
$\Phi^{\rm tot}_t\le\overline\Phi$.
\item[\textup{(S3)}] \emph{Growth bound.}
$\Phi_{i,t+1}-\Phi_{i,t}\ge y_{i,t}^2\,w_{i,t+1}(w_{i,t+1}-2)$.
\end{enumerate}
If $W_{\mathrm{act},t}\ge2+\delta$ for some $\delta>0$ and all sufficiently
large $t$, then $\sum_t\norm{x_t}^2<\infty$.  In particular $x_t\to0$,
$v_{i,t}\to0$ and $w_{i,t}\to w_{i,\max}$.
\end{proposition}

\begin{proof}
Choose $T$ so large that \textup{(S1)}--\textup{(S3)} hold and
$W_{\mathrm{act},t}\ge2+\delta$ for every $t\ge T$.  Fix $t\ge T$.  By
$y_{i,t}=x_{i,t}/(c\eta)$ and \eqref{eq:diag-weights},
\begin{equation}
 \lambda_iy_{i,t}^2w_{i,t+1}=\frac{a_{i,t}}{(c\eta)^2}
 \qquad\text{and, by \textup{(S1)},}\qquad
 a_{i,t}\ \ge\ (c\eta)^2\alpha\,\lambda_iy_{i,t}^2 .
 \label{eq:a-in-y}
\end{equation}
Multiply \textup{(S3)} by $\lambda_i$, sum over $i$, and use
\eqref{eq:diag-excess} with $\omega=2$:
\begin{align}
 \Phi^{\rm tot}_{t+1}-\Phi^{\rm tot}_t
 &=\sum_i\lambda_i\bigl(\Phi_{i,t+1}-\Phi_{i,t}\bigr)
 \ \ge\ \sum_i\lambda_iy_{i,t}^2w_{i,t+1}\bigl(w_{i,t+1}-2\bigr)
 &&\text{\small(S3)}
 \notag\\
 &=\frac{1}{(c\eta)^2}\sum_ia_{i,t}\bigl(w_{i,t+1}-2\bigr)
 &&\text{\small\eqref{eq:a-in-y}}
 \notag\\
 &=\frac{1}{(c\eta)^2}\Bigl(\sum_ia_{i,t}\Bigr)\bigl(W_{\mathrm{act},t}-2\bigr)
 &&\text{\small\eqref{eq:diag-excess}}
 \notag\\
 &\ge\frac{\delta}{(c\eta)^2}\sum_ia_{i,t}
 \ \ge\ \delta\alpha\sum_i\lambda_iy_{i,t}^2
 &&\text{\small$W_{\mathrm{act},t}\ge2+\delta$, \eqref{eq:a-in-y}}.
 \label{eq:storage-sum}
\end{align}
Sum \eqref{eq:storage-sum} over $t=T,\dots,n$; the left-hand side
telescopes and \textup{(S2)} bounds it:
\begin{equation}
 \delta\alpha\sum_{t=T}^{n}\sum_i\lambda_iy_{i,t}^2
 \ \le\ \sum_{t=T}^{n}\bigl(\Phi^{\rm tot}_{t+1}-\Phi^{\rm tot}_t\bigr)
 \ =\ \Phi^{\rm tot}_{n+1}-\Phi^{\rm tot}_T
 \ \le\ \overline\Phi-\Phi^{\rm tot}_T\qquad(n\ge T).
 \label{eq:storage-telescoped}
\end{equation}
The right-hand side does not depend on $n$, so
$\sum_t\sum_i\lambda_iy_{i,t}^2<\infty$, hence
\begin{equation*}
 \sum_t\norm{x_t}^2
 \ \le\ \frac{(c\eta)^2}{\min_i\lambda_i}\sum_t\sum_i\lambda_iy_{i,t}^2
 \ <\ \infty
 \qquad\text{and}\qquad x_t\to0 .
\end{equation*}
Then $v_{i,t+1}=\beta_2v_{i,t}+(1-\beta_2)\lambda_i^2x_{i,t}^2\to0$ by
\eqref{eq:diag-coord}, so $w_{i,t}=c\eta\lambda_i/(\sqrt{v_{i,t}}+\eps)\to
w_{i,\max}$.
\end{proof}

\subsubsection{Lemmas for Proposition~\ref{prop:supercritical-abstract}}
\label{app:sup-lemmas}

\begin{lemma}[Moment bound and curvature floor]
\label{lem:sup-bounds}
Let $0<\beta_1<1$ and $\beta_1^2<\beta_2<1$.  In every coordinate the
momentum is controlled by the second moment:
\begin{equation}
 {\abs{m_{i,t+1}}\ \le\ \beta_1^{t+1}\abs{m_{i,0}}+K\sqrt{v_{i,t+1}}},\qquad
 K^2:=\frac{(1-\beta_1)^2}{(1-\beta_2)(1-\beta_1^2/\beta_2)}.
 \label{eq:moment-bound}
\end{equation}
As a consequence, every trajectory is bounded:
\begin{equation}
 \limsup_{t\to\infty}\abs{x_{i,t}}\ \le\ \frac{\eta K}{1-\beta_1},\qquad
 \limsup_{t\to\infty}\sqrt{v_{i,t}}\ \le\ \frac{\eta K\lambda_i}{1-\beta_1}.
 \label{eq:x-bounded}
\end{equation}
Finally, the curvatures have a common floor that does not depend on the
coordinate:
\begin{equation}
 w_{i,t}\ \ge\ \alpha:=\frac{c}{K/(1-\beta_1)+c/2}\ >0
 \qquad\text{for every $i$ and all sufficiently large $t$.}
 \label{eq:w-floor}
\end{equation}
\end{lemma}

\begin{proof}
Fix $i$ and drop the index.  Unrolling \eqref{eq:diag-coord},
\begin{equation*}
 m_{t+1}=\beta_1^{t+1}m_0+(1-\beta_1)\sum_{j=0}^{t}\beta_1^{j}\lambda x_{t-j},
 \qquad
 v_{t+1}\ \ge\ (1-\beta_2)\sum_{j=0}^{t}\beta_2^{j}\lambda^2x_{t-j}^2 .
\end{equation*}
Write $\beta_1^{j}=(\beta_1/\sqrt{\beta_2})^{j}\beta_2^{j/2}$ and apply
Cauchy--Schwarz:
\begin{multline*}
 \sum_{j=0}^{t}\beta_1^{j}\abs{\lambda x_{t-j}}
 \ \le\ \Bigl(\sum_{j\ge0}\Bigl(\frac{\beta_1^2}{\beta_2}\Bigr)^{j}\Bigr)^{1/2}
        \Bigl(\sum_{j=0}^{t}\beta_2^{j}\lambda^2x_{t-j}^2\Bigr)^{1/2}\\
 \ \le\ \frac{1}{\sqrt{1-\beta_1^2/\beta_2}}\cdot\frac{\sqrt{v_{t+1}}}{\sqrt{1-\beta_2}}
 \ =\ \frac{K}{1-\beta_1}\sqrt{v_{t+1}} .
\end{multline*}
With the first display this is \eqref{eq:moment-bound}.  Hence, by
\eqref{eq:diag-coord} and $w_{t+1}=c\eta\lambda/(\sqrt{v_{t+1}}+\eps)$,
\begin{equation}
 \abs{x_{t+1}-x_t}=\frac{\eta\abs{m_{t+1}}}{\sqrt{v_{t+1}}+\eps}
 \ \le\ \eta K+\rho_t,
 \qquad
 \rho_t:=\frac{\eta\beta_1^{t+1}\abs{m_0}}{\eps}\ \to\ 0 .
 \label{eq:short-step}
\end{equation}

\emph{Position.}  Put $\xi_t:=m_t/\lambda-x_t$.  By \eqref{eq:diag-coord},
\begin{equation}
 \frac{m_{t+1}}{\lambda}=x_t+\beta_1\xi_t,
 \qquad
 \xi_{t+1}=\frac{m_{t+1}}{\lambda}-x_{t+1}=\beta_1\xi_t+(x_t-x_{t+1}).
 \label{eq:xi-recursion}
\end{equation}
Iterating the second identity from a time $T$ and using
\eqref{eq:short-step},
\begin{equation}
 \abs{\xi_t}\ \le\ \beta_1^{\,t-T}\abs{\xi_T}
 +\frac{\eta K+\sup_{s\ge T}\rho_s}{1-\beta_1}\quad(t\ge T),
 \qquad\text{so}\qquad
 \limsup_{t\to\infty}\abs{\xi_t}\ \le\ \frac{\eta K}{1-\beta_1}.
 \label{eq:xi-bound}
\end{equation}
Fix $\zeta>0$ and put
\begin{equation*}
 R:=\frac{\eta K}{1-\beta_1},\qquad
 b:=\frac{\beta_1\eta K}{1-\beta_1}+\frac{\zeta}{2},
 \qquad\text{so that}\qquad
 b+\eta K+\frac{\zeta}{2}=R+\zeta .
\end{equation*}
By \eqref{eq:xi-bound} and \eqref{eq:short-step} there is $T$ with
\begin{equation}
 \beta_1\abs{\xi_t}\le b,\qquad
 \rho_t\le\zeta/2,\qquad
 \abs{x_{t+1}-x_t}\le\eta K+\zeta/2
 \qquad(t\ge T).
 \label{eq:after-T}
\end{equation}
For $t\ge T$ write $x_{t+1}=x_t-d_t$ with
$d_t:=w_{t+1}m_{t+1}/(c\lambda)$.  If $\abs{x_t}>b$, the first identity in
\eqref{eq:xi-recursion} and \eqref{eq:after-T} give
\begin{equation}
 \operatorname{sign}d_t=\operatorname{sign}x_t,
 \qquad
 \frac{\abs{m_{t+1}}}{\lambda}\ \ge\ \abs{x_t}-\beta_1\abs{\xi_t}\ \ge\ \abs{x_t}-b,
 \qquad
 \abs{x_{t+1}}=\bigl|\abs{x_t}-\abs{d_t}\bigr| .
 \label{eq:step-inward}
\end{equation}
We first show that the trajectory eventually remains in
$\{|x|\le R+\zeta\}$.  Suppose $t\ge T$ and $\abs{x_t}\le R+\zeta$.
If $\abs{x_t}\le b$, then
\[
\abs{x_{t+1}}
\le \abs{x_t}+\abs{d_t}
\le b+\eta K+\zeta/2
=R+\zeta.
\]
If instead $b<\abs{x_t}\le R+\zeta$, \eqref{eq:step-inward} gives
\[
\abs{x_{t+1}}
\le \max\{\abs{x_t},\abs{d_t}\}
\le R+\zeta.
\]
Thus $\{|x|\le R+\zeta\}$ is forward invariant after time $T$.

It remains to show that the trajectory reaches this interval in finite time.
Suppose that $\abs{x_s}>R+\zeta$ for $T\le s\le t$.  By
\eqref{eq:step-inward} and \eqref{eq:after-T},
\[
\abs{d_s}\le \eta K+\zeta/2<\abs{x_s},
\]
so the step points inward and
\[
\abs{x_{s+1}}
=\abs{x_s}-\abs{d_s}
<\abs{x_s}.
\]
In particular, $\abs{x_s}\le\abs{x_T}$ throughout this period.  Hence
\eqref{eq:v-vs-past} yields
\[
v_{s+1}\le
\overline V:=\max\{v_T,\lambda^2x_T^2\}.
\]
Moreover,
\[
\abs{d_s}
=\frac{\eta\abs{m_{s+1}}}{\sqrt{v_{s+1}}+\eps}
\ge
\frac{\eta\lambda(\abs{x_s}-b)}
     {\sqrt{\overline V}+\eps}
>
\frac{\eta\lambda(\eta K+\zeta/2)}
     {\sqrt{\overline V}+\eps}
=:d_{\min}>0.
\]
Therefore, as long as the trajectory remains outside $R+\zeta$,
its distance from the origin decreases by at least $d_{\min}$ at each step:
\[
\abs{x_{t+1}}
\le
\abs{x_T}-(t-T+1)d_{\min}.
\]
It must consequently enter $\{|x|\le R+\zeta\}$ after finitely many steps,
and by the preceding argument it cannot leave afterward.  Since $\zeta>0$
is arbitrary,
\[
\limsup_{t\to\infty}\abs{x_t}\le R.
\]

Finally, \eqref{eq:diag-coord} gives
\[
v_t
\le
\beta_2^{\,t-T}v_T
+\lambda^2\sup_{s\ge T}x_s^2,
\]
and therefore
\[
\limsup_{t\to\infty}\sqrt{v_t}
\le
\lambda R
=
\frac{\eta K\lambda}{1-\beta_1},
\]
which proves \eqref{eq:x-bounded}.

\emph{Floor.}  $w_{\max}=c\eta\lambda/\eps>2$ gives
$\zeta':=c\eta\lambda/2-\eps>0$.  By \eqref{eq:x-bounded},
$\sqrt{v_t}\le\lambda R+\zeta'$ for all sufficiently large $t$, so
\begin{equation*}
 w_t=\frac{c\eta\lambda}{\sqrt{v_t}+\eps}
 \ \ge\ \frac{c\eta\lambda}{\lambda R+\zeta'+\eps}
 \ =\ \frac{c\eta\lambda}{\eta\lambda K/(1-\beta_1)+c\eta\lambda/2}
 \ =\ \frac{c}{K/(1-\beta_1)+c/2}=\alpha,
\end{equation*}
which is \eqref{eq:w-floor}.
\end{proof}

\begin{lemma}[Edge identity]
\label{lem:edge-identity}
Let $0<\beta_1<1$.  Fix one coordinate and write
\begin{equation}
 p:=w_t,\qquad w:=w_{t+1},\qquad u:=\frac wp,\qquad
 y:=y_t,\qquad y_-:=y_{t-1},\qquad \Delta:=y-y_- .
 \label{eq:storage-shorthand}
\end{equation}
Thus $p$ and $w$ are the curvatures at two consecutive times, $u$ is their
ratio, and $\Delta$ is the last step.  The positions satisfy the two-step
recursion
\begin{equation}
 y_{t+1}=\bigl[1+\beta_1u-(1+\beta_1)w\bigr]\,y-\beta_1u\,y_- .
 \label{eq:second-order}
\end{equation}
Moreover, put
\begin{equation}
\begin{gathered}
 E_t:=\frac{y_t^2-\beta_1y_{t-1}^2}{1+\beta_1},\qquad
 Q_t:=\beta_1(u\Delta-wy)^2=\frac{[y_{t+1}+(w-1)y]^2}{\beta_1},\\
 R_t:=\gamma(u-1)\bigl[(u+1)\Delta^2-2y\Delta\bigr],
\end{gathered}
 \label{eq:storage-pieces}
\end{equation}
so that $Q_{t-1}=[y+(p-1)y_-]^2/\beta_1$.  Then
\begin{equation}
 y_t^2\,w_{t+1}(w_{t+1}-2)\;=\;E_{t+1}-E_t-Q_t+R_t .
 \label{eq:edge-identity}
\end{equation}
\end{lemma}

\begin{proof}
The previous update, the last equation of \eqref{eq:normalised} at time
$t-1$, gives $\mu_t=-c\Delta/p$.  Substituting this into
\eqref{eq:normalised} removes the momentum and gives
\eqref{eq:second-order}.  Substituting \eqref{eq:second-order} into
$E_{t+1}-E_t$ and collecting terms gives \eqref{eq:edge-identity}.
\end{proof}

The potential is built from the pieces of Lemma~\ref{lem:edge-identity}
and one cubic term in the second moment.  Put $\nu:=1-\beta_2$ and
$\Gamma(V):=\tfrac{8\gamma}{3}\min\{\sqrt V,\,1-e\}^3$.  For $t\ge1$ let
\begin{equation}
 \Phi_t:=\begin{cases}
 -\dfrac c2\,(y_t-y_{t-1})^2, & w_t\le1,\\[6pt]
 E_t-\dfrac{\beta_1}{2}\,Q_{t-1}
 =\dfrac{y_t^2-\beta_1y_{t-1}^2}{1+\beta_1}-\dfrac12\bigl[y_t+(w_t-1)y_{t-1}\bigr]^2, & w_t>1,
 \end{cases}
 \ \ -\ \Gamma(V_t).
 \label{eq:storage-def}
\end{equation}
Below curvature one the potential is a single kinetic square.  Above
curvature one it is $E_t$ corrected by the previous iterate.  The cubic
term accounts for the change of curvature.

The proof of the growth bound uses only three facts about the dynamics and
four inequalities between the parameters.  The facts are:
\begin{enumerate}[leftmargin=2.2em,itemsep=1pt,topsep=2pt]
\item[\textup{(F1)}] the second-moment identity and the speed bound,
 $(1/w_{t+1}-e)^2=\beta_2(1/w_t-e)^2+\nu y_t^2$ and
 $w_{t+1}\le w_t/\sqrt{\beta_2}$;
\item[\textup{(F2)}] $V_t\ge\nu y_{t-1}^2$;
\item[\textup{(F3)}] eventually, $w_t>1$ implies $w_{t+1}>c$.
\end{enumerate}
\textup{(F1)} and \textup{(F2)} hold for all parameters.  \textup{(F3)} is
proved in Step~1 below.  The parameter inequalities are
\begin{equation}
 \beta_1^2+\beta_1>1,\qquad
 \beta_2\ge\frac{(1+\beta_1)^2}{4},\qquad
 4\gamma\nu<\frac c2-\Bigl(1-\frac{\beta_1^2}{2}\Bigr)c^2,\qquad
 (1+\beta_1)\sqrt{\frac{\beta_2}{\beta_2-\beta_1^2}}+\frac12<\frac1c .
 \label{eq:std-ineq}
\end{equation}
The case in which both curvatures exceed one also uses explicit numerical
bounds.  All of these hold at Adam's default parameters
$(\beta_1,\beta_2)=(0.9,0.999)$, where $c=1/19$, $\gamma=9/19$ and
$\nu=1/1000$.  There the four inequalities in \eqref{eq:std-ineq} read $1.71>1$,
$0.999\ge0.9025$, $0.00189<0.02467$ and $4.87<19$.  We prove the lemma at these parameters.

\begin{lemma}[Potential inequality]
\label{lem:storage}
Let $(\beta_1,\beta_2)=(0.9,0.999)$.  The potential \eqref{eq:storage-def}
satisfies $\Phi_t\le0$ if $w_t\le1$ and $\Phi_t\le y_t^2/(1+\beta_1)$ if
$w_t>1$.  In particular $\Phi_t\le y_t^2/(1+\beta_1)$.  Along every
trajectory of \eqref{eq:normalised}, from any initialization, and for all
sufficiently large $t$, it satisfies the growth bound
\begin{equation}
 \Phi_{t+1}-\Phi_t\ \ge\ y_t^2\,w_{t+1}(w_{t+1}-2).
 \label{eq:storage}
\end{equation}
So \eqref{eq:storage-def} is a potential in the sense of
Definition~\ref{def:potential}, with $C=1/(1+\beta_1)$.
\end{lemma}

\begin{proof}
The upper bounds hold because every term of \eqref{eq:storage-def} except
$y_t^2/(1+\beta_1)$ is nonpositive.

For the growth bound, use the shorthand \eqref{eq:storage-shorthand} and
let $\Psi:=\Phi_{t+1}-\Phi_t-y^2w(w-2)$ be the margin.  By
\eqref{eq:second-order}, $y_{t+1}=y+\beta_1u\Delta-(1+\beta_1)wy$.  We first
prove \textup{(F3)} and a bound for downward steps.  Then we show
$\Psi\ge0$ in four cases, according to whether $p$ and $w$ lie above or
below $1$.  Only Step~1 needs $t$ to be large.

\emph{Step 1: if $p>1$, then eventually $w>c$.}  Unrolling the momentum in
\eqref{eq:normalised} writes $y$ through the earlier positions:
\begin{equation*}
 y=\bigl[1-(1+\beta_1)p\bigr]y_--(1+\beta_1)p\sum_{j=1}^{t-1}\beta_1^jy_{t-1-j}
 -\frac{\beta_1^tp}{c}\,\mu_0 .
\end{equation*}
The second moment controls the earlier positions:
$V_t\ge\nu\sum_{j=0}^{t-1}\beta_2^jy_{t-1-j}^2$.  Cauchy--Schwarz with these
weights bounds the linear part of $y$.  The triangle inequality for
$\sqrt{V_{t+1}}=\norm{(\sqrt{\beta_2V_t},\sqrt\nu\,y)}_2$ then gives
\begin{equation*}
 \sqrt{V_{t+1}}\ \le\
 \sqrt{V_t\Bigl[1+\beta_2-2(1+\beta_1)p+\frac{(1+\beta_1)^2\beta_2}{\beta_2-\beta_1^2}\,p^2\Bigr]}
 +\frac{\beta_1^tp\sqrt\nu}{c}\,\abs{\mu_0}.
\end{equation*}
Now let $p>1$.  Then $V_t=(1/p-e)^2\le p^{-2}$, $1+\beta_2<2(1+\beta_1)p$
and $p\le1/e$.  Hence
\begin{equation*}
 \sqrt{V_{t+1}}+e\ \le\ (1+\beta_1)\sqrt{\frac{\beta_2}{\beta_2-\beta_1^2}}
 +e+\frac{\beta_1^t\sqrt\nu}{ce}\,\abs{\mu_0}.
\end{equation*}
Since $e<\tfrac12$, the last inequality in \eqref{eq:std-ineq} makes the
right-hand side smaller than $1/c$ for all large $t$.  So $w>c$.  This holds
for every initialization.

\emph{Step 2: a bound for downward steps.}  Let $u\le1$ and $y\ne0$.  Then
$1/w-e\ge u^{-1}(1/p-e)$, so \textup{(F1)} gives
$u^{-2}V_t\le\beta_2V_t+\nu y^2$.  Together with \textup{(F2)},
\begin{equation}
 \Bigl|\frac{y_-}{y}\Bigr|\ \le\ \frac{u}{\sqrt{1-\beta_2u^2}} .
 \label{eq:ratio-bound}
\end{equation}
If instead $y=0$, then $V_{t+1}=\beta_2V_t\le V_t$, so $w\ge p$.  In
particular a step from $p>1$ to $w\le1$ has $y\ne0$.

\emph{Step 3: both curvatures at most one.}  Then
$\min\{\sqrt V,1-e\}=1-e$ at both times, so the cubic term does not change,
and
\begin{equation*}
 \Psi=\frac c2\Delta^2-\frac c2\bigl[\beta_1u\Delta-(1+\beta_1)wy\bigr]^2
 +w(2-w)y^2 .
\end{equation*}
By Cauchy--Schwarz with the weights $\beta_1^2u^2$ and $1-\beta_1^2u^2$,
$\Psi\ge0$ whenever
$\tfrac c2(1+\beta_1)^2\,w/(2-w)\le1-\beta_1^2u^2$.  This holds because
$w\le1$, $u^2\le1/\beta_2$ by \textup{(F1)}, and
$\beta_2\ge(1+\beta_1)^2/4\ge2\beta_1^2/(1+\beta_1^2)$.  No lower bound on
$u$ is needed.

\emph{Step 4: upward crossings, $p\le1<w$.}  Here $\sqrt{V_t}\ge1-e>\sqrt{V_{t+1}}$,
so the cubic term $-\Gamma$ increases.  The rest of $\Psi$ is the quadratic
form $A\Delta^2+By\Delta+Cy^2$ with
\begin{equation*}
 A=\frac c2\bigl(1+\beta_1^2u^2\bigr),\qquad
 B=2\beta_1u\Bigl[\frac1{1+\beta_1}-\Bigl(1-\frac{\beta_1}2\Bigr)w\Bigr],\qquad
 C=c+\beta_1\Bigl(1-\frac{\beta_1}2\Bigr)w^2 .
\end{equation*}
Its determinant factors as
\begin{equation*}
 AC-\frac{B^2}4=\frac c2\,C+\frac{\beta_1^2u^2}{1+\beta_1}
 \Bigl[\frac c2-\Bigl(1-\frac{\beta_1}2\Bigr)(w-1)^2\Bigr].
\end{equation*}
Here $1<w\le\beta_2^{-1/2}\le2/(1+\beta_1)=1+c$, so the bracket is at least
$\tfrac c2-(1-\tfrac{\beta_1}2)c^2=\tfrac{c}{2(1+\beta_1)}(-1+4\beta_1-\beta_1^2)>0$
by $\beta_1^2+\beta_1>1$.  So the form is positive definite and $\Psi\ge0$.
This also covers $y=0$.

\emph{Step 5: downward crossings, $p>1\ge w$.}  Now $c<w\le1$ by Step~1, and
$0<u=w/p<w$.  The cubic term $-\Gamma(V)$ changes by at least $-4\gamma\nu y^2$: as a
function of $V$ its derivative lies in $[-4\gamma,0]$, and
$V_{t+1}-V_t\le\nu y^2$.  After this bound,
$\Psi\ge A\Delta^2+B\Delta y+Cy^2$ with
\begin{equation*}
\begin{gathered}
 A=\gamma+\tfrac12(p-1)^2-\tfrac c2\beta_1^2u^2,\qquad
 B=\beta_1(1-\beta_1)uw-2\gamma-p(p-1),\\
 C=2w-\bigl[1+\tfrac12(1-\beta_1^2)\bigr]w^2-c+\tfrac12p^2-4\gamma\nu .
\end{gathered}
\end{equation*}
By Step~2, $y\ne0$.  Put $z:=y_-/(uy)$, so that $\Delta=(1-uz)y$ and
\eqref{eq:ratio-bound} reads $(1-\beta_2u^2)z^2\le1$.  We claim
\begin{equation}
 \frac{\Psi}{y^2}\ \ge\ \mathcal H(w):=w-\Bigl(1-\frac{\beta_1^2}2\Bigr)w^2
 -\frac c2-4\gamma\nu\ >\ 0 .
 \label{eq:down-claim}
\end{equation}
For the first inequality, put $\vartheta:=\beta_1(1-\beta_1)/2$.  Subtract
$\mathcal H(w)+\tfrac12w(1-w)\bigl[1-(1-\beta_2u^2)z^2\bigr]$, which is at
least $\mathcal H(w)$, from $A(1-uz)^2+B(1-uz)+C$.  What remains is the
quadratic $az^2+bz+d$ in $z$, with
\begin{equation*}
\begin{gathered}
 a=\frac w2-wu+\Bigl[\frac12+\gamma-\frac{\beta_2}{2}w(1-w)\Bigr]u^2-\gamma\vartheta u^4,\\
 b=(w-u)-2\vartheta u^2(w-\gamma u),\qquad
 d=\frac w2+\vartheta u(2w-\gamma u).
\end{gathered}
\end{equation*}
Since $0<u\le w\le1$, we have $d\ge w/2>0$, and the two terms of $b$ are
nonnegative, so $b^2\le(w-u)^2+4\vartheta^2u^4(w-\gamma u)^2$.  The exact
factorization
\begin{equation*}
 a-\frac{(w-u)^2}{2w}=\gamma u^2-\gamma\vartheta u^4
 +\frac{u(1-w)}{2}\Bigl(2-\frac uw-\beta_2wu\Bigr)
\end{equation*}
has a nonnegative last term.  Hence
\begin{equation*}
 a-\frac{b^2}{4d}\ \ge\ \gamma u^2-\gamma\vartheta u^4-2\vartheta^2u^4w
 \ \ge\ \gamma u^2(1-2\vartheta)\ >\ 0,
\end{equation*}
using $2\vartheta\le\gamma$ and $u,w\le1$.  So the remaining quadratic is
positive, which proves the first inequality in \eqref{eq:down-claim}.
Finally, $\mathcal H$ is concave, $\mathcal H(c)>0$ by the third inequality
in \eqref{eq:std-ineq}, and
$\mathcal H(1)-\mathcal H(c)=(1-c)(\beta_1^2+\beta_1-1)/(1+\beta_1)>0$.  So
$\mathcal H>0$ on $[c,1]$.

\emph{Step 6: both curvatures above one.}  Put
$\chi:=\beta_1(1-\beta_1/2)=99/200$ and $k:=\chi-\gamma=81/3800$.  The
quadratic part of $\Psi$ is $A\Delta^2+2By\Delta+Cy^2$ (here $B$ is half
the cross coefficient), with
\begin{equation*}
 A=ku^2+\frac{(p-1)^2}2+\gamma,\qquad
 B=-\chi pu^2-\frac{p(p-1)}2+\gamma(u-1),\qquad
 C=p^2\Bigl(\chi u^2+\frac12\Bigr).
\end{equation*}
At both times $\sqrt V<1-e$, so $\Gamma(V)=\tfrac{8\gamma}3V^{3/2}$ there.
If $y=0$, then $w\ge p$ by Step~2, the cubic term does not decrease, and
$\Psi\ge A\Delta^2\ge0$.  So let $y\ne0$.

The cubic term is worst at $e=0$.  Put $s:=\sqrt{V_t}=1/p-e$ and
$d:=1/w-1/p$, so that $\sqrt{V_{t+1}}=s+d$.  Here $d$ is fixed once $p$
and $w$ are.  By \textup{(F1)}, $y^2=[(s+d)^2-\beta_2s^2]/\nu$, and the
change of the cubic term, divided by $y^2$, is
\begin{equation*}
 \rho_e:=-\frac{\Gamma(V_{t+1})-\Gamma(V_t)}{y^2}
 =-\frac{8\gamma\nu}{3}\,d\,\frac{3s^2+3ds+d^2}{\nu s^2+2ds+d^2}.
\end{equation*}
Its derivative in $s$ is
$-\tfrac{8\gamma\nu}{3}\,d^2\bigl[\beta_2s^2+2(1+\beta_2)s(s+d)+(s+d)^2\bigr]
/\bigl[(s+d)^2-\beta_2s^2\bigr]^2\le0$.  The denominator stays positive as
$s$ increases to $1/p$: for upward steps this follows from
$u<\beta_2^{-1/2}$, and for downward steps it is immediate.  So $\rho_e$ is
smallest at $s=1/p$, that is at $e=0$:
\begin{equation}
 \rho_e\ \ge\ \rho_0:=\frac{8\gamma\nu}{3p}\,\frac{u^3-1}{u(1-\beta_2u^2)} .
 \label{eq:cubic-e0}
\end{equation}
For a downward step,
\begin{equation}
 -\rho_0\ \le\ \frac{8\gamma\nu}{3}\Bigl(\frac1p+\frac1w\Bigr)
 \ \le\ \frac{16\gamma\nu}{3}=\frac{6}{2375}.
 \label{eq:cubic-down}
\end{equation}

\emph{Large downward steps, $u\le\beta_1$.}  Put $U:=1/p$ and
$P:=(AC-B^2)/p^2$, as a function of $v:=u^2$.  Direct differentiation gives
\begin{equation*}
 \frac{\partial^2P}{\partial v^2}=-2\chi\gamma+\frac{3\chi\gamma U}{2u}
 -\frac{\gamma(1-U)+2\gamma^2U^2}{4u^3}\ <\ 0\qquad(U\le u\le1).
\end{equation*}
At the two ends,
\begin{equation*}
 P\big|_{u=1}=\frac\chi2(p-2)^2\ge0,\qquad
 P\big|_{u=U}=\frac k2(1-2U)^2+\frac\gamma2U^2+\gamma kU^2(1-U)^2
 \ \ge\ \frac{81}{8968}>\frac{9}{1000}.
\end{equation*}
By concavity, $P\ge(1-u^2)\cdot\tfrac{9}{1000}\ge\tfrac{171}{100000}$.
Since $A/p^2\le\tfrac12$, minimizing over $\Delta$ gives
$C-B^2/A\ge\tfrac{171}{50000}>\tfrac{6}{2375}$.  Together with
\eqref{eq:cubic-down}, $\Psi>0$.

\emph{All other steps, $\beta_1\le u<\beta_2^{-1/2}<1.001$.}  Put
\begin{equation*}
 X:=p-2,\qquad z:=\frac\Delta y-2,\qquad \tau:=u-1,\qquad
 F_1:=\frac{p-1}2-\chi u^2,\qquad F_2:=\chi u^2+\frac12 .
\end{equation*}
Expanding about $\Delta=2y$, the quadratic part divided by $y^2$ is exactly
\begin{equation*}
 Az^2+2\bigl[F_1X-\gamma(2u+1)\tau\bigr]z+F_2X^2-4\gamma u\tau .
\end{equation*}
The bound \eqref{eq:cubic-e0} satisfies the exact identity
\begin{equation*}
 \rho_0-4\gamma u\tau=-\frac{4\gamma u}{p}X\tau+\frac{8\gamma}{p}T(u)\tau^2,
 \qquad
 T(u):=\frac{3\beta_2u^2(u+1)-\nu(2u+1)}{3u(1-\beta_2u^2)} .
\end{equation*}
On this interval $T(u)\ge\tfrac32u/(1-\beta_2u^2)$.  So $\Psi/y^2$ is
bounded below by a quadratic form in $(X,z,\tau)$.  Its $(X,z)$ block is
\begin{equation*}
 \mathcal B=\begin{pmatrix}F_2&F_1\\F_1&A\end{pmatrix},\qquad
 \det\mathcal B=\frac\chi2w^2+\gamma F_2(1-u^2)\ \ge\ \frac{6}{25}w^2>0 .
\end{equation*}
(For $u>1$ the second term is negative, but $u^2<1/\beta_2$ and $F_2<1$
give $\gamma F_2(1-u^2)>-0.0005$, while $(\chi/2-6/25)w^2>0.0075$ because
$w>1$.)  Also $A/p^2\le\tfrac12$, $\abs{F_1}/p\le\tfrac12$ and $F_2<1$.  Minimizing
over $(X,z)$ costs at most
\begin{equation*}
\begin{aligned}
 \frac{\gamma^2}{\det\mathcal B}\Bigl[\frac{4u^2A}{p^2}-\frac{4u(2u+1)F_1}{p}
 +(2u+1)^2F_2\Bigr]\tau^2
 &\ \le\ \frac{\gamma^2(10u^2+6u+1)}{(6/25)w^2}\,\tau^2\\
 &\ \le\ \frac{75}{4}\,\frac{U^2}{u^2}\,\tau^2
 \ \le\ \frac{625}{27}\,U^2\tau^2 ,
\end{aligned}
\end{equation*}
since $10u^2+6u+1<18$ and $\gamma<\tfrac12$.  The available coefficient of
$\tau^2$ is larger:
\begin{equation*}
 \frac{8\gamma}{p}T(u)\ \ge\ \frac{12\gamma Uu}{1-\beta_2u^2}
 \ \ge\ \frac{243}{10}\,U\ >\ \frac{625}{27}\,U^2,
\end{equation*}
using $u\ge9/10$, $\gamma\ge9/20$ and $1-\beta_2u^2<1/5$.  So the quadratic
form is nonnegative and $\Psi\ge0$.  This treats upward and downward steps
together and finishes the proof.
\end{proof}

\subsection{Proof of the subcritical regime }
\subsubsection{Main argument}
\label{app:sub-main}

This subsection proves the subcritical item of Theorem~\ref{prop:main-diag}.
Proposition~\ref{prop:subcritical-abstract} reduces it to three hypotheses
\textup{(H1)}--\textup{(H3)} with constants $\rho$, $CB$, $R$ and the
condition $RCB<1$.  Lemmas~\ref{lem:sub-energy}--\ref{lem:sub-compensation}
(Appendix~\ref{app:sub-lemmas}) supply the hypotheses;
Proposition~\ref{prop:sub-constants} (Appendix~\ref{app:sub-constants})
chooses the constants and gives the cutoff.

\paragraph{Setting and notation.}

Let $n=2$, $0<\beta_1<1$ and
$\beta_1^2<\beta_2<1$, and write $\sigma:=\sqrt{\beta_2}\in(\beta_1,1)$.  The
notation of this subsection is local to it.  By
\eqref{eq:h-z}--\eqref{eq:A-map}, \eqref{eq:diag-coord} reads
\begin{equation}
 z_{i,t+1}=\Amap(w_{i,t+1})\,z_{i,t},\qquad
 \Amap(w)=\begin{pmatrix}1-w&-\beta_1w\\2-w&\beta_1(1-w)\end{pmatrix},\qquad
 w_{i,t}\ge\sigma w_{i,t+1}.
 \label{eq:sub-z}
\end{equation}
The last inequality follows from $v_{i,t+1}\ge\sigma^2v_{i,t}$.  Put
\begin{equation}
 w_\pm:=1\pm\frac{2\sqrt{\beta_1}}{1+\beta_1},\qquad
 \kappa:=\frac{1-\beta_1}{c^2}=\frac{(1+\beta_1)^2}{1-\beta_1},
 \label{eq:sub-const}
\end{equation}
so that $w_-<c^2<w_+<2$ (indeed $w_+w_-=c^2$ and $2\sqrt{\beta_1}<1+\beta_1$).

Fix a level $q$, and suppose that $W_{\mathrm{act},t}\le q$ for all sufficiently
large $t$.  Write, as in \eqref{eq:diag-weights},
\begin{equation}
 u_{i,t}:=\lambda_ix_{i,t}^2\,w_{i,t+1}\,(w_{i,t+1}-q)_+,\qquad
 d_{i,t}:=\lambda_ix_{i,t}^2\,w_{i,t+1}\,(q-w_{i,t+1})_+
 \label{eq:excess-deficit-def}
\end{equation}
for the excess of a coordinate above $q$ and its deficit below $q$.  By
\eqref{eq:diag-excess} with $\omega=q$, the hypothesis
$W_{\mathrm{act},t}\le q$ is exactly
\begin{equation}
 u_{i,t}\le d_{j,t}\quad(j\ne i).
 \label{eq:excess-vs-deficit}
\end{equation}

\begin{proposition}[Abstract subcritical criterion]
\label{prop:subcritical-abstract}
Let $0<q<2$, fix a threshold $S\in[c^2,w_+)$, and let
\begin{equation}
\begin{gathered}
 E_{i,t}:=\lambda_i\,z_{i,t}\trans\mathcal P_S(w_{i,t})\,z_{i,t},\qquad
 \mathcal P_S(w):=\varphi(\bar w)M(\bar w),\qquad
 \bar w:=\min\{S,\max\{c^2,w\}\},\\
 \varphi(w):=\frac{w_+-c^2}{w_+-w},\qquad
 M(w):=\begin{pmatrix}2-w&\tfrac{1-\beta_1}{2}(w-1)\\[2pt]
 \tfrac{1-\beta_1}{2}(w-1)&\beta_1w\end{pmatrix}
\end{gathered}
 \label{eq:sub-energy}
\end{equation}
be the coordinate energies ($E_{i,t}\ge0$ since
$\mathcal P_S(w)\succeq M(c^2)\succ0$, Lemma~\ref{lem:sub-energy}).  Let $\rho<1$ and $R,C,B,N>0$ be constants such that, for
all sufficiently large $t$, each step of each coordinate is of one of two
kinds, \emph{low} ($w_{i,t+1}\le S$) or \emph{high} ($w_{i,t+1}>S$):
\begin{enumerate}[leftmargin=1.6em,itemsep=1pt,topsep=2pt]
\item[\textup{(H1)}] \emph{Low-curvature contraction.}  On a low step,
$E_{i,t+1}\le\rho\,E_{i,t}$; a low window of length $N$ contracts by the
factor $B$.
\item[\textup{(H2)}] \emph{Compensation.}
If coordinate $i$ is high at time $t$, then the other coordinate
$j\neq i$ is low on the preceding window of length $N$, and
\begin{equation}
 \max\{u_{i,t},u_{i,t-1}\}
 \le
 CB\max\{E_{j,t-N},E_{j,t-N-1}\}.
 \label{eq:compensation-abstract}
\end{equation}
\item[\textup{(H3)}] \emph{Excursion reset.}  On a high step,
$E_{i,t+1}\le R\max\{u_{i,t},u_{i,t-1}\}$.
\end{enumerate}
If $RCB<1$, then $W_{\mathrm{act},t}\le q$ cannot hold for all sufficiently large $t$;
that is, $W_{\mathrm{act},t}>q$ for arbitrarily large $t$.
\end{proposition}

\begin{proof}
Suppose $W_{\mathrm{act},t}\le q$ for all $t\ge T$, with $T$ so large that
\textup{(H1)}--\textup{(H3)} hold for $t\ge T$.  Put
\begin{equation}
 \bar\rho:=\max\{\rho,RCB\}<1,
 \qquad
 \mathcal E_t:=\max_{\substack{i\in\{1,2\}\\ t-N-1\le s\le t}}E_{i,s}.
 \label{eq:sliding-energy}
\end{equation}
For $t\ge T$ and $i=1,2$,
\begin{equation}
 E_{i,t+1}\ \le\ \bar\rho\,\mathcal E_t :
 \label{eq:new-energy-contracts}
\end{equation}
on a low step $E_{i,t+1}\le\rho E_{i,t}$ by \textup{(H1)}; on a high step
$E_{i,t+1}\le R\max\{u_{i,t},u_{i,t-1}\}\le RCB\max\{E_{j,t-N},\allowbreak E_{j,t-N-1}\}$
($j\ne i$) by \textup{(H3)}, \textup{(H2)}.  Hence $\mathcal E_{t+1}\le\mathcal E_t$.  Every entry of
$\mathcal E_{t+N+2}$ is some $E_{i,s+1}$ with $t\le s\le t+N+1$, so
\eqref{eq:new-energy-contracts} gives
\begin{equation}
 \mathcal E_{t+N+2}\ \le\ \bar\rho\,\mathcal E_t,
 \qquad
 \lambda_ix_{i,T+n}^2\ \le\ E_{i,T+n}\ \le\ \mathcal E_T\,\bar\rho^{\,\lfloor n/(N+2)\rfloor}
 \qquad(n\ge0),
 \label{eq:x-geometric}
\end{equation}
the first inequality in the second display by Lemma~\ref{lem:sub-energy}.
By \eqref{eq:diag-coord} and \eqref{eq:x-geometric},
\begin{gather*}
 v_{i,T+n}
 =\beta_2^{\,n}v_{i,T}+(1-\beta_2)\lambda_i^2\sum_{s=0}^{n-1}\beta_2^{\,n-1-s}x_{i,T+s}^2
 \ \xrightarrow[n\to\infty]{}\ 0,\\
 w_{i,T+n}=\frac{c\eta\lambda_i}{\sqrt{v_{i,T+n}}+\eps}
 \ \to\ w_{i,\max}>2>q .
\end{gather*}
So $w_{i,t+1}>q$ for $i=1,2$ and all large $t$, and \eqref{eq:P-diag} with
$x_t\ne0$ gives $W_{\mathrm{act},t}>q$, a contradiction.
\end{proof}

\paragraph{Proof of the subcritical item.}

Let $q_0<\overline W_{\rm diag}(\beta_1,\beta_2)$ as in
\eqref{eq:Wdiag-def}, $N:=\lfloor L\rfloor+1$ and $q:=S\sigma^{N+1}$.  By
Proposition~\ref{prop:sub-constants}, $q\ge\overline W_{\rm diag}>q_0$, and
if $W_{\mathrm{act},t}\le q$ for all large $t$ then
\textup{(H1)}--\textup{(H3)} hold with $RCB<1$, which
Proposition~\ref{prop:subcritical-abstract} rules out.  Hence
$W_{\mathrm{act},t}>q>q_0$ at arbitrarily late times.
\qed

\subsubsection{Lemmas for Proposition~\ref{prop:subcritical-abstract}}
\label{app:sub-lemmas}

\begin{lemma}[Coordinate energy]
\label{lem:sub-energy}
Let $c^2\le S<w_+$, let $E_{i,t}$ be the energy \eqref{eq:sub-energy}, and
put
\begin{equation}
 \mu:=\frac{\beta_1}{\sigma}\Bigl(1+\frac{1-\sigma}{w_+-1}\Bigr),\qquad
 G(w):=\mu\,\frac{w_+-\sigma w}{w_+-w}
 \label{eq:sub-G}
\end{equation}
and, for $\beta_1<\theta<1$, the contraction ceiling
\begin{equation}
 C(\theta):=\begin{cases}
 c^2,&\theta\le G(c^2),\\[3pt]
 w_+\dfrac{\theta-\mu}{\theta-\sigma\mu},&\theta>G(c^2),
 \end{cases}
 \qquad c^2\le C(\theta)<w_+ .
 \label{eq:sub-C}
\end{equation}
($G$ is increasing, and $G(C(\theta))=\theta$ in the second case.)  Then
$\lambda_ix_{i,t}^2\le E_{i,t}$, and whenever
$w=w_{i,t+1}\le\min\{S,C(\theta)\}$ for some $\theta\in(\beta_1,1)$,
\begin{equation}
 E_{i,t+1}\ \le\ \bigl[1-\min\{\kappa w,\,1-\theta\}\bigr]\,E_{i,t}.
 \label{eq:sub-contract}
\end{equation}
\end{lemma}

\begin{proof}
Direct computation gives
\begin{equation*}
 \Amap(w)\trans M(w)\Amap(w)=\beta_1M(w),\qquad
 4\det M(w)=(1+\beta_1)^2(w_+-w)(w-w_-),
\end{equation*}
so $M(w)\succ0$ for $w_-<w<w_+$ ($M_{11}=2-w>0$) and $M(w_\pm)\succeq0$.
Since $M$ is affine,
\begin{equation*}
 \frac{d}{dw}\,\frac{M(w)}{w_+-w}=\frac{M(w_+)}{(w_+-w)^2}\ \succeq\ 0 .
\end{equation*}
Thus $\mathcal P_S$ is nondecreasing, $\mathcal P_S(w)\succeq M(c^2)$, and
$z\trans M(c^2)z=x^2+\beta_1c^2\bigl(h-2x/(1-\beta_1)\bigr)^2\ge x^2$.

For $0\le w\le c^2$, convexity of $X\mapsto X\trans M(c^2)X$ along the
affine $w\mapsto\Amap(w)$ gives
\begin{equation*}
 \Amap(w)\trans M(c^2)\Amap(w)
 \preceq\Bigl(1-\frac{w}{c^2}\Bigr)\Amap(0)\trans M(c^2)\Amap(0)
 +\frac{w}{c^2}\,\beta_1M(c^2)
 \preceq(1-\kappa w)\,M(c^2),
\end{equation*}
since $\Amap(0)$ fixes $x$ and multiplies the second square by
$\beta_1^2$.  With $\mathcal P_S(w)=M(c^2)\preceq\mathcal P_S(w_{i,t})$ this
is $E_{i,t+1}\le(1-\kappa w)E_{i,t}$.

For $c^2\le w\le S$, put $p:=\max\{\sigma w,c^2\}\in[c^2,w]$.  Affinity also
gives
\begin{equation*}
 M(p)\ \succeq\ \frac{p-w_-}{w-w_-}\,M(w),\qquad
 \frac{w-w_-}{p-w_-}\ \le\ \frac1\sigma\Bigl(1+\frac{1-\sigma}{w_+-1}\Bigr).
\end{equation*}
(the difference of the two sides of the first is
$(w-p)M(w_-)/(w-w_-)\succeq0$; the second uses $w\le p/\sigma$, $p\ge c^2$
and $w_-/(c^2-w_-)=1/(w_+-1)$).  By \eqref{eq:sub-z}, $\bar w_{i,t}\ge p$,
so $\mathcal P_S(w_{i,t})\succeq\mathcal P_S(p)$ and
\begin{equation*}
 E_{i,t+1}=\beta_1\varphi(w)\,\lambda_iz_{i,t}\trans M(w)z_{i,t}
 \ \le\ \beta_1\,\frac{w_+-p}{w_+-w}\,\frac{w-w_-}{p-w_-}\,E_{i,t}
 \ \le\ G(w)\,E_{i,t},
\end{equation*}
the last step by $p\ge\sigma w$.  If $\theta\le G(c^2)$ then $C(\theta)=c^2$
and only the case $w\le c^2$ occurs; otherwise $G(w)\le G(C(\theta))=\theta$
for $c^2\le w\le C(\theta)$.  Both $1-\kappa w$ and $\theta$ are at most
$1-\min\{\kappa w,1-\theta\}$.
\end{proof}

\begin{lemma}[Reset]
\label{lem:sub-reset}
For $0<q<\sigma S$ put
\begin{equation}
\begin{gathered}
 J(q,S):=1+2\Bigl(\frac1{\sigma S}-1\Bigr)_{+}+\frac{1-\sigma}{\sigma(\sigma S-q)},\\
 R(q,S):=(3-\beta_1)\,\varphi(S)\,J(q,S)^2\max\Bigl\{\frac{\beta_1}{S(S-q)},\,(1+\beta_1)^2\Bigr\}.
\end{gathered}
 \label{eq:sub-JR}
\end{equation}
and let $u_{i,t}=\lambda_ix_{i,t}^2\,w_{i,t+1}(w_{i,t+1}-q)_+$ be as in
\eqref{eq:excess-deficit-def}.  On a step with $w_{i,t+1}>S$ and $t\ge1$,
\begin{equation}
 E_{i,t+1}\ \le\ R(q,S)\max\{u_{i,t},u_{i,t-1}\}.
 \label{eq:sub-reset}
\end{equation}
The bound remains valid with $R(Q,S)$ for $q\le Q<\sigma S$.
\end{lemma}

\begin{proof}
Write $w=w_{i,t+1}$, $p=w_{i,t}\ge\sigma w>q$, $z=(x,h)\trans=z_{i,t}$,
and normalize $\max\{u_{i,t},u_{i,t-1}\}=\lambda_i=1$ (the zero case is
trivial).  Then
\begin{equation*}
 \abs{x}\le\frac1{\sqrt{w(w-q)}},\qquad
 \abs{x_{i,t-1}}\le\frac1{\sqrt{p(p-q)}},\qquad
 h=(1-p^{-1})\,x+p^{-1}x_{i,t-1},
\end{equation*}
the last by $x_{i,t}-x_{i,t-1}=-w_{i,t}m_{i,t}/(c\lambda_i)$
\eqref{eq:diag-coord}.  With $r:=(w-q)/(\sigma w-q)\ge\sqrt{(w-q)/(\sigma(\sigma w-q))}$,
\begin{gather*}
 \abs h\sqrt{w(w-q)}\ \le\ \abs{1-p^{-1}}+\frac rp
 \ \le\ 1+2\Bigl(\frac1p-1\Bigr)_++\frac{r-1}{p},\\
 \frac{r-1}{p}=\frac{(1-\sigma)w}{p(\sigma w-q)}\ \le\ \frac{1-\sigma}{\sigma(\sigma S-q)}
\end{gather*}
($p\ge\sigma w$, $w>S$, $p\ge\sigma S$), hence
$\norm{z}_\infty\le J(q,S)/\sqrt{w(w-q)}$.

Let $\mathbf 1=(1,1)\trans$, $\mathbf b=(1,\beta_1)\trans$ and
$\norm{y}_{M(S)}:=(y\trans M(S)y)^{1/2}$.  Then
\begin{gather*}
 \Amap(w)=\Amap(S)-(w-S)\,\mathbf 1\mathbf b\trans,\qquad
 \norm{\Amap(S)z}_{M(S)}=\sqrt{\beta_1}\,\norm{z}_{M(S)},\\
 \norm{z}_{M(S)}^2\le(3-\beta_1)\norm{z}_\infty^2,\qquad
 \norm{\mathbf 1}_{M(S)}^2=1+\beta_1\le3-\beta_1,
\end{gather*}
and the triangle inequality gives
\begin{equation*}
 \norm{\Amap(w)z}_{M(S)}\ \le\ \frac{J(q,S)\sqrt{3-\beta_1}}{\sqrt{w(w-q)}}\,
 \bigl[\sqrt{\beta_1}+(w-S)(1+\beta_1)\bigr].
\end{equation*}
Since $\sqrt{w(w-q)}\ge\sqrt{S(S-q)}+(w-S)$,
\begin{equation*}
 \frac{\sqrt{\beta_1}+(w-S)(1+\beta_1)}{\sqrt{w(w-q)}}
 \ \le\ \max\Bigl\{\frac{\sqrt{\beta_1}}{\sqrt{S(S-q)}},\,1+\beta_1\Bigr\};
\end{equation*}
squaring and multiplying by $\varphi(S)$ ($\mathcal P_S(w)=\varphi(S)M(S)$)
gives \eqref{eq:sub-reset}.  Replacing $q$ by $Q\ge q$ only enlarges the
position bounds.
\end{proof}

\begin{lemma}[Compensation]
\label{lem:sub-compensation}
Let $\theta\in(\beta_1,1)$, let $N\ge1$ be an integer, let
$0<q\le\min\{S\sigma^{N+1},C(\theta)\}$, and suppose that
$W_{\mathrm{act},s}\le q$ for all $s\ge T$.  If $t\ge T+N+1$ and
$w_{i,t+1}>S$, then the other coordinate $j\ne i$ has $w_{j,k+1}\le q$ for
$t-N-1\le k\le t$ (in particular it is low there), and
\begin{equation}
\begin{gathered}
 \max\{u_{i,t},u_{i,t-1}\}\le D\max\{E_{j,t-N},E_{j,t-N-1}\},\\
 D:=\max\Bigl\{\frac{q}{e\kappa N\sigma^N},\ \frac{q^2\theta^N}{4}\Bigr\}.
\end{gathered}
 \label{eq:sub-window}
\end{equation}
\end{lemma}

\begin{proof}
By \eqref{eq:excess-vs-deficit}, $u_{i,s}\le d_{j,s}$ for $s\ge T$, so at
most one curvature exceeds $q$ at any time.  By \eqref{eq:sub-z},
$w_{i,k+1}\ge\sigma^{t-k}w_{i,t+1}>\sigma^{N+1}S\ge q$ for
$t-N-1\le k\le t$, hence $w_{j,k+1}\le q$ there.  For $s\in\{t-1,t\}$ put
$u:=w_{j,s+1}$; the preceding $N$ curvatures of $j$ lie in
$[\sigma^Nu,q]\subset[0,\min\{S,C(\theta)\}]$, so Lemma~\ref{lem:sub-energy}
and $\lambda_jx_{j,s}^2\le E_{j,s}$ give
\begin{multline*}
 u_{i,s}\ \le\ d_{j,s}
 \ \le\ u(q-u)\bigl[1-\min\{\kappa\sigma^Nu,1-\theta\}\bigr]^NE_{j,s-N}\\
 \ \le\ \max\Bigl\{qu\,e^{-\kappa\sigma^NNu},\ \frac{q^2}{4}\theta^N\Bigr\}E_{j,s-N}
 \ \le\ D\,E_{j,s-N},
\end{multline*}
using $u(q-u)\le\min\{qu,q^2/4\}$, $1-a\le e^{-a}$ and
$\sup_{u\ge0}ue^{-Au}=1/(eA)$.
\end{proof}

\subsubsection{The constants and the explicit cutoff}
\label{app:sub-constants}
\paragraph{The cutoff.}

The following explicit choices give a cutoff that depends on
$(\beta_1,\beta_2)$ only:
\begin{equation}
\begin{gathered}
 \theta:=\min\bigl\{\beta_1^{3/4},\ \beta_1+\sqrt{(1-\beta_1)(1-\beta_2)}\bigr\},\qquad
 S:=C(\sqrt{\beta_1}),\qquad
 Q:=\min\{\beta_2S,\ C(\theta)\},\\
 R:=R(Q,S)=(3-\beta_1)\,\frac{w_+-c^2}{w_+-S}
 \Bigl(1+2\Bigl(\frac1{\sigma S}-1\Bigr)_{+}+\frac{1-\sigma}{\sigma(\sigma S-Q)}\Bigr)^2\\
 \qquad\qquad\times\max\Bigl\{\frac{\beta_1}{S(S-Q)},\,(1+\beta_1)^2\Bigr\},\\
 \tau:=-\log\sigma,\qquad
 L:=\max\Bigl\{\frac{\log(S/Q)}{\tau}-1,\ \frac{R\sigma S}{e\kappa},\
 \frac{\log(R\beta_2S^2/4)}{2\tau-\log\theta}\Bigr\},\\
 \boxed{\ \overline W_{\rm diag}(\beta_1,\beta_2):=S\,\sigma^{L+2}\ }.
\end{gathered}
 \label{eq:Wdiag-def}
\end{equation}
Here $\theta$ and $\sqrt{\beta_1}$ lie in
$(\beta_1,1)$, $C$ is the contraction ceiling \eqref{eq:sub-C}, and $R$ is
the constant \eqref{eq:sub-JR} at $q=Q$.  All quantities are elementary;
$0<\overline W_{\rm diag}<S<w_+<2$.

\begin{proposition}[When $RCB<1$ holds]
\label{prop:sub-constants}
Let $\theta$, $S$, $Q$ and $R$ be as in \eqref{eq:Wdiag-def}, and let
$N\ge1$ be an integer with
\begin{equation}
 S\sigma^{N+1}\le Q,\qquad
 \frac{R\sigma S}{e\kappa N}<1,\qquad
 \frac{R\beta_2S^2}{4}\,(\beta_2\theta)^N<1 .
 \label{eq:sub-N}
\end{equation}
Put $q:=S\sigma^{N+1}$.  Then $0<q\le Q<2$, and if
$W_{\mathrm{act},t}\le q$ for all $t\ge T$, then for all sufficiently large
$t$ the hypotheses \textup{(H1)}--\textup{(H3)} of
Proposition~\ref{prop:subcritical-abstract} hold at level $q$ and threshold
$S$, with
\begin{equation*}
 \rho=\max\{1-\kappa\alpha,\sqrt{\beta_1}\}<1,\qquad
 CB=D\ \text{of \eqref{eq:sub-window}},\qquad
 R=R(Q,S),
\end{equation*}
and $RCB=RD<1$.  The integer $N:=\lfloor L\rfloor+1$ satisfies
\eqref{eq:sub-N}, and then $q\ge\overline W_{\rm diag}(\beta_1,\beta_2)$.
\end{proposition}

\begin{proof}
First, $0<q\le Q\le\beta_2S<S<w_+<2$.  We check the three hypotheses.
\begin{enumerate}[leftmargin=2.4em,itemsep=1pt,topsep=2pt]
\item[\textup{(H1)}] A low step has $w_{i,t+1}\le S=C(\sqrt{\beta_1})$, so
Lemma~\ref{lem:sub-energy} with $\theta=\sqrt{\beta_1}$ and the curvature
floor $w_{i,t+1}\ge\alpha$ of \eqref{eq:w-floor} give
$E_{i,t+1}\le\rho E_{i,t}$ for all large $t$.  (The floor is
Lemma~\ref{lem:sup-bounds}, which assumes only $\beta_1^2<\beta_2$.)
\item[\textup{(H2)}] Since $q=S\sigma^{N+1}$ and $q\le Q\le C(\theta)$,
Lemma~\ref{lem:sub-compensation} gives (H2) with $CB=D$.
\item[\textup{(H3)}] Since $q\le Q<\sigma S$, Lemma~\ref{lem:sub-reset}
gives (H3) with $R=R(Q,S)$.
\end{enumerate}
With $q=S\sigma^{N+1}$,
$RD=\max\{Rq/(e\kappa N\sigma^N),\,Rq^2\theta^N/4\}
=\max\{R\sigma S/(e\kappa N),\,R\beta_2S^2(\beta_2\theta)^N/4\}<1$ by
\eqref{eq:sub-N}.  Finally $L<N\le L+1$ gives the three inequalities
\eqref{eq:sub-N}, one per entry of the maximum defining $L$ (for the third,
$2\tau-\log\theta=-\log(\beta_2\theta)$), and
$q=S\sigma^{N+1}\ge S\sigma^{L+2}=\overline W_{\rm diag}$.
\end{proof}

\begin{remark}[The standard parameters]
\label{rem:sub-standard}
\allowdisplaybreaks
At $(\beta_1,\beta_2)=(9/10,999/1000)$ one has exactly $c^2=1/361$,
$\kappa=361/10$, $3-\beta_1=21/10$ and $(1+\beta_1)^2=361/100$.  The other
quantities in \eqref{eq:Wdiag-def} are, first in closed form and then
numerically,
\begin{alignat*}{2}
 \sigma&=\sqrt{999/1000} &\qquad&\approx0.999500,\\
 \sqrt{\beta_1}&=3/\sqrt{10} &&\approx0.948683,\\
 w_\pm&=1\pm6\sqrt{10}/19 &&\approx1.998614,\ 0.001386,\\
 \theta&=\min\bigl\{(9/10)^{3/4},\,91/100\bigr\}=91/100 &&=0.91,\\
 \mu&=\frac{9}{10\sigma}\Bigl(1+\frac{19(1-\sigma)}{6\sqrt{10}}\Bigr) &&\approx0.900901,\\
 G(c^2)&=\mu\,\frac{w_+-\sigma/361}{w_+-1/361} &&\approx0.900902,\\
 S&=C(\sqrt{\beta_1})=w_+\,\frac{\sqrt{\beta_1}-\mu}{\sqrt{\beta_1}-\sigma\mu} &&\approx1.979944,\\
 Q&=C(\theta)=w_+\,\frac{91/100-\mu}{91/100-\sigma\mu} &&\approx1.904313,\\
 \varphi(S)&=\frac{w_+-1/361}{w_+-S} &&\approx106.901,\\
 J(Q,S)&=1+\frac{1-\sigma}{\sigma(\sigma S-Q)} &&\approx1.006704,\\
 R&=\frac{21}{10}\,\varphi(S)\,J(Q,S)^2\,\frac{9}{10\,S(S-Q)} &&\approx1367.40,\\
 \tau&=\tfrac12\log(1000/999) &&\approx5.00250\cdot10^{-4},\\
 L&=\frac{\log(S/Q)}{\tau}-1 &&\approx76.855 .
\end{alignat*}
\begin{equation*}
 \boxed{\ 1.90336<\overline W_{\rm diag}(9/10,999/1000)=S\sigma^{L+2}<1.90337\ },
\end{equation*}
so $\overline W_{\rm diag}(0.9,0.999)\ge1.9$ as stated in
Theorem~\ref{prop:main-diag}.
\end{remark}

\section{Proofs from Section~\ref{sec:general}}
\label{app:general}

With $\beta_1=0$ and $\eps>0$, \eqref{eq:adam} reads
\begin{equation}
\begin{aligned}
 g_t&=\nabla f(x_t), &
 v_{t+1}&=\beta_2v_t+(1-\beta_2)g_t^{\odot2}, &
 D_{t}&=\diag\bigl(\sqrt{v_{i,t+1}}+\eps\bigr)_{i=1}^d,\\
 z_t&=D_{t}^{-1}g_t, &
 x_{t+1}&=x_t-\eta z_t,
\end{aligned}
\label{eq:scale-free}
\end{equation}
with $\eta>0$, $0\le\beta_2<1$, and $v_0\ge0$ coordinatewise.
Write the \emph{transition activity}
\begin{equation}
 \Gamma_t:=g_t\trans D_t^{-1}g_t
 =z_t\trans D_tz_t
 =\sum_{i=1}^d\frac{g_{i,t}^2}{\sqrt{v_{i,t+1}}+\eps},
 \label{eq:activity}
\end{equation}
so that the first-order change of the loss along the step is
$g_t\trans(x_{t+1}-x_t)=-\eta\Gamma_t$.  If $g_T=0$ at some step then
$z_T=0$, $x_{T+1}=x_T$, and every later step repeats this: the trajectory
is \emph{stationary from $T$ on}.  A trajectory that is not stationary in
finite time therefore has $g_t\ne0$ at every step.

\begin{lemma}[Scale-free metric bounds]
\label{lem:scalefree-bounds}
Suppose $g_t\ne0$.  Then $z_t\ne0$, $\norm{D_t}_{\rm op}>0$,
$\Gamma_t>0$, and
\begin{equation}
 \|z_{t}\|_\infty\le\frac1{\sqrt{1-\beta_2}},\qquad
 \norm{z_t}\le Z:=\sqrt{\frac{d}{1-\beta_2}},\qquad
 \frac{\norm{g_t}^2}{\norm{D_t}_{\rm op}}
 \ \le\ \Gamma_t\ \le\ \norm{D_t}_{\rm op}\norm{z_t}^2 .
 \label{eq:scalefree-bounds}
\end{equation}
In particular every step has length $\norm{x_{t+1}-x_t}\le\eta Z$.
\end{lemma}

\begin{proof}
We first bound the preconditioned gradient $z_t$.  From the
current-gradient term in \eqref{eq:scale-free},
\[
v_{i,t+1}\ge(1-\beta_2)g_{i,t}^2 .
\]
Hence, on every active coordinate,
\[
\abs{z_{i,t}}
=\frac{\abs{g_{i,t}}}{\sqrt{v_{i,t+1}}+\eps}
\le\frac{\abs{g_{i,t}}}{\sqrt{v_{i,t+1}}}
\le(1-\beta_2)^{-1/2}.
\]
On inactive coordinates $z_{i,t}=0$.  Summing over coordinates therefore
gives $\norm{z_t}\le Z$.

We next bound $\Gamma_t$.  Since $g_{i,t}=0$ outside the active support,
\eqref{eq:activity}, together with $\sqrt{v_{i,t+1}}+\eps\le\norm{D_t}_{\rm op}$,
gives the claimed lower bound on $\Gamma_t$.  For the upper bound, simply use
\[
\Gamma_t=z_t\trans D_tz_t\le\norm{D_t}_{\rm op}\norm{z_t}^2
\]
and the bound on $\norm{z_t}$ above.

Finally, if $g_t\ne0$, then at least one active coordinate satisfies
$g_{i,t}\ne0$, and hence $z_t\ne0$.  Since $D_t\succ0$, it follows that
\[
\Gamma_t=z_t\trans D_tz_t>0 .
\]
\end{proof}

\subsection{Proof of Theorem~\ref{thm:main-general}}
\label{app:general-path-passage}

Throughout, $f$ satisfies hypotheses (a)--(c) of
Theorem~\ref{thm:main-general}.

Under (b) and
(c) the stationary set $\mathcal C$ of $f$ is finite.  For each $\bar x\in\mathcal C$, we have
\begin{equation}
 \norm{\nabla f(x)}^2\ \ge\ \kappa_{\bar x}\,\abs{f(x)-f(\bar x)},
 \qquad x\in U_{\bar x},
 \qquad
 \kappa_{\bar x}:=\frac{1}{2\Lambda\norm{\nabla^2f(\bar x)^{-1}}_{\rm op}^2},
 \label{eq:local-critical-gap}
\end{equation}
on some neighborhood $U_{\bar x}$, since
$\nabla f(x)=\nabla^2f(\bar x)(x-\bar x)+o(\norm{x-\bar x})$.
 Put
\begin{equation}
 \kappa_f:=\min_{\bar x\in\mathcal C}\kappa_{\bar x}>0,
 \qquad
 \delta_\eps:=\frac{2\eps}{\eta\,\kappa_f},
 \label{eq:delta-eps}
\end{equation}
with $\kappa_f:=\infty$ and $\delta_\eps:=0$ if $\mathcal C=\emptyset$.
Thus $\delta_\eps$ is proportional to $\eps/\eta$ and vanishes only in
the degenerate case $\mathcal C=\emptyset$; it is a constant of
$(f,\eta,\eps)$ alone.

We prove Theorem~\ref{thm:main-general} with $C_f:=2/\kappa_f$, so that
$C_f\eta^{-1}\eps=\delta_\eps$, in the following equivalent form: unless the
trajectory is stationary from some finite step on, for every
$\delta>\delta_\eps$ neither
\begin{equation}
 W'_{\mathrm{act},t}\ge2+\delta\quad\text{for all large }t
 \label{eq:superpersistent}
\end{equation}
nor
\begin{equation}
 W'_{\mathrm{act},t}\le2-\delta\quad\text{for all large }t
 \label{eq:subpersistent}
\end{equation}
is possible.

\begin{proof}
On a trajectory that is not stationary in finite time, $g_t\ne0$,
$\Gamma_t>0$ and $\norm{D_t}_{\rm op}>0$ at every step
(Lemma~\ref{lem:scalefree-bounds}), so $W'_{\mathrm{act},t}$ is defined
throughout.  Fix $\delta>\delta_\eps$.  Two consequences of the recursion
are used for both regimes.

\emph{(i) The preconditioner is controlled by the current gradient.}
The Hessian bound makes $\nabla f$ $\Lambda$-Lipschitz, and each step has
length at most $\eta Z$ (Lemma~\ref{lem:scalefree-bounds}), so
$\norm{g_{t-j}}\le\norm{g_t}+\Lambda\eta Zj$.  Unrolling
$v_{i,t+1}=\beta_2^{t+1}v_{i,0}+(1-\beta_2)\sum_{j=0}^t\beta_2^jg_{i,t-j}^2$
and taking the weighted $\ell^2$ norm with weights $(1-\beta_2)\beta_2^j$,
whose total mass is at most $1$, gives $\sqrt{v_{i,t+1}}+\eps\le\norm{g_t}+C_D$
with $C_D:=\eps+\sqrt{\norm{v_0}_\infty}+\Lambda\eta Z\sqrt{\beta_2(1+\beta_2)}/(1-\beta_2)$.
Hence $\norm{D_t}_{\rm op}\le C_D+\norm{g_t}$ and, by
\eqref{eq:scalefree-bounds},
\begin{equation}
 \Gamma_t\ \ge\ \frac{\norm{g_t}^2}{C_D+\norm{g_t}},
 \label{eq:q-grad-lower}
\end{equation}
whose right-hand side is increasing and unbounded in $\norm{g_t}$: bounded
$\Gamma_t$ forces bounded $\norm{g_t}$, and $\Gamma_t\to0$ forces
$g_t\to0$.

\emph{(ii) Vanishing gradients empty the second moment.}  If $g_t\to0$ then
$v_t\to0$, being an exponential average of a null sequence, so
\begin{equation}
 \norm{D_t}_{\rm op}=\max_i\bigl(\sqrt{v_{i,t+1}}+\eps\bigr)\ \longrightarrow\ \eps,
 \label{eq:D-to-zero}
\end{equation}
and since $\delta>\delta_\eps$, i.e.\ $\eta\delta\kappa_f>2\eps$, there is
a $T_1$ with
\begin{equation}
 \frac{\eta\delta\kappa_f}{2\norm{D_t}_{\rm op}}\ >\ 1
 \qquad(t\ge T_1).
 \label{eq:factor-negative}
\end{equation}

\emph{Supercritical regime.}  Suppose \eqref{eq:superpersistent} holds for
$t\ge T$.  By \eqref{eq:Wact-loss},
\begin{equation}
 f(x_{t+1})-f(x_t)\ \ge\ \frac{\eta\delta}{2}\,\Gamma_t\ \ge\ 0
 \qquad(t\ge T),
 \label{eq:super-loss-growth}
\end{equation}
so $f(x_t)$ is nondecreasing.  Descent along the step also gives
$f(x_{t+1})-f(x_t)\le-\eta\Gamma_t+\tfrac12\Lambda\eta^2Z^2$, so
$\Gamma_t\le\tfrac12\Lambda\eta Z^2$ and, by \eqref{eq:q-grad-lower},
$\sup_{t\ge T}\norm{g_t}<\infty$; coercivity then bounds the tail
$\{x_t:t\ge T\}$ and with it $f$, so $f(x_t)\uparrow f_\infty<\infty$.
Summing \eqref{eq:super-loss-growth} gives
$\sum_{t\ge T}\Gamma_t\le\tfrac{2}{\eta\delta}(f_\infty-f(x_T))<\infty$,
hence $\Gamma_t\to0$, $g_t\to0$, and \eqref{eq:D-to-zero}
and~\eqref{eq:factor-negative} apply.

Every accumulation point of the bounded tail is then a critical point of
value $f_\infty$, so $\mathcal C\ne\emptyset$ and, for all large $t$, $x_t$
lies in the union of the neighborhoods $U_{\bar x}$ of those accumulation
points.  Applying \eqref{eq:local-critical-gap} there with
$\kappa_{\bar x}\ge\kappa_f$,
\begin{equation}
 \norm{g_t}^2\ \ge\ \kappa_f\,\Delta_t,
 \qquad
 \Delta_t:=f_\infty-f(x_t)\ge0,
 \label{eq:uniform-gap-super}
\end{equation}
and \eqref{eq:super-loss-growth}, \eqref{eq:scalefree-bounds}
and~\eqref{eq:uniform-gap-super} combine into
\[
 \Delta_{t+1}\ \le\ \Delta_t-\frac{\eta\delta}{2}\,
 \frac{\norm{g_t}^2}{\norm{D_t}_{\rm op}}
 \ \le\ \Delta_t\Bigl(1-\frac{\eta\delta\kappa_f}{2\norm{D_t}_{\rm op}}\Bigr).
\]
The factor is negative for $t\ge T_1$ by \eqref{eq:factor-negative} while
$\Delta_{t+1}\ge0$, so $\Delta_t=0$ for all large $t$; then
\eqref{eq:super-loss-growth} forces $\Gamma_t=0$, that is $g_t=0$, and the
trajectory is stationary from a finite step on, a contradiction.

\emph{Subcritical regime.}  Suppose \eqref{eq:subpersistent} holds for
$t\ge T$.  By \eqref{eq:Wact-loss},
\begin{equation}
 f(x_{t+1})-f(x_t)\ \le\ -\frac{\eta\delta}{2}\,\Gamma_t
 \qquad(t\ge T),
 \label{eq:sub-loss-decay}
\end{equation}
and $f$ is bounded below, so $\sum_{t\ge T}\Gamma_t<\infty$,
$\Gamma_t\to0$, $g_t\to0$ and $f(x_t)\downarrow f_\infty$.  The tail is
bounded as before and the same local gap gives $\norm{g_t}^2\ge\kappa_f\Delta_t$
with $\Delta_t:=f(x_t)-f_\infty$, whence
$\Delta_{t+1}\le\Delta_t\bigl(1-\eta\delta\kappa_f/(2\norm{D_t}_{\rm op})\bigr)$
with an eventually negative factor: again $\Delta_t=0$ for all large $t$
and \eqref{eq:sub-loss-decay} forces $\Gamma_t=0$, a contradiction.

Neither \eqref{eq:superpersistent} nor \eqref{eq:subpersistent} is
therefore possible for any $\delta>\delta_\eps$, which is
\eqref{eq:main-two-sided} with $C_f=2/\kappa_f$.
\end{proof}

\subsection{Analyzing the active curvature}
\label{app:general-cur}

Define
\begin{equation}
 e_t:=\eta\,\frac{z_t\trans(\overline H_t-H_t)z_t}{\Gamma_t},
 \qquad
 W'_{\mathrm{act},t}=W_{\mathrm{act},t}+e_t,
 \qquad
 f(x_{t+1})-f(x_t)=\frac{\eta}{2}\Gamma_t\bigl(W_{\mathrm{act},t}-2+e_t\bigr).
 \label{eq:defect}
\end{equation}
Further, introduce the directional curvature
\begin{equation}
 {\varkappa_t}:=\frac{z_t\trans H_tz_t}{\norm{z_t}^2}.
 \label{eq:current-directional}
\end{equation}

The proposition is straightforward; we omit the proof.

\begin{proposition}
\label{prop:defect}
Suppose that
\begin{equation}
 \norm{\nabla^2f(x)-\nabla^2f(y)}_{\rm op}\le L_H\norm{x-y}_\infty.
 \label{eq:lipschitz-hessian}
\end{equation}
Then
\begin{equation}
 \norm{\overline H_t-H_t}_{\rm op}\le\frac{L_H\eta}{3}\norm{z_t},
 \qquad
 \abs{e_t}=\abs{W'_{\mathrm{act},t}-W_{\mathrm{act},t}}
 \le\frac{L_H\eta^2}{3}\,\frac{\norm{z_t}^3}{\Gamma_t}.
 \label{eq:additive-defect-bound}
\end{equation}
If moreover $\varkappa_t>0$, then
\begin{equation}
 \abs{W'_{\mathrm{act},t}-W_{\mathrm{act},t}}
 \le\frac{L_H\eta\norm{z_t}}{3\varkappa_t}\,W_{\mathrm{act},t}
 \le\frac{L_HZ\eta}{3\varkappa_t}\,W_{\mathrm{act},t},
 \label{eq:relative-defect-bound}
\end{equation}
with $Z$ as in Lemma~\ref{lem:scalefree-bounds}.
\end{proposition}

The second bound follows from the first: by \eqref{eq:Wact} with $c=1$
and \eqref{eq:current-directional},
$W_{\mathrm{act},t}=\eta\,z_t\trans H_tz_t/\Gamma_t=\eta\varkappa_t\norm{z_t}^2/\Gamma_t$,
and $\norm{z_t}\le Z$ by Lemma~\ref{lem:scalefree-bounds}.  Thus, whenever
the step direction carries curvature bounded away from zero, the measured
$W_{\mathrm{act},t}$ and the loss criterion $W'_{\mathrm{act},t}$ differ by
a relative error $O(\eta)$.

\paragraph{Consecutive gradients.}
The same averaging argument transfers the reversal identity
\eqref{eq:flipping-Wact} of Section~\ref{sec:flip} from quadratics to general
objectives.  Write
\begin{equation}
 \widehat H_t:=\int_0^1\nabla^2f(x_t-s\eta z_t)\,ds,
 \qquad
 \hat e_t:=\eta\,\frac{z_t\trans(\widehat H_t-H_t)z_t}{\Gamma_t},
 \qquad
 \widehat P_t:=\eta\,D_t^{-1/2}\widehat H_tD_t^{-1/2},
 \label{eq:flip-defect}
\end{equation}
so that $g_{t+1}=g_t-\eta\widehat H_tz_t$ by the fundamental theorem of
calculus, and $\hat z_t\trans\widehat P_t\hat z_t=W_{\mathrm{act},t}+\hat e_t$.
The weight in $\widehat H_t$ is uniform, whereas $\overline H_t$ in
\eqref{eq:defect} carries the weight $2(1-s)$; this is why the constants
below are $\tfrac12$ where those of Proposition~\ref{prop:defect} are
$\tfrac13$.

\begin{proposition}[Gradient reversal on general objectives]
\label{prop:flip-general}
Let $\beta_1=0$ and $g_t\ne0$.  Then
\begin{equation}
\begin{aligned}
 \frac{\langle g_t,g_{t+1}\rangle_{D_t^{-1}}}{\norm{g_t}_{D_t^{-1}}^2}
 &=1-W_{\mathrm{act},t}-\hat e_t,\\
 \cos(D_t^{-1/2}g_t,D_t^{-1/2}g_{t+1})
 &=\frac{1-W_{\mathrm{act},t}-\hat e_t}
       {\sqrt{(1-W_{\mathrm{act},t}-\hat e_t)^2+\widehat V_t}},
\end{aligned}
 \label{eq:flip-general}
\end{equation}
with $\widehat V_t:=\norm{(\widehat P_t-(W_{\mathrm{act},t}+\hat e_t)I)\hat z_t}^2$.
Under \eqref{eq:lipschitz-hessian},
\begin{equation}
 \abs{\hat e_t}\le\frac{L_H\eta^2}{2}\,\frac{\norm{z_t}^3}{\Gamma_t},
 \qquad\text{and, if }\varkappa_t>0,\qquad
 \abs{\hat e_t}\le\frac{L_H\eta\norm{z_t}}{2\varkappa_t}\,W_{\mathrm{act},t}
 \le\frac{L_HZ\eta}{2\varkappa_t}\,W_{\mathrm{act},t}.
 \label{eq:flip-defect-bound}
\end{equation}
In particular $W_{\mathrm{act},t}>1+\abs{\hat e_t}$ implies
$\langle g_t,g_{t+1}\rangle_{D_t^{-1}}<0$.
\end{proposition}

\begin{proof}
Since $x_{t+1}-x_t=-\eta z_t$ and $z_t=D_t^{-1}g_t$,
\[
\begin{aligned}
 \langle g_t,g_{t+1}\rangle_{D_t^{-1}}
 &=g_t\trans D_t^{-1}g_t-\eta\,g_t\trans D_t^{-1}\widehat H_tz_t
 =\Gamma_t-\eta\,z_t\trans H_tz_t-\eta\,z_t\trans(\widehat H_t-H_t)z_t\\
 &=\Gamma_t\bigl(1-W_{\mathrm{act},t}-\hat e_t\bigr)
\end{aligned}
\]
holds
because $\eta\,z_t\trans H_tz_t=\Gamma_tW_{\mathrm{act},t}$ by \eqref{eq:Wact}
with $c=1$.  Dividing by $\Gamma_t=\norm{g_t}_{D_t^{-1}}^2$ gives the first
identity.  For the second, $D_t^{-1/2}g_{t+1}=\norm{D_t^{-1/2}g_t}\,(I-\widehat P_t)\hat z_t$,
so the cosine equals $\hat z_t\trans(I-\widehat P_t)\hat z_t/\norm{(I-\widehat P_t)\hat z_t}$,
and $\norm{(I-\widehat P_t)\hat z_t}^2=(1-\hat z_t\trans\widehat P_t\hat z_t)^2+\widehat V_t$
by the Pythagorean decomposition of $(I-\widehat P_t)\hat z_t$ along and
orthogonal to $\hat z_t$.  For the bound, \eqref{eq:lipschitz-hessian} and
$\norm{\cdot}_\infty\le\norm{\cdot}$ give
$\norm{\nabla^2f(x_t-s\eta z_t)-H_t}_{\rm op}\le L_Hs\eta\norm{z_t}$, hence
\[
 \norm{\widehat H_t-H_t}_{\rm op}\le L_H\eta\norm{z_t}\int_0^1s\,ds
 =\frac{L_H\eta}{2}\norm{z_t},
 \quad
 \abs{\hat e_t}\le\frac{\eta\norm{z_t}^2\norm{\widehat H_t-H_t}_{\rm op}}{\Gamma_t}
 \le\frac{L_H\eta^2\norm{z_t}^3}{2\Gamma_t}.
\]
The relative form follows from
$W_{\mathrm{act},t}=\eta\varkappa_t\norm{z_t}^2/\Gamma_t$ and
$\norm{z_t}\le Z$ (Lemma~\ref{lem:scalefree-bounds}).
\end{proof}

\end{document}